%% file: moulin25.tex
\documentclass[final]{settings/colt2025} 

\title[Optimistically Optimistic Exploration for Efficient Infinite-Horizon RL]{Optimistically Optimistic 
Exploration for Provably Efficient Infinite-Horizon Reinforcement and Imitation Learning}
\usepackage{times}

\author[Moulin, Neu and Viano]{%
 \Name{Antoine Moulin} \Email{antoine.moulin@upf.edu}\\
 \addr Universitat Pompeu Fabra, Barcelona, Spain%
 \AND
 \Name{Gergely Neu} \Email{gergely.neu@gmail.com}\\
 \addr Universitat Pompeu Fabra, Barcelona, Spain
 \AND
 \Name{Luca Viano} \Email{luca.viano@epfl.ch}\\
 \addr EPFL, Lausanne, Switzerland
}

\input{settings/lmacros}

\input{settings/gmacros}

\input{settings/amacros}

\input{settings/preamble}

\renewcommand{\phi}{\varphi}

\newcommand{\rewardbias}{\texttt{reward-bias}}
\newcommand{\modelbias}{\texttt{model-bias}}

\usepackage{ninecolors}
\hypersetup{colorlinks=true, linkcolor=blue3, citecolor=blue3, urlcolor=blue3}

\usepackage{tikz}
\usetikzlibrary{arrows.meta}
\usetikzlibrary{decorations.pathreplacing}
\usetikzlibrary{positioning}
\usetikzlibrary{automata}

\usepackage{transparent}
\newcommand{\algcomment}[1]{\textcolor{blue!70!black}{\transparent{0.7}\small{\texttt{\textbf{\#\hspace{2pt}#1}}}}}

\newlength{\minipagewidth}
\usepackage{tocloft}
\addtocontents{toc}{\protect\setcounter{tocdepth}{0}}

\usepackage{array}
\newcolumntype{P}[1]{>{\centering\arraybackslash}p{#1}}
\newcolumntype{M}[1]{>{\centering\arraybackslash}m{#1}}

\begin{document}

\maketitle

\begin{abstract}%
  We study the problem of reinforcement learning in infinite-horizon discounted linear Markov decision processes (MDPs), and propose the first computationally efficient algorithm achieving rate-optimal regret guarantees in this setting. Our main idea is to combine two classic techniques for optimistic exploration: additive exploration bonuses applied to the reward function, and artificial transitions made to an absorbing state with maximal return. We show that, combined with a regularized approximate dynamic-programming scheme, the resulting algorithm achieves a regret of order $\tilde{\mathcal{O}} (\sqrt{d^3 (1 - \gamma)^{- 7 / 2} T})$, where $T$ is the total number of sample transitions, $\gamma \in (0,1)$ is the discount factor, and $d$ is the feature dimensionality. The results continue to hold against adversarial reward sequences, enabling application of our method to the problem of imitation learning in linear MDPs, where we achieve state-of-the-art results.
\end{abstract}


\input{sections/introduction.tex}
\input{sections/preliminaries.tex}
\input{sections/algorithm.tex}
\input{sections/analysis.tex}

\input{sections/applications.tex}

\input{sections/conclusion.tex}

\acks{The authors wish to thank Asaf Cassel and Aviv Rosenberg for sharing further insights about their work at the 
virtual RL theory seminars, and Volkan Cevher for initial discussions about this project. 
This project has received funding from the European Research Council (ERC), under the European Union's Horizon 2020 
research and innovation programme (Grant agreement No.~950180).
This work is funded (in part) through a PhD fellowship of the Swiss Data Science Center, a joint
venture between EPFL and ETH Zurich.  Luca Viano acknowledges travel support from ELISE (GA no 951847).}

\bibliography{ref}

\clearpage
\appendix
\renewcommand{\contentsname}{Contents of Appendix}
\addtocontents{toc}{\protect\setcounter{tocdepth}{2}}
{
  \setlength{\cftbeforesecskip}{.8em}
  \setlength{\cftbeforesubsecskip}{.5em}
  \tableofcontents
}
\clearpage
\input{sections/omitted_pseudocodes.tex}

\clearpage
\input{sections/analysis_proofs.tex}

\clearpage
\input{sections/related_works_IL.tex}

\clearpage
\input{sections/applications_proofs.tex}

\clearpage
\input{sections/lower_bounds_appendix.tex}

\clearpage
\input{sections/technical_tools.tex}


\end{document}

%% file: settings/lmacros.tex
\usepackage[capitalize,noabbrev]{cleveref}
\newcommand{\bc}[1]{\left\{{#1}\right\}}
\newcommand{\brr}[1]{\left({#1}\right)}
\newcommand{\bs}[1]{\left[{#1}\right]}

\newcommand{\aspace}{\mcf{A}}     

\usepackage{algorithm}
\usepackage{algorithmic}
\usepackage{multirow}
\newcounter{protocol}
\makeatletter
\newenvironment{protocol}[1][htb]{%
  \let\c@algorithm\c@protocol
  \renewcommand{\ALG@name}{Protocol}
  \begin{algorithm}[#1]%
  }{\end{algorithm}
}
\usepackage{amssymb}
\usepackage{amsfonts}
\usepackage{graphics}
\usepackage{graphicx}
\usepackage{setspace}
    
\newcommand{\trans}{\intercal}

\newcommand{\mcf}{\mathcal}

\newcommand{\innerprod}[2]{\left\langle{#1},{#2}\right\rangle}

\newcommand{\expert}{{\pi_E}}
\newcommand{\experti}{\pi_E^i}

\newcommand{\cost}{r}
\newcommand{\true}{r_{\textup{true}}}

\newcommand{\FRAalg}{\texttt{\textbf{FRA-IL}}\xspace}
\newcommand{\phim}{\Phi}
\newcommand{\initial}{\nu_0}

\newcommand{\lv}[1]{%
    \ifmmode
    \text{\textcolor{orange}{[LV: #1]}}
    \else
    \textcolor{orange}{[LV: #1]}
    \fi
}
\usepackage[utf8]{inputenc} 
\usepackage[T1]{fontenc}    
\usepackage{url}            
\usepackage{booktabs}       
\usepackage{amsfonts}       
\usepackage{nicefrac}       
\usepackage{microtype}      
\usepackage{xcolor}         
\usepackage{dsfont}
\usepackage{pifont}
\newcommand{\greentick}{\textcolor{green}{\ding{51}}}
\newcommand{\redcross}{\textcolor{red}{\ding{55}}}
\newtheorem{assumption}{\protect\assumptionname}
\providecommand{\assumptionname}{Assumption}

\usepackage{mathtools}
\makeatletter
\newcommand{\mytag}[2]{%
  \text{#1}%
  \@bsphack
  \protected@write\@auxout{}%
         {\string\newlabel{#2}{{#1}{\thepage}}}%
  \@esphack
}
\makeatother

\usepackage{thm-restate}

%% file: settings/gmacros.tex
\newcommand{\raviUCB}{\texttt{\textbf{RAVI-UCB}}\xspace}
\newcommand{\RMAXalg}{\texttt{\textbf{Rmax}}\xspace}
\newcommand{\algname}{\texttt{\textbf{Rmax-RAVI-UCB}}\xspace}

\newcommand{\RMAX}{R_{\max}}

\newcommand{\DDKL}[2]{\mathcal{D}_{\text{KL}}\pa{#1\middle\|#2}}

\newcommand{\X}{\mathcal{X}}

\newcommand{\A}{\mathcal{A}}

\newcommand{\real}{\mathbb{R}}

\newcommand{\II}[1]{\mathbb{I}_{\left\{#1\right\}}}

\newcommand{\EE}[1]{\mathbb{E}\left[#1\right]}

\newcommand{\EEs}[2]{\mathbb{E}_{#2}\left[#1\right]}

\newcommand{\ra}{\rightarrow}

\newcommand{\iprod}[2]{\left\langle#1,#2\right\rangle}
\newcommand{\biprod}[2]{\bigl\langle#1,#2\bigr\rangle}

\newcommand{\norm}[1]{\left\|#1\right\|}

\newcommand{\ev}[1]{\left\{#1\right\}}
\newcommand{\pa}[1]{\left(#1\right)}
\newcommand{\abs}[1]{\left|#1\right|}
\newcommand{\bpa}[1]{\bigl(#1\bigr)}

\newcommand{\wh}{\widehat}

\newcommand{\transpose}{^\mathsf{\scriptscriptstyle T}}

\definecolor{PalePurp}{rgb}{0.66,0.57,0.66}
\newcommand{\regret}{\mathfrak{R}_T}
\newcommand{\regretK}{\mathfrak{R}_K}

%% file: settings/amacros.tex
\newcommand{\antoine}[1]{%
    \ifmmode
    \text{\textcolor{blue}{[Antoine: #1]}}
    \else
    \textcolor{blue}{[Antoine: #1]}
    \fi
}

\newcommand{\tmpassumption}[1]{%
    \ifmmode
    \text{\textcolor{orange}{[Assumption: #1]}}
    \else
    \textcolor{orange}{[Assumption: #1]}
    \fi
}

\makeatletter
\DeclareRobustCommand\onedot{\futurelet\@let@token\@onedot}
\def\@onedot{\ifx\@let@token.\else.\null\fi\xspace}

\def\eg{\textit{e.g}\onedot}

\def\ie{\textit{i.e}\onedot}

\makeatother

\newcommand{\spr}[1]{\left( #1 \right)}
\newcommand{\sbr}[1]{\left[ #1 \right]}

\newcommand{\scbr}[1]{\left\{ #1 \right\}}
\newcommand{\inp}[1]{\left\langle #1 \right\rangle}

\newcommand{\sumkK}{\sum_{k=1}^K}
\newcommand{\sumtT}{\sum_{t=1}^T}

\usepackage{mathtools} 
\newcommand{\given}{\mathrel{}\middle|\mathrel{}}
\newcommand{\KL}{\mathcal{D}_{\text{KL}}}
\newcommand{\upplus}{\mathsf{\scriptscriptstyle +}} 

\usepackage{mathtools}  

\newcommand{\bbE}{\mathbb{E}}

\newcommand{\bbN}{\mathbb{N}}
\newcommand{\bbP}{\mathbb{P}}

\newcommand{\bbR}{\mathbb{R}}

\newcommand{\bfone}{\mathbf{1}}
\newcommand{\bfe}{\mathbf{e}}

\newcommand{\cA}{\mathcal{A}}

\newcommand{\cD}{\mathcal{D}}
\newcommand{\cE}{\mathcal{E}}
\newcommand{\cF}{\mathcal{F}}

\newcommand{\cH}{\mathcal{H}}

\newcommand{\cK}{\mathcal{K}}

\newcommand{\cM}{\mathcal{M}}
\newcommand{\cN}{\mathcal{N}}
\newcommand{\cO}{\mathcal{O}}

\newcommand{\cQ}{\mathcal{Q}}
\newcommand{\cR}{\mathcal{R}}

\newcommand{\cT}{\mathcal{T}}
\newcommand{\cU}{\mathcal{U}}
\newcommand{\cV}{\mathcal{V}}
\newcommand{\cW}{\mathcal{W}}
\newcommand{\cX}{\mathcal{X}}

\newcommand{\cZ}{\mathcal{Z}}

\newcommand{\opE}{E}
\newcommand{\opP}{P}

\newcommand{\ophPk}{\widehat{P}_k}

\newcommand{\opPplusk}{P_k^\mathsf{\scriptscriptstyle +}}

\newcommand{\ophPplusk}{\widehat{P}_k^\mathsf{\scriptscriptstyle +}}

\DeclareMathOperator{\optrace}{trace}

\DeclareMathOperator{\regretKplus}{\mathfrak{R}_K^\mathsf{\scriptscriptstyle +}}
\DeclareMathOperator{\regretIL}{\mathfrak{R}_K^\mathrm{I\hspace{.05em}L}}
\DeclareMathOperator{\CB}{CB}
\DeclareMathOperator{\kl}{kl}

\newcommand{\piunif}{\pi_{\texttt{unif}}}
\newcommand{\QMAX}{Q_{\max}}
\newcommand{\VMAX}{Q_{\max}}
\newcommand{\LMAX}{L_{\max}}
\newcommand{\WMAX}{W_{\max}}
\newcommand{\reset}{reset\xspace}
\DeclareMathOperator{\LSE}{LSE}

%% file: settings/preamble.tex
\usepackage{hyperref}       
\usepackage{url}            
\usepackage{amsfonts}       
\usepackage{nicefrac}       
\usepackage{amsmath}
\usepackage{color}
\usepackage{standalone}
\usepackage{setspace}
\usepackage{amssymb}

\newcommand{\Reg}{\mathfrak{R}} 

\usepackage{enumitem}
\setitemize{noitemsep,topsep=3pt,parsep=3pt,partopsep=3pt}
\setenumerate{noitemsep,topsep=3pt,parsep=3pt,partopsep=3pt,leftmargin=*}
\setdescription{noitemsep,topsep=3pt,parsep=3pt,partopsep=3pt,leftmargin=10pt}

%% file: sections/introduction.tex
\section{Introduction}

Since the breakthrough work of \citet{jin2019provably}, the class of linear Markov decision processes (MDPs) has become a standard model for theoretical analysis of reinforcement learning (RL) algorithms under linear function approximation. This work demonstrated the possibility of constructing computationally and statistically efficient methods for large-scale RL, and pioneered an analysis technique that influenced the entire field of RL theory. Hundreds of follow-up papers have studied variations of this model, studying extensions such as learning with adversarial rewards \citep{Neu:2021,zhong2024theoretical,sherman2023,dai2023, sherman2023rate,cassel2024warmupfree,liu2023towards}, without rewards \citep{Wang:2020b,wagenmaker2022reward,hu2022towards}, or with unknown features \cite{agarwal2020flambe,uehara2021representation,zhang2022efficient,mhammedi2024efficient,modi2024model}. The linearity constraint itself has been relaxed in a variety of ways \cite{zanette2020learning,Cai:2020,du2021bilinear,weisz2024online,golowich2024linear,wu2024computationally}. However, practically all of these developments retained one major limitation of the original work of \citet{jin2019provably}: it only applies to finite-horizon MDPs. Generalizations to the more challenging (and practically much more popular) infinite-horizon MDP models have so far remained very limited, yielding only highly impractical methods or suboptimal performance guarantees \citep{WJLJ20}. In this paper, we propose an efficient algorithm that successfully addresses this long-standing open problem.

We consider the problem of learning a nearly optimal policy in $\gamma$-discounted MDPs \citep{Puterman:1994}, under the linear MDP assumption first proposed by \citet{jin2019provably} (see also \citealp{Yang:2019}). We consider an interaction protocol where a learning agent interacts with the environment in a sequence of $K$ episodes of geometrically distributed length, and aims to pick a sequence of policies such that its regret against the best fixed policy is as small as possible. Our algorithm achieves a regret bound of order $H\sqrt{d^3T} + H^{7/4} \sqrt{d T \log |\cA|}$, where $d$ is the feature dimensionality, $H = \frac{1}{1-\gamma}$ is the effective horizon, $\cA$ is the action space, and $T$ is the number of interactions. This implies a bound on the sample complexity of learning an $\varepsilon$-optimal policy of the order $\frac{H^3 d^3 + H^{7/2} d \log \abs{\cA}}{\varepsilon^2}$. The algorithm returns a single softmax policy that is fully described in terms of a $d$-dimensional parameter vector and a $d^2$-dimensional feature-covariance matrix. This constitutes the first sample-complexity result of the optimal order $1/\varepsilon^2$ achieved by a computationally efficient algorithm. The regret guarantees are also shown to hold if the reward function changes adversarially over time, and we additionally provide an extension of our method for the setting of imitation learning.

On the technical side, our main contribution is the development of a new optimistic exploration mechanism that combines two classic ideas from two different eras of RL theory. First, following the recipe of \citet{jin2019provably}, we make use of additive UCB-style exploration bonuses which have been successfuly used for several decades in both bandit problems \citep{LR85,ACF02,A02,dani08stoch,APS11} and reinforcement learning \citep{kaelbling1996reinforcement,Strehl:2008,jaksch10ucrl,azar2017minimax}. Second (and more importantly), we adapt another classic (but apparently recently less well-known) idea underlying the \RMAXalg algorithm of \citet{brafman2002r} (see also \citealp{szita2010model} and Chapter~8 in \citealp{kakade2003sample}). Roughly speaking, this technique amounts to replacing the standard empirical model estimate with a fixed optimistic estimate in state-action pairs that are very poorly explored. This addresses the notorious problem of empirical estimates in linear MDPs that they tend to have extremely high variance in under-explored states, which can only be offset with very large additive exploration bonuses. Our \RMAXalg-style scheme counteracts these large bonuses by effectively swapping out the possibly over-optimistic estimates obtained via additive bonuses with more reasonably sized estimates. Besides \citet{brafman2002r}, our algorithm design and analysis is also strongly inspired by the recent work of \citet{cassel2024warmupfree} who proposed a slightly limited variant of the same exploration mechanism for finite-horizon MDPs.

%% file: sections/preliminaries.tex
\vspace{-2mm}
\section{Preliminaries}
\vspace{-1mm}

In this section, we first provide the general definitions that will repeatedly appear throughout the paper, and then go on to describe a set of ideas that will be heavily featured in our algorithm design and analysis. We finally describe the concrete learning setting in detail at the end of the section.

\vspace{-2mm}
\subsection{Markov decision processes}

A Markov decision process (MDP) with reward function $r$ is defined by the tuple $\cM(r) = \spr{\cX, \cA, r, P, \gamma, \nu_0}$, where $\cX$ is the (possibly infinite) state space, $\cA$ is the finite action space, $r: \cX \times \cA \times \cX \rightarrow \sbr{0, 1}$ is the reward function assigning rewards to each state-action-next-state transition, $P: \cX \times \cA \rightarrow \Delta \spr{\cX}$ is the transition kernel, $\gamma \in \spr{0, 1}$ is the discount factor, and $\nu_0 \in \Delta \spr{\cX}$ is the initial-state distribution. For convenience, we will assume that $\cX$ is countable but note that this can be lifted at the expense of making the measure-theoretic notation much heavier.
The MDP $\cM(r)$ models a sequential decision-making problem between a decision-making \emph{agent} and its \emph{environment}. The interaction starts with the environment drawing the random initial state $X_0\sim\nu_0$, whereafter in each time step $t = 0, 1, 2, \dots$, the following steps are repeated: the agent observes state $X_t \in \cX$, takes action $A_t \in \cA$, and consequently the environment generates the next state $X_{t + 1} \sim P(\cdot | X_t, A_t)$, resulting in reward $R_t = r(X_t, A_t, X_{t + 1})$. With a slight abuse of notation, we denote the mean reward of a state-action pair $\spr{x, a} \in \cX \times \cA$ by $r \spr{x, a} = \EEs{r(x,a,X')}{X'\sim P(\cdot|x,a)}$.

A \emph{stationary state-feedback policy} (or, in short, \emph{policy}) is a randomized behavior rule $\pi: \cX \rightarrow \Delta \spr{\cA}$ that determines the action taken in each time step $t$ as $A_t \sim \pi(\cdot | X_t)$. The \emph{action-value function} of a policy $\pi$ in $\cM$ is defined for any state-action pair $\spr{x, a}$ as
\begin{align*}
    Q_{P, r}^\pi \spr{x, a} &= \bbE_{P, \pi} \sbr{\sum_{\tau=0}^\infty \gamma^\tau r \spr{X_\tau, A_\tau} \middle| (X_0,A_0) = (x,a)}\,,
\end{align*}
where $\bbE_{P, \pi}$ denotes the expectation with respect to the random sequence of states and actions generated by the transition kernel $P$ and the policy $\pi$. The \emph{value function} of $\pi$ at state $x$  is defined as $V_{P, r}^\pi \spr{x} = \bbE_{A \sim \pi \spr{\cdot \middle| x}} \bigl[Q_{P, r}^\pi \spr{x, A}\bigr]$. With some abuse of notation, we define the conditional expectation operator $P: \bbR^{\cX} \rightarrow \bbR^{\cX \times \cA}$ via its action $\bpa{P f} \spr{x, a} = \bbE_{X' \sim P \spr{\cdot \middle| x, a}} \sbr{f \spr{X'}}$ for any function $f \in \bbR^{\cX}$ and state-action pair $\spr{x, a}$. Its adjoint $P\transpose$ is the operator that acts on distributions $\mu \in \Delta \spr{\cX \times \cA}$ as $P\transpose \mu = \bbE_{\spr{X, A} \sim \mu} \sbr{P \spr{\cdot \middle| X, A}}$. With this notation, the value functions can be shown to satisfy the \emph{Bellman equations} written as
\begin{equation*}
    Q_{P, r}^\pi = r + \gamma P V_{P, r}^\pi\,.
\end{equation*}
For convenience, we also introduce the operator $E: \bbR^{\cX} \rightarrow \bbR^{\cX \times \cA}$ defined via $\bpa{E f} \spr{x, a} = f \spr{x}$ and whose adjoint acts on state-action distributions as $\bpa{E\transpose \mu }\spr{x} = \sum_{a \in \cA} \mu \spr{x, a}$. When interacting with an  MDP, any stationary policy $\pi$ induces a unique \emph{state-occupancy measure} denoted as $\nu\spr{\pi} \in \Delta \spr{\cX}$ and a state-action occupancy measure $\mu\spr{\pi} \in \Delta \spr{\cX\times\cA}$ defined (with an unusual but helpful abuse of notation) as
\begin{equation*}
    \nu\spr{\pi,\cdot} = \spr{1 - \gamma} \sum_{\tau=0}^\infty \gamma^\tau \bbP_{P, \pi} \sbr{X_\tau \in \cdot} \quad \mbox{and} \quad \mu\spr{\pi,\cdot} = \spr{1 - \gamma} \sum_{\tau=0}^\infty \gamma^\tau \bbP_{P, \pi} \sbr{(X_\tau,A_\tau) \in \cdot}.
\end{equation*}

\subsection{Optimistically augmented Markov decision processes}\label{sec:OAMDP}

A key concept in our algorithm design is that of \emph{optimistically augmented Markov decision processes} (OA-MDPs), inspired by the construction of \citet{brafman2002r}. The OA-MDP associated with $\cM(r)$ is defined on the augmented state space $\cX^\upplus = \cX \cup \ev{x^\upplus}$, where $x^\upplus$ is an artificial \emph{heaven} state appended to the original set of states. The transition dynamics are defined via a perturbation of the original transition function, governed by the \emph{ascension function} $p^\upplus:\cX^\upplus\times\cA\ra[0,1]$. In particular, the transition kernel from state-action pair $x,a$ to $x'$ is defined as
\begin{align*}
    P^\upplus \!\!\spr{\cdot \middle| x, a} = \bpa{1 - p^\upplus  \!\!\spr{x, a}} P \spr{\cdot \middle| x, a} + p^\upplus \!\!\spr{x, a} \II{x^\upplus\in \cdot}\,.
\end{align*}
In words, the sequence of states in the augmented MDP follows the dynamics of the original process, except that the process \emph{ascends} to heaven with probability $p^\upplus (X_t, A_t)$ in round $t$. The augmented reward function is the same for all triples in the original MDP $x,a,x'$, and ascension to heaven results in maximal reward $r(x,a,x^\upplus) = \RMAX$. The resulting state-action reward function is then
\begin{equation*}
    r^\upplus \spr{x, a} = \EEs{r \spr{x, a, X'}}{X'\sim P^\upplus \!\spr{\cdot \middle| x, a}} = \bpa{1 - p^\upplus \!\!\spr{x, a}} r \spr{x, a} + p^\upplus \!\!\spr{x, a} \RMAX\,.
\end{equation*}
Once the process enters $x^\upplus$, it remains there forever (\ie, $P^\upplus(\ev{x^\upplus}|x^\upplus,a) =1$ for all actions $a$) and obtains maximal reward $\RMAX$ in each round. Without loss of generality (and for notational convenience), we will assume throughout the state $x^\upplus$ also exists in the original MDP $\cM(r)$, but is not reachable either via regular transitions ($P \spr{\ev{x^\upplus} \middle| x, a} = 0$) or initialization ($\nu_0(\ev{x^\upplus})=0$). We will also follow the convention that $p^\upplus(x^\upplus,a)=0$ for all actions $a$. We will refer to the optimistically augmented MDP as $\cM^\upplus(r,p^\upplus)$, and illustrate the relation of the two processes in Figure~\ref{fig:illustration-mdps}.

Our algorithm and its analysis will feature a sequence of ascension functions denoted by $p_k^\upplus$, and the associated transition function will be denoted by $P_k^\upplus$. Within the augmented MDP induced by $p_k^\upplus$, we denote the value functions of a policy $\pi$ in $\cM^\upplus(r,p_k^\upplus)$ as $V_{P_k^\upplus, r^\upplus}^\pi$ and $Q_{P_k^\upplus, r^\upplus}^\pi$. Likewise, we will use $\nu_k^\upplus \spr{\pi}$ and $\mu_k^\upplus \spr{\pi}$ to refer to the occupancy measures of $\pi$ in $\cM^\upplus(r,p_k^\upplus)$. It is easy to see that for any policy $\pi$, the value functions satisfy $V_{P_k^\upplus, r^\upplus}^\pi \ge V_{P, r}^\pi$ and $Q_{P_k^\upplus, r^\upplus}^\pi \ge Q_{P, r}^\pi$, which explains why we call the resulting MDP ``optimistic''. Furthermore, for all non-heaven states $x$, we have $\nu_k^\upplus \spr{\pi,x} \le \nu\spr{\pi,x}$ and $\mu_k^\upplus \spr{\pi,x,a} \le \mu^\upplus \spr{\pi,x,a}$. Our analysis will heavily rely on these facts (which will be proved formally later).

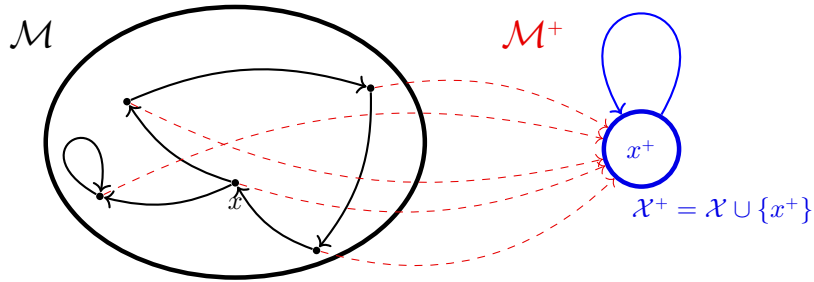
\begin{figure}
    \centering
    \begin{tikzpicture}[scale=1.8, every node/.style={font=\small}]
        \draw[black, thick, line width=.6mm] (0,0) ellipse (1.4cm and 1cm);
        \node[black] at (-1.5,.8) {\Large $\cM$};
        
        \node[circle, fill=black, inner sep=1pt, label=below:{$x$}] (x) at (0,-0.3) {};
        
        \node[circle, fill=black, inner sep=1pt] (p1) at (-0.8,0.3) {};
        \node[circle, fill=black, inner sep=1pt] (p2) at (0.6,-0.8) {};
        \node[circle, fill=black, inner sep=1pt] (p3) at (1,0.4) {};
        \node[circle, fill=black, inner sep=1pt] (p4) at (-1,-0.4) {};
        
        \draw[->, thick] (x) to[bend left=20] (p1);
        \draw[->, thick] (p1) to[bend left=20] (p3);
        \draw[->, thick] (p3) to[bend left=20] (p2);
        \draw[->, thick] (p2) to[bend left=20] (x);
        \draw[->, thick] (x) to[bend left=20] (p4);
        \draw[->, thick] (p4) to[in=80, out=150, looseness=50] (p4);
        \node[draw, circle, thick, blue4, minimum size=1cm, line width=.6mm] (x_plus) at (3,-0.05) {\textcolor{blue4}{$x^\upplus$}};
        \node at (3.6,-0.5) {\textcolor{blue4}{$\cX^\upplus = \cX \cup \scbr{x^\upplus}$}};
        \draw[blue4, thick, ->] (x_plus) edge[loop above, looseness=10, out=60, in=120] (x_plus);

        \draw[red!90!black, dashed, ->] (x) to[bend right=20] (x_plus);
        \draw[red!90!black, dashed, ->] (p1) to[bend right=20] (x_plus);
        \draw[red!90!black, dashed, ->] (p2) to[bend right=30] (x_plus);
        \draw[red!90!black, dashed, ->] (p3) to[bend left=20] (x_plus);
        \draw[red!90!black, dashed, ->] (p4) to[bend left=20] (x_plus);

        \node[red!90!black] at (2.2,.8) {\Large $\cM^\upplus$};
    \end{tikzpicture}
    \caption{Illustration of the MDP $\cM$ in black and its extension in blue. The MDP $\cM^\upplus$ contains the 
additional red dashed edges that allow ascension to heaven.}
    \label{fig:illustration-mdps}
\end{figure}

\subsection{Online learning in linear MDPs}

We consider a variation of the MDP setup described above, incorporating two modifications: \emph{i)} periodic resets of the state evolution to the initial-state distribution $\nu_0$, and \emph{ii)} the ability of the environment to change the reward function adversarially after each reset. This is a natural adaptation\footnote{See Section~\ref{sec:conclusion} for a discussion of the role of resets and related online-learning settings.} of the well-explored setting of online learning in adversarial MDPs to the discounted-reward case we study in this work. More precisely, we consider the following sequential interaction process between the learning agent and its environment. The interaction proceeds through $T$ time steps, organized into $K$ episodes (of random length) as follows. The initial state drawn as $X_0 \sim \nu_0$ and then the following steps are repeated in every consecutive round $t = 0, 1, \dots, T$: 
\begin{itemize}
 \item The agent observes the state $X_t \in \cX$, 
 \item the agent chooses an action $A_t \in \cA$,
 \item the environment generates the next state $X_{t+1}' \sim P \spr{\cdot \middle| X_t, A_t}$,
 \item the environment selects a reward function $r_t$,
 \item the agent receives a reward $r_t \spr{X_t, A_t} \in \sbr{0, 1}$ and observes the function $r_t$,
 \item with probability $\gamma$, the process moves to the next state $X_{t+1} = X_{t+1}'$, otherwise a new episode begins and the process is reset to the initial-state distribution as $X_{t+1}\sim \nu_0$.
\end{itemize}
Without significant loss of generality, we will restrict the environment to update the reward function only at the end of each episode, and use $r_k$ to refer to the reward function within episode $k$. Other than this restriction, the environment is free to choose the rewards in an adaptive (and possibly adversarial) way. The objective for the agent is to select a sequence of policies $\pi_k$ so as to minimize its \emph{pseudo-regret} over $K$ episodes with respect to an arbitrary comparator policy $\pi^\star: \cX^\upplus \rightarrow \Delta \spr{\cA}$, given by \looseness=-1
\begin{equation*}
    \Reg_K = \sumkK \inp{\nu_0, V_{P, r_k}^{\pi^\star} - V_{P, r_k}^{\pi_k}} = \frac{1}{1 - \gamma} \sumkK \inp{\mu \spr{\pi^\star} - \mu \spr{\pi_k}, r_k}\,,
\end{equation*}
where the second equality follows from the definition of occupancy measures. Since the learning agent can only learn about the transition function via interaction with the environment, it needs to address the classic dilemma of exploration versus exploitation. Clearly, this setup generalizes the more standard problem formulation where $r_k = r$ holds for all episodes $k$. In this case, $\frac{\Reg_K}{K}$ corresponds to the expected suboptimality of the average policy played by the agent.

In later sections, we will consider the following structural assumption on the transitions.
\begin{assumption}[Linear MDP]\label{ass:LinMDP}
    A discounted MDP $\cM = \spr{\cX, \cA, r, P, \gamma, \nu_0}$ is a \emph{linear MDP} if there exist a known feature map $\phi: \cX \times \cA \rightarrow \bbR^d$, an unknown map $m: \cX \rightarrow \bbR^d$ and an unknown vector $w \in \bbR^d$ such that for any triplet $\spr{x, a, x'} \in \cX \times \cA \times \cX$,
    \begin{equation*}
        P \spr{x' \middle| x, a} = \inp{\phi \spr{x, a}, m \spr{x'}}, \quad r \spr{x, a} = \inp{\phi \spr{x, a}, w}\,.
    \end{equation*}
    We will also use the operators $\phim: \mathbb{R}^d \rightarrow \real^{\X\times\A }$ such that $\pa{\Phi\theta}\pa{x,a} = \iprod{\theta}{\varphi(x,a)}$ holds for any $\theta\in\real^d$ and $M: \real^\X \rightarrow \mathbb{R}^d$ such that $\pa{Mf}_i = \sum_{x' \in \cX} f(x') m_i(x')$. Thus, we can write the transition operator and the reward function as $P=\phim M$ and $r = \phim w$. Moreover, we assume that $\norm{w}\leq \WMAX$ and that for all $x,a\in\X\times\A$, the features have bounded norm, \ie $\norm{\phi(x,a)}\leq B$.
\end{assumption}

\paragraph{Further notation.} We will use $\piunif$ to denote both the uniform probability distribution over $\cA$ and the policy that plays uniformly at random at every state. $\Delta (A)$ denotes the simplex over a discrete set $A$. Given two distributions $p,q \in \Delta(\cZ)$ on the countable set $\cZ$, we denote the Kullback--Leibler divergence as $\DDKL{p}{q} = \sum_{z\in\cZ} \log \brr{\frac{p(z)}{q(z)}} p(z)$ and we use the convention that $\DDKL{p}{q}= +\infty$ whenever there exists an element $z\in \cZ$ such that $ q(z) = 0$ and $p(z) > 0$. For a distribution $p\in\Delta(\cZ)$ and a function $f\in\real^{\cZ}$, we will use the notation $\iprod{p}{f} = \EEs{f(Z)}{Z\sim p}$.

%% file: sections/algorithm.tex
\section{Algorithm}
\label{sec:algo}

Our algorithm implements the principle of \emph{optimism in the face of uncertainty} (OFU), by combining two classic ideas for optimistic exploration in reinforcement learning. We refer to these two separate mechanisms as two \emph{degrees of optimism}, with \emph{first-degree optimism} defined using the idea of exploration bonuses added to the rewards, and \emph{second-degree optimism} leveraging the notion of optimistically augmented MDPs defined in Section~\ref{sec:OAMDP}. These two incentives for exploration are respectively inspired by the upper-confidence-bound (UCB) methods popularized by \cite{azar2017minimax}, and the classic \RMAXalg algorithm of \citet{brafman2002r}. These two mechanisms are combined with the regularized approximate dynamic programming method of \citet{MN23}, called ``regularized approximate value iteration with upper confidence bounds'' (\raviUCB).

\subsection{Overview}

We begin by describing each element of our solution in general terms, and provide the pseudocode of the resulting algorithm (specifically tailored to linear MDPs) as Algorithm~\ref{alg:linear-rmax-ravi-ucb}. In recognition of the influence of the two algorithms mentioned above, we refer to our method as \algname.

\paragraph{Regularized dynamic programming.} If the transition kernel $P$ were known, the learner could achieve low regret by deploying the following regularized value iteration (RVI) scheme:
\begin{align} \label{eq:DPupdate}
    Q_{k+1} = r_k + \gamma P V_k, ~~~ V_{k+1}(x) = \max_{u\in\Delta_{\A}} \ev{\innerprod{u}{Q_{k+1}(x,\cdot)} - \frac{1}{\eta} \DDKL{u }{\pi_k(\cdot|x)}},
\end{align}
and update its policies as $\pi_{k+1}(x|a) \propto \pi_{k}(x|a)e^{\eta Q_{k+1} (x, a)}$ for some positive learning-rate parameter $\eta > 0$. As observed by \citet{MN23}, the regularization in the policy updates is helpful for controlling the difference between consecutive policies and occupancy measures, thus addressing a major challenge that one faces when analyzing approximate DP methods in infinite-horizon MDPs. Unfortunately, $P$ is unknown and needs to be estimated. The estimation error introduced in this process is taken care of by the two degrees of optimistic adjustments we describe next.

\paragraph{First-degree optimism.} Our method will make use of a (possibly implicitly defined) sequence of estimates of the transition operator $\wh{P}_k: \real^{\cX}\ra \real^{\cX\times\cA}$, and an associated sequence of \emph{exploration bonuses} $\CB_k: \cX\times\cA \ra \real$. We say that the sequence of bonuses is \emph{valid} if it satisfies $\abs{\pa{\bpa{\wh{P}_k - P} V_k}(x,a)} \le \CB_k(x,a)$ holds for all value-function estimates $V_k$ calculated by the algorithm, simultaneously for all $x,a$ and $k$. Using this property, a key idea in our algorithm is to use $\wh{P}_k V_k + \CB_k$ as an upper confidence bound on $P V_k$, therefore providing an optimistic estimate of the ideal value-function updates~\eqref{eq:DPupdate}.

\paragraph{Second-degree optimism.} Unfortunately, relying only on first-degree optimism as defined above may result in value estimates that grow without bounds, thus leading to unstable policy updates. In order to prevent this unbounded growth, we employ the idea of optimistically augmented MDPs (defined in Section~\ref{sec:OAMDP}), with the ascension function defined as $p_k^\upplus(x,a) = \sigma \spr{\alpha \CB_k \spr{x, a} - \omega}$, where $\sigma: z\mapsto \frac{e^{z}}{1+e^{z}}$ is the sigmoid function and $\alpha > 0$ and $\omega > 0$ are positive hyperparameters. Technically, this is implemented by defining our action-value updates for each $x,a$ as
\begin{equation*}
    Q_{k+1} \spr{x, a} = \spr{1 - p_k^\upplus \!\!\spr{x, a}} \spr{r_k \spr{x, a} + \CB_k \spr{x, a} + \gamma \wh{ P}_k V_k \spr{x, a}} + p_k^\upplus \!\!\spr{x, a} \frac{\RMAX}{1 - \gamma}\,.
\end{equation*}
This adjustment makes sure that the value estimates remain bounded, thanks to the multiplicative effect of the ascension function that effectively trades off the possibly huge values of $\CB_k + \gamma\wh{P}_k V_k$ by the constant upper bound $\RMAX/(1-\gamma)$ in highly uncertain state-action pairs. Specifically, the effective bonus $\spr{1 - p_k^\upplus} \odot \CB_k$ is bounded \emph{independently} of the magnitude of the value estimates $V_k$, which is crucial for preventing the exponential growth of their magnitude with $k$. Additionally, supposing that $\CB_k$ is a valid sequence of exploration bonuses, one can verify that the inequality $Q_{k+1}  \ge r_k^\upplus + \gamma P_k^\upplus V_k$ holds elementwise---that is, the estimates $Q_k$ provide upper bounds on the ideal action-value updates~\eqref{eq:DPupdate} defined in the optimistically augmented MDP.

\subsection{Technical details}

To complete the outline given above, we specify the missing technical details specific to linear MDPs, provide a pseudocode sketch in \Cref{alg:linear-rmax-ravi-ucb}, and a detailed algorithm in Appendix~\ref{app:pseudocodes}.

\paragraph{Model estimation and bonus design.} For estimating the transition model and defining the exploration bonuses, we follow the classic approach of \citet{jin2019provably} (see also \citealp{Neu:2020}). Specifically, we define the least-squares model estimate $\wh{P}_k = \Phi \wh{M}_k$ as the operator that maps a function $v\in\real^\cX$ to $\phim \wh{M_k v}\in \real^{\cX\times\cA}$. The vector $\wh{M_k v}\in\real^d$ is the solution to the least-squares regression problem with features $\bc{\phi(X_t,A_t)}^{T_k - 1}_{t=1}$ and targets $\bc{v(X_t)}^{T_k}_{t=2}$, where $T_k$ denotes the beginning of episode $k$. The problem (with target $v = V_k$) admits the closed form solution given in line~\ref{line-alg:ridge} of \Cref{alg:linear-rmax-ravi-ucb}. The matrix $\Lambda_{T_k}= \sum^{T_k}_{t=1} \phi(X_t,A_t)\phi(X_t,A_t)\transpose + I$ used in the computation of the least-squares model estimate is called the \emph{empirical covariance matrix}. Finally, we define the \emph{exploration bonuses} for each $(x,a)\in\cX\times\cA$ as $\CB_k(x,a) = \beta \norm{\phi(x,a)}_{\Lambda^{-1}_{t_e}} = \beta \sqrt{\biprod{\phi(x,a)}{\Lambda^{-1}_{t_e}\phi(x,a)}}$, where $\beta>0$ will be chosen during the analysis to be large enough to guarantee bonus validity, and $t_e$ is the time step marking the beginning of the $e$th epoch (see below). The state $x^\upplus$ is given special treatment: for all actions $a$, we fix $\CB_k(x^\upplus,a) = 0$, $p^\upplus_k(x^\upplus,a) = 0$, and $Q_k(x^\upplus,a) = V_k(x^\upplus) =\frac{\RMAX}{1-\gamma}$.

\paragraph{Bookkeeping.} In order to turn the above ideas into a tractable algorithm amenable to theoretical analysis, a few additional bookkeeping steps are necessary. The most important of these is the introduction of an \emph{epoch schedule}, which is instrumental in keeping the complexity of the exploration bonuses and the policies low. Using a classic trick from \citet{APS11}, a new epoch is started every time that there is a significant reduction in the uncertainty of the model estimates (as measured by the determinant of the empirical covariance matrix of the features in the case of linear MDPs---see line~\ref{line-alg:epoch} in Algorithm~\ref{alg:linear-rmax-ravi-ucb}). When a epoch change is triggered, the exploration bonuses are recomputed and the policy is reset to a uniform policy.

\begin{algorithm}[!h]
    \caption{\algname for Linear MDPs.}
      \begin{algorithmic}[1]
      \label{alg:linear-rmax-ravi-ucb}
          \STATE {\bfseries Inputs:} Number of resets $K$, learning rate $\eta > 0$, exploration coefficient $\beta > 0$, threshold $\omega > 0$, slope sigmoid $\alpha > 0$, $\pi_1 = \piunif$, $Q_1 = 0$, $\cD_1 = \emptyset$, $\Lambda_1 = I$, $t = 1$, $e = 0$.
          \FOR{$k = 1, \dots, K$}
          \STATE \algcomment{interact with the environment}
          \STATE The adversary adaptively chooses $r_k$.
          \STATE Rollout $\pi_k$ to collect a new trajectory, $\cD_{T_k} = \cD_{T_{k-1}} \cup \spr{X_t, A_t, X_{t+1}}^{T_k-1}_{t=T_{k-1}+1}$.
          \STATE Update the covariance matrix $\Lambda_{T_k} = \Lambda_{T_{k-1}} + \sum^{T_k-1}_{t=T_{k-1}+1}\phi \spr{X_t, A_t} \phi \spr{X_t, A_t}\transpose$.
          \STATE \algcomment{initialize new epoch}
          \IF{$t = T_1$ \OR $\det \Lambda_{T_k} \geq 2 \det \Lambda_{t_e}$} \label{line-alg:epoch}
          \STATE Set $e = e + 1$, $k_e = k$ and $t_e = t$. Reset the policy $\pi_k = \piunif$.
          \ENDIF
          \STATE For any $\spr{x, a} \in \cX \times \cA$, $\CB_k \spr{x, a} = \beta \norm{\phi \spr{x, a}}_{\Lambda_{t_e}^{-1}}\mathds{1}\bc{x \neq x^\upplus}$. \label{line-alg:bonuses}
          \STATE For any $\spr{x, a} \in \cX \times \cA$, $p_k^\upplus \spr{x, a} = \sigma \spr{\alpha \CB_k \spr{x, a} - \omega}\mathds{1}\bc{x \neq x^\upplus}$. \label{line-alg:pplus}
          \STATE \algcomment{optimistic regularized value iteration}
          \STATE $r_k^\upplus = \spr{1 - p_k^\upplus} \odot r_k + p_k^\upplus \cdot \RMAX$.
          \STATE $\wh{M V_k} = \Lambda_{T_k}^{-1} \sum_{\spr{x, a, x'} \in \cD_{T_k}} \phi \spr{x, a} V_k \spr{x'}$. \label{line-alg:ridge}
          \STATE $\wh{P_k^\upplus V_k} = \spr{1 - p_k^\upplus} \odot \phim \wh{M V_k} + p_k^\upplus \cdot V_k \spr{x^\upplus}$,\;\;and\; $\wh{P_k^\upplus V_k} \spr{x^\upplus, \cdot} = \frac{\RMAX}{1 - \gamma}$. \label{line-alg:transf}
          \STATE $Q_{k + 1} = r_k^\upplus + \spr{1 - p_k^\upplus} \odot \CB_k + \gamma \wh{P_k^\upplus V_k}$.
          \STATE $V_{k+1} \spr{x} = \frac1\eta \log \spr{\sum_a \pi_{k} \spr{a \given x} e^{\eta Q_{k+1} \spr{x, a}}}$.
          \STATE $\pi_{k+1} = \pi_{k} \odot e^{\eta \spr{Q_{k+1} - E V_{k+1}}}$.
          \ENDFOR
          \STATE {\bfseries Output:} $\pi_I$, with $I \sim \cU \spr{\sbr{K}}$.
      \end{algorithmic}
    \end{algorithm}

\subsection{Discussion}

Before moving to the the analysis, we highlight some important features of our method.

\paragraph{Computational and storage complexity.} Due to its design outlined above (and detailed in \Cref{alg:linear-rmax-ravi-ucb}), \algname produces a sequence of policies that are simply parameterized by a $d$-dimensional vector and a $d^2$-dimensional covariance matrix. Specifically, this is made possible by keeping the bonus function fixed within each epoch and resetting the policy to uniform at the beginning of each new epoch. Therefore, storing the policies in memory and drawing actions in each new state can be both done efficiently. This should be contrasted with most other known algorithms making use of softmax policies, which crucially rely on clipped value-function estimates that cannot be stored or sampled from efficiently (as they require storing the entire history of parameter vectors and exploration bonuses). Examples of such methods include \cite{Cai:2020,zhong2024theoretical,sherman2023,MN23}. This not only makes the implementation of these algorithms impractical, but also results in suboptimal regret guarantees due to the excessive complexity of the policy and value-function classes. This major improvement is made possible in our algorithm by the incorporation of second-degree optimism (inspired by both \citealp{brafman2002r} and \citealp{cassel2024warmupfree}), which obviates the need for explicit clipping of the value estimates and keeps these bounded via alternative means. All other elements in our algorithm (such as estimating the value estimates via least-squares regression) are standard, and match the complexity of other efficient methods for online learning in linear MDPs \cite{jin2019provably,wang2021provably,he2023nearly}.

\paragraph{Relation with existing algorithms.} Being a combination of \RMAXalg and \raviUCB, our algorithm can recover these two extremes and several other known methods by an appropriate choice of hyperparameters. Setting $\CB_k = 0$, we recover algorithms which leverage only \emph{second-degree} optimism. For example, in the particular case of $\alpha = \infty$ and tabular features, we recover \RMAXalg up to some very minor changes \citep{brafman2002r,szita2010model}. On the other hand, using the ascension function $p^\upplus_k(x,a) = \mathds{1}\bc{r_k \spr{x, a} + \CB_k \spr{x, a} + \gamma \wh{ P}_k V_k \spr{x, a} \ge \frac{\RMAX}{1-\gamma}}$ essentially recovers the standard truncation rule applied by most related methods (under the condition that the exploration bonuses be valid). In particular, under this choice of $p^\upplus_k$, our algorithm reduces to \raviUCB. Setting the regularization parameter $\eta=\infty$ in the resulting method recovers optimistic value iteration methods such as \cite{azar2017minimax,jin2019provably}.

%% file: sections/analysis.tex
\section{Main Result and Analysis}
\label{sec:analysis}

The following theorem states our main result about the performance of \algname.
\begin{theorem} \label{thm:main}
    Suppose that Assumption~\ref{ass:LinMDP} holds, and that Algorithm~\ref{alg:linear-rmax-ravi-ucb} is executed with parameters specified in Appendix~\ref{app:putting-together-main} for a fixed number $K$ of episodes. Then, with probability at least $1 - \delta$,
    \begin{equation*}
        \regretK = \tilde\cO \spr{\sqrt{d^3 H^3 K} + \sqrt{d H^{9/2} K \log \spr{\abs{\cA}}}}\,.
    \end{equation*}
\end{theorem}
For the commonly studied case where the reward function $r$ is fixed, this result can be easily translated to a bound on the sample complexity of producing an $\varepsilon$-optimal policy as well, as stated below.
\begin{corollary} \label{cor:sample}
    Let $I$ be drawn uniformly from $\ev{1,2,\dots,K}$. Then, policy $\pi_I$ produced by \Cref{alg:linear-rmax-ravi-ucb} (with the same parameter tuning as in Theorem~\ref{thm:main}) satisfies $\EEs{V^{\pi^*}_{P,r} - V^{\pi_I}_{P,r}}{I} \le \varepsilon$, if the number of episodes satisfies
    \begin{equation*}
        K = \Omega\pa{\frac{H^3 d^3 + H^{9/2} d \log \abs{\cA}}{\varepsilon^2}}\,.
    \end{equation*}
\end{corollary}
As the length of each episode is geometrically distributed with expectation $H$, the number of interaction steps satisfies $\EE{T} = HK$ when the number of episodes $K$ is fixed. Taking this into account, both results can be restated in terms of the number of sample transitions $T$. Likewise, similar results can be proved when treating the sample size $T$ as fixed and letting $K$ be the smallest (random) number of episodes covering the sample budget.
The sample complexity bound in Corollary~\ref{cor:sample} is rate-optimal in the sense that the dependence on $\varepsilon$ cannot be improved in virtue of the $\Omega(H^3 X A \epsilon^{-2})$ lower bound in \cite{azar2012sample}. Notice their lower bound is proven for the setting of generative model access to the environment which is strictly easier than the trajectory access model we study.

In the remaining part of this section, we describe the main steps constituting the proof of Theorem~\ref{thm:main}. The analysis will make crucial use of the notion of optimistically augmented MDPs defined in Section~\ref{sec:OAMDP}. Specifically, we define an augmented MDP for each episode $k$ as $\cM^{\upplus}_k = \cM^+(r_k,p_k^\upplus)$ and use the shorthand $\cM_k = \cM(r_k)$ for the true MDP with reward function $r_k$. Letting $\mu_k^\upplus(\pi)$ denote the occupancy measure induced by policy $\pi$ in $\cM_k^\upplus$, the first step in our analysis is to rewrite the regret as follows:
\begin{align*}
    \sumkK \inp{\mu \spr{\pi^\star} - \mu \spr{\pi_k}, r_k} &= \sumkK \underbrace{\inp{\mu \spr{\pi^\star} - \mu \spr{\pi_k}, r_k - r_k^\upplus}}_{= \rewardbias_k} + \sumkK \underbrace{\inp{\mu \spr{\pi^\star} - \mu_k^\upplus \spr{\pi^\star}, r_k^\upplus}}_{= - \modelbias_k \spr{\pi^\star}} \\
    &\phantom{=}+ \underbrace{\sumkK \inp{\mu_k^\upplus \spr{\pi^\star} - \mu_k^\upplus \spr{\pi_k}, r_k^\upplus}}_{= \regretKplus} + \sumtT \underbrace{\inp{\mu_k^\upplus \spr{\pi_k} - \mu \spr{\pi_k}, r_k^\upplus}}_{= \modelbias_k \spr{\pi_t}}\,.
\end{align*}
Here, $\regretKplus$ corresponds to the (normalized) regret of \raviUCB in the sequence of OA-MDPs $\pa{\cM_k^\upplus}$, and the other terms account for the difference between this term and the regret in the original problem. Some of these are easy to handle by exploiting the optimistic nature of our augmentation technique, whereas others require a more careful tuning of the ascension functions and exploration bonuses.

A large part of the analysis will be based on the condition that the exploration bonuses $\scbr{\CB_k}_k$ used for performing optimistic value iteration are \emph{valid} estimates of the uncertainty we have on the model, in the sense that the following event holds
\begin{equation} \label{eq:validity-bonuses}
    \cE_{\text{valid}} = \scbr{\forall \spr{x, a} \in \cX \times \cA, \forall k \in \sbr{K}, \abs{\spr{\opP - \ophPk} V_k \spr{x, a}} \leq \CB_k \spr{x, a}}\,.
\end{equation}
For the sake of clearly presenting the main ideas of our analysis, we will work under the condition that the event $\cE_{\text{valid}}$ holds, and we will show later in the context of linear MDPs that this is indeed true with high probability. Furthermore, we will work under the condition that for any $k \in \sbr{K}$, there exists $\QMAX > 0$ such that $\norm{Q_k}_\infty \leq \QMAX$. We will show this is true under an appropriate choice of $p_k^\upplus$, and give an expression for $\QMAX$. Finally, we assume the episode lengths $\spr{L_k}$ are all less than some $\LMAX$ and refer to this event as $\cE_L$. We will show that this also holds with high probability.

\subsection{Controlling the bias due to the augmented rewards}

The first lemma controls the bias introduced by the second-degree optimism introduced into the reward function in terms of the ascension functions.
\begin{restatable}{Lem}{rewardbiasbound} \label{lem:reward-bias-bound}
    For any choice of $p_k^\upplus$, we have $\rewardbias_k \leq \RMAX \inp{\mu \spr{\pi_k}, p_k^\upplus}$.
\end{restatable}
The proof essentially follows from the definition of $r_k^\upplus$, and is found in Appendix~\ref{app:reward-bias-bound}.

\subsection{Controlling the bias due to the augmented transitions}

Next, we provide a bound on the bias due to introducing second-degree optimism in the transition function. It is easy to see that playing in the augmented MDP $\cM_k^\upplus$ always yields higher discounted return due to the presence of the heaven state $x^\upplus$. On the other hand, one can intuitively see that the difference in the bias term can be upper bounded by the amount of time spent in heaven $x^\upplus$. The following lemma formalizes this claim by providing a bound on $\modelbias_k \spr{\pi} = \inp{\mu_k^\upplus \spr{\pi} - \mu \spr{\pi}, r_k^\upplus}$ for any policy $\pi$.
\begin{restatable}{Lem}{modelbiasbounds} \label{lem:model-bias-bounds}
    Let $\pi$ be any policy and $p_k^\upplus$ be any ascension function at episode $k$. Then,
    \begin{equation*}
        0 \leq \modelbias_k \spr{\pi} \leq \frac{\RMAX}{1 - \gamma} \inp{\mu \spr{\pi}, p_k^\upplus}\,.
    \end{equation*}
\end{restatable}
Both bounds are proved through a coupling argument, provided in Appendix~\ref{app:model-bias-bounds}. Applying the lower bound to $\pi^\star$ and the upper bound to $\pi_k$, we obtain
\begin{equation} \label{eq:model-bias-bounds}
    - \sumkK \modelbias_k \spr{\pi^\star} \leq 0\,, \text{ and}\;\; \sumkK \modelbias_k \spr{\pi_k} \leq \frac{\RMAX}{1 - \gamma} \sumkK \inp{\mu \spr{\pi_k}, p_k^\upplus}\,.
\end{equation}

\subsection{Regret analysis in the optimistically augmented MDP}

To control the main term $\regretKplus$, we adapt the analysis of \raviUCB due to \citet{MN23} with some appropriate changes. The key idea is to define an estimate $\ophPplusk$ of the optimistically augmented transition operator $\opPplusk$ associated with $\cM^\upplus_k$, via its action on a function $v\in\real^\X$:
\begin{equation*}
    \pa{\wh{P}^\upplus_k v}\pa{x,a} = \spr{1 - p_k^\upplus \spr{x, a}} \cdot \pa{\wh{P}_k v}{\pa{x,a}} + p_k^\upplus \spr{x, a}\cdot \frac{\RMAX}{1-\gamma}\,.
\end{equation*}
Then, it is easy to verify that the validity of the exploration bonuses (Eq.~\ref{eq:validity-bonuses}) implies that the scaled bonuses $\spr{1 - p_k^\upplus} \odot \CB_k$ satisfy the following analogous validity condition in the augmented MDP:
\begin{equation*}
    \abs{\spr{\opPplusk - \ophPplusk} V_k \spr{x, a}} \leq \spr{1 - p_k^\upplus \spr{x, a}} \CB_k \spr{x, a}\,.
\end{equation*}
With these insights, our algorithm can be seen as an instantiation of \raviUCB on the sequence of optimistically augmented MDPs $\pa{\cM_k^\upplus}$, and thus it can be analyzed by following the steps of \citet{MN23}. In particular, the following lemma (an adaptation of Lemma~4.3 of \citet{MN23}, proved here in Appendix~\ref{app:bound-regret-plus}) gives a bound on the regret in the augmented MDP.
\begin{restatable}{Lem}{boundregretplus} \label{lem:bound-regret-plus}
    Suppose that the bonuses $\scbr{\CB_k}$ are valid in the sense of Equation~\ref{eq:validity-bonuses} and that for any $k$, $\norm{Q_k}_\infty \leq \QMAX$. Then, the sequence of policies output by Algorithm~\ref{alg:linear-rmax-ravi-ucb} satisfies
    \begin{equation} 
        \regretKplus \leq \frac{E \spr{K} \log \abs{\cA}}{\eta} + 4 \VMAX E \spr{K} + \frac{2 \QMAX^2 \eta K}{\sqrt{1 - \gamma}} + 2 \sumkK \inp{\mu \spr{\pi_k}, \spr{1 - p_k^\upplus} \odot \CB_k}\,,
    \end{equation}
    where $E(K)$ denotes the number of epochs epochs executed within $K$ episodes.
\end{restatable}

\subsection{Choosing the ascension functions}

It remains to verify that our choice of the probabilities $p_k^\upplus$ is such that the terms appearing in Lemmas~\ref{lem:reward-bias-bound} and Equation~\eqref{eq:model-bias-bounds} are small, yet the value of $\QMAX$ also remains bounded. 
In particular, we show in Lemma~\ref{lem:sigmoid-bound} that our choice satisfies
\begin{equation} \label{eq:ascension-function-bound}
    p_k^\upplus \spr{x, a} \leq 2 \alpha^2 \CB_k \spr{x, a}^2 + 2 e^{- \omega}\,.
\end{equation}
Furthermore, the following lemma (proved in in Appendix~\ref{app:qmax}) shows that the above choice for the ascention function leads to a suitable value of $\QMAX$.
\begin{restatable}{Lem}{qmaxlemma} \label{lem:qmax}
    Suppose the bonuses $\scbr{\CB_k}_{k \in \sbr{K}}$ are valid in the sense of Equation~\ref{eq:validity-bonuses} and the ascension functions are chosen as in Line~\ref{line-alg:pplus} of Algorithm~\ref{alg:linear-rmax-ravi-ucb}. Then, for any $k \in \sbr{K}$, the iterate $Q_k$ satisfies $\norm{Q_k}_\infty \leq \QMAX$ with $\QMAX = \frac{\RMAX + 2 \omega / \alpha}{1 - \gamma}$.
\end{restatable}

\subsection{Exploration bonuses}

The final technical step is to verify the validity of the exploration bonuses and to bound their cumulative size. The following lemma addresses the latter question.
\begin{restatable}{Lem}{expectedbonusesbound} \label{lem:expected-bonuses-bound}
    Suppose Assumption~\ref{ass:LinMDP} and event $\cE_L$ hold and denote $T = \LMAX K$. Then, with probability at least $1 - \delta$, the policies $\scbr{\pi_k}$ and bonuses $\scbr{\CB_k}$ satisfy
    \begin{equation*}
        \sumkK \inp{\mu \spr{\pi_k}, \CB_k} \leq 4 \spr{1 - \gamma} \beta B \sqrt{d T\log \spr{1 + \frac{B^2 T}{d}}} + 4 \beta B \log \spr{\frac{2 K}{\delta}}^2\,,
    \end{equation*}
    and
    \begin{equation*}
        \sumkK \inp{\mu \spr{\pi_k}, \CB_k^2} \leq 8 \spr{1 - \gamma} \beta^2 B^2 d \log \spr{1 + \frac{B^2 T}{d}} + 4 \beta^2 B^2 \log \spr{\frac{2 K}{\delta}}^2\,,
    \end{equation*}
\end{restatable}
For the proof, see Appendix~\ref{app:expected-bonuses-bound}. Finally, it remains to show that the events $\cE_{\text{valid}}$ and $\cE_L$ hold which is done in the following lemma, whose proof can be found in Appendix~\ref{app:good-event-holds}.
\begin{restatable}{Lem}{goodeventholds} \label{lem:good-event-holds}
    Let $\beta = 8 \QMAX d \log \spr{c \alpha \WMAX \RMAX B^{9/2} \QMAX^4 \LMAX^{5/2} K^{7/2} d^{5/2} \delta^{-1}}$ for some constant $c \in \mathbb{R}$. Then, the event $\mathcal{E}_{\text{valid}} \cap \mathcal{E}_L$ holds with probability $1 - 2 \delta$.
\end{restatable}

\subsection{Putting everything together}

Theorem~\ref{thm:main} then follows from applying Lemmas~\ref{lem:reward-bias-bound}-\ref{lem:good-event-holds}, using Equation~\eqref{eq:ascension-function-bound}, and bounding the total number of epochs (Lemma~\ref{lemma:number-epochs-bound}). The full details are provided in Appendix~\ref{app:putting-together-main}.

%% file: sections/applications.tex
\section{Application: Imitation Learning from features alone}
\label{sec:application}

In this section, we show an application of the results presented in Section~\ref{sec:analysis}, making crucial use of 
the fact that our main result in Theorem~\ref{thm:main} allows adaptively chosen reward functions. 

\subsection{Setting and motivation}

We consider a linear MDP with an unknown reward function $\true$ defined in terms of the feature map 
$\phi_r:\cX\times\cA\ra\real^{d_r}$ as $\true(x,a) = \innerprod{\phi_r(x,a)}{w_{\mathrm{true}}}$ or with the operator notation $\true = \Phi_r w_{\mathrm{true}} $. The transition 
function is defined using the feature map $\phi_P: \cX\times\cA\ra\real^{d_P}$ via $P(x'|x,a)=\innerprod{\phi_P(x,a)}{ 
m(x')}$. It is easy to see that this is a linear MDP in terms of the concatenated feature map of dimension $d = d_r + 
d_P$. We consider a learning problem that we call \emph{imitation learning from features alone}, where a learner 
receives as input a data set of feature vectors $\mathcal{D}_\expert = \bc{\phi_{\cost}(X^i_E,A^i_E)}^{\tau_E}_{i=1}$ 
where $X^i_E,A^i_E \sim \mu(\expert)$ generated by an \emph{expert policy} $\expert$ by interacting with the MDP 
$\cM(\true)$. The learner is tasked with producing an $\varepsilon$-suboptimal policy $\pi^{\mathrm{out}}$ that 
satisfies $\mathbb{E}\bigl[\biprod{\initial}{V_{\true}^{\expert} - V_{\true}^{\pi^{\mathrm{out}}}}\bigr] \leq 
\varepsilon$. The learner has no knowledge of the reward function $\true$ apart from knowing a function class including 
it, but has access to the feature maps and can interact with the MDP.

This framework captures the important special case where the rewards depend only on the state but not on the action, 
by choosing a feature map $\varphi_r$ that only depends on the states. In this case, the data set taken as input is 
significantly less informative than a full record of states and actions $\mathcal{D}^{\mathrm{x,a}}_\expert = 
\bc{(X^i_E, A^i_E)}^{\tau_E}_{i=1}$. However, observing expert actions has been found restrictive in various practical 
scenarios, which gives a strong motivation to study the setting described above.
%
For instance in robotics, a features only dataset describing a robotic manipulation task 
can be collected easily 
via cameras and sensors \citep{torabi2018generative,zhu2020off,yang2019imitation,torabi2019recent}.
On the other hand, collecting actions on top of features is more challenging as it 
requires knowledge of the internal dynamics of the observed robots. 
Another example of imitation learning from features alone is learning to drive from a video 
which does not show the driver's actions but only the movements of the car induced by those actions.
Finally, notice that imitation learning from states only, studied for example in \cite{sun2019provably}, is a particular 
case of our setting for $\phi_r(x,a) = \mathbf{e}_x$. In this case, the expert dataset consists of states sampled from 
the expert state occupancy measure.

\subsection{Algorithm and sample complexity guarantees} 
In the following, we propose a provably efficient algorithm for imitation learning from features alone.
Our algorithm design and analysis is driven by the following decomposition of the regret, defined as $\regretIL = \sum^K_{k=1}\innerprod{\initial}{V^{\expert}_{P,\true} - V^{\pi_k}_{P,\true}}$, in terms of an appropriately chosen sequence of reward functions $r_1,\dots,r_K$:
\begin{equation*}
    (1 - \gamma) \regretIL = \sum^K_{k=1} \innerprod{r_k}{\mu\brr{\expert} - \mu\brr{\pi_k}} + \sum^K_{k=1} \innerprod{\Phi_r\transpose\mu(\pi_k) - \Phi_r\transpose\mu(\expert)}{w_k - w_{\mathrm{true}}}.
\end{equation*}
The first term in this decomposition corresponds to the regret of our online learning algorithm for adversarial MDPs, and can be controlled by invoking \algname on the sequence of rewards $\pa{r_k}$. The second term in the decomposition corresponds to the regret of another online learning algorithm picking a sequence of reward functions, aiming to minimize the sequence of linear loss functions $\Phi_r\transpose\mu(\pi_k) - \Phi_r\transpose\mu(\expert)$ (or at least do as well as the fixed comparator $w_{\mathrm{true}}$). This objective can be achieved by running a standard online learning method such as projected online gradient descent (OGD, \citealp{Zin03}), using a sequence of  unbiased loss estimates that can be computed efficiently using the observed feature vectors. 
The full algorithm is specified as \Cref{alg:fra} in Appendix~\ref{app:pseudocodes} and it is shown to satisfy the following guarantees.
\begin{restatable}{theorem}{FraUpper} \label{thm:FraUpper}
  Algorithm~\ref{alg:fra}, when run for $K = \tilde{\mathcal{O}}\brr{d^{3} H^{9/2} \varepsilon^{-2}\log (\abs{\aspace})}$ iterations with an expert dataset of size $\tau_E = \widetilde{\mathcal{O}}\brr{ \WMAX^2 H^2\varepsilon^{-2}}$, outputs an $\varepsilon$-suboptimal policy.
\end{restatable}
Notably, the above is the first known bound for this setting that achieves 
a scaling $\varepsilon^{-2}$ with the precision parameter for both the number of MDP interactions $K$ and expert samples $\tau_E$. 
We provide a detailed comparison with existing imitation learning theory works in Appendix~\ref{app:related_works_IL}, 
whereas the complete technical details supporting the above theorem are provided in Appendix~\ref{app:proof_IL}.

\subsection{Lower bounds}
The upper bound of Theorem~\ref{thm:FraUpper} depends both on the number of interaction steps $K$ and the expert 
samples $\tau_E$. It is natural to ask if these dependences can be improved in the setting we consider. We address both 
of these questions in the negative in a set of lower bounds described below.


\paragraph{Lower bound on the number of MDP interactions $K$.}
Theorem~\ref{thm:LowerK} in Appendix~\ref{app:lower} proves 
that for any imitation learning from features alone algorithm, even in the setting $\tau_E = \infty$,
there exists an MDP and an expert policy where at least $K=\Omega\brr{\frac{d H^2}{\varepsilon^2}}$ interactions 
are needed to output a $\varepsilon$-optimal policy.
This lower bound shows that the upper bound provided for $K$ in Theorem~\ref{thm:FraUpper}
achieves the optimal scaling in $\varepsilon$ and can be improved at most by a factor $ d^2 H^{5/2}$.
More importantly, this lower bound marks a clear separation between standard and features only imitation learning: 
purely offline learning ($K=0$) is impossible in imitation learning from features alone
while it is possible in standard imitation learning 
where the experts actions are observed (see, e.g., \citealp{foster2024behavior}).

In the construction of the lower bound, we consider a two state MDP with a reward 
function that depends only on the state variable and we set $\phi_r(x,a) = \mathbf{e}_x$ which prevents 
observing expert actions. 
In this MDP, for $\tau_E=\infty$, the learner observes the expert state occupancy 
measure exactly and 
therefore, the ``good'' state that achieves the maximum of the expert state occupancy measure.
However, since the learner does not know the MDP dynamics, interactions with the MDP 
are needed to find out the action that allows to maximize the learner state occupancy measure 
in the ``good'' state.
Following standard techniques in MDP and bandits lower bound, we can ensure that the amount 
of MDP interactions is at least $\Omega\brr{\varepsilon^{-2}}$.

\paragraph{Lower bound on the number of expert samples $\tau_E$.}
Theorem~\ref{thm:LowerTauE} in Appendix~\ref{app:lower} proves 
that for any algorithm for imitation learning from features alone, even for $K=\infty$,
there exists an MDP and an expert policy where learning an $\varepsilon$-optimal policy requires at least 
$\tau_E =\Omega\brr{\WMAX^2 H^2\varepsilon^{-2}}$ expert samples.
This lower bound proves that Algorithm~\ref{alg:fra} scales optimally with all the problem parameters 
and in the precision $\varepsilon$. Moreover, this lower bound highlights again a clear distinction with standard 
imitation learning.
On the one hand, standard imitation learning can be reduced to a supervised 
classification problem when the optimal policy
is deterministic and the actions are observed in the dataset. As a consequence,
the classic lower bound for supervised classification of order $\mathcal{O}\brr{\varepsilon^{-1}}$ \citep{shalev2014understanding}
holds and it is matched by a purely offline behavioural cloning \citep{rajaraman2020toward}.
On the other hand, in our lower bound construction we choose again $\phi_r(x,a) =\mathbf{e}_x$
to make the expert actions unobservable for the learner and we can prove a larger lower bound of order $\mathcal{O}\brr{\varepsilon^{-2}}$
which holds even if the expert policy is deterministic.

For the proof, our construction is again a two state MDP (a ``good'' state with high reward and a ``bad'' state with 
lower reward). The expert policy is chosen to be the optimal one.
The transition dynamics and the initial distribution are chosen in a way that the expert 
state occupancy meausure is only marginally higher in the ``good'' state. That is, the 
expert occupancy measure equals roughly $1/2 + \varepsilon$ in the ``good'' state 
and  $1/2 - \varepsilon$ in the ``bad'' state.
By standard arguments, we then conclude that the learner needs at least $\Omega\brr{\varepsilon^{-2}}$
samples from the expert to identify the ``good'' state.

%% file: sections/conclusion.tex
\section{Concluding remarks}\label{sec:conclusion}

We close by discussing a few open problems and potential improvements to our results.

\paragraph{On the objective function.} We have focused on a relatively under-studied objective function: the discounted return from a fixed initial-state distribution. This is different from the objectives studied by other works such as \citet{liu2020regret,he2021nearly,zhou2021provably}, but arguably more natural when one is interested in learning algorithms that produce a single near-optimal policy at the end of an interactive learning period (which is the case in most practical applications one can think of). It is easy to see that resets to the initial state are absolutely necessary in this setting, unless one wants to make strong assumptions about the transition dynamics. A more exciting question is if our algorithm can be adapted to the significantly more challenging setting of undiscounted infinite-horizon reinforcement learning where existing methods \citep{WJLJ20,hong2024provably,he2024sample} either obtain suboptimal regret bounds or leverage oracles whose computationally efficient implementation is unknown. So far, our attempts towards tackling this problem have remained unsuccessful. We believe that significant new ideas are necessary for solving this major open problem, but also that the techniques we introduce in this paper will be part of an eventual solution.\looseness=-1

\paragraph{On second-degree optimism.} Our key technical contribution draws inspiration from two sources: the \RMAXalg algorithm of \citet{brafman2002r}, and the very recent work of \citet{cassel2024warmupfree}. While many of our technical tools are directly imported from the latter work, the concept of optimistically augmented MDPs and the connection with \RMAXalg has arguably brought about a new level of understanding that can be potentially valuable for future work. It has certainly proved useful for our setting, where the notion of ``contracted sub-MDP'' used in the analysis of \citet{cassel2024warmupfree} cannot be meaningfully interpreted and used for analysis. We hope our work can bring some fresh attention to older (but apparently still powerful) ideas from the past of RL theory such as the \RMAXalg trick.

\paragraph{On the tightness of the bounds.} We find it very likely that our performance guarantees can be improved to some extent in terms of their dependence on $H$, $d$ and $\log \abs{\A}$. In fact, we have made no attempt to optimize the scaling with respect to these parameters, and actually believe that the $H^{9/4}$ factor in the regret bound can be improved relatively easily. Specifically, we find it very likely that performing several value-function and policy updates at the end of each episode can reduce this factor---but we opted to keep the algorithm simple and the paper easy to read. We invite future researchers to verify this conjecture. Likewise, we believe that the dependence on the dimension $d$ can be improved by using more sophisticated estimators and concentration inequalities (as done in the finite-horizon setting by \citealp{he2023nearly,agarwal2023vo}), but leave working out the (possibly non-trivial) details for another paper.

%% file: sections/omitted_pseudocodes.tex
\section{Omitted pseudocodes}
\label{app:pseudocodes}
This section includes the pseudocode for \algname.
Each of the steps is explained in details in \Cref{sec:algo}.
\begin{algorithm}[!h]
\caption{\algname for Linear MDPs.}
  \begin{algorithmic}[1]
  \label{alg:linear-rmax-ravi-ucb-full}
      \STATE {\bfseries Inputs:} Number of resets $K$, learning rate $\eta > 0$, exploration coefficient $\beta > 0$, threshold $\omega > 0$, slope sigmoid $\alpha > 0$.
      \STATE {\bfseries Initialize:} $X_1 \sim \nu_0$, $\pi_1 = \piunif$, $Q_1 = 0$, $\cD_1 = \emptyset$, $\Lambda_1 = I$, $t = 1$, $e = 0$.
      \FOR{$k = 1, \dots, K$}
        
      \STATE \vspace{3pt}
      \algcomment{interact with the environment}
      \STATE The adversary adaptively chooses $r_k$, i.e.  $r_k = \textsc{RewardUpdate} \spr{\scbr{\pi_\ell}^k_{\ell=1}, \scbr{r_\ell}^{k-1}_{\ell=1}}$.
      \WHILE{\TRUE}
      \STATE Observe the state $X_t$ and play an action $A_t \sim \pi_k \spr{\cdot \given X_t}$.
      \STATE Receive reward $r_k \spr{X_t, A_t}$ and observe the function $r_k$.
      \STATE With probability $1 - \gamma$, \reset to initial distribution: $X_{t+1} \sim \nu_0$ set $T_k = t$ and \textbf{break} .
      \STATE Otherwise observe the next state $X_{t+1} \sim P \spr{\cdot \given X_t, A_t}$.
      \STATE Update $\Lambda_{t + 1} = \Lambda_t + \phi \spr{X_t, A_t} \phi \spr{X_t, A_t}\transpose$.
      \STATE $\cD_{t+1} = \cD_t \cup \scbr{\spr{X_t, A_t, X_{t+1}}}$.
      \STATE $t = t + 1$.
      \ENDWHILE
      \STATE \algcomment{initialize new epoch}
      \IF{$t = T_1$ \OR $\det \Lambda_{T_k} \geq 2 \det \Lambda_{t_e}$} \label{line-alg:epoch-full}
      \STATE $e = e + 1$.
      \STATE Set $k_e = k$ and $t_e = t$.
      \STATE Reset the policy $\pi_k = \piunif$.
      \ENDIF
      \STATE For any $\spr{x, a} \in \cX \times \cA$, $\CB_k \spr{x, a} = \beta \norm{\phi \spr{x, a}}_{\Lambda_{t_e}^{-1}}$,\;\;and\; $\CB_k\spr{x^\upplus, a} = 0$. \label{line-alg:bonuses-full}
      \STATE For any $\spr{x, a} \in \cX \times \cA$, $p_k^\upplus \spr{x, a} = \sigma \spr{\alpha \CB_k \spr{x, a} - \omega}$,\;\;and\; $p_k^\upplus \spr{x^\upplus, a} = 0$. \label{line-alg:pplus-full}
      \STATE \algcomment{optimistic regularized value iteration}
      \STATE $r_k^\upplus = \spr{1 - p_k^\upplus} \odot r_k + p_k^\upplus \cdot \RMAX$.
      \STATE $\wh{M V_k} = \Lambda_{T_k}^{-1} \sum_{\spr{x, a, x'} \in \cD_{T_k}} \phi \spr{x, a} V_k \spr{x'}$. \label{line-alg:ridge-full}
      \STATE $\wh{P_k^\upplus V_k} = \spr{1 - p_k^\upplus} \odot \phim \wh{M V_k} + p_k^\upplus \cdot V_k \spr{x^\upplus}$,\;\;and\; $\wh{P_k^\upplus V_k} \spr{x^\upplus, \cdot} = \frac{\RMAX}{1 - \gamma}$. \label{line-alg:transf-full}
      \STATE $Q_{k + 1} = r_k^\upplus + \spr{1 - p_k^\upplus} \odot \CB_k + \gamma \wh{P_k^\upplus V_k}$.
      \STATE $V_{k+1} \spr{x} = \frac1\eta \log \spr{\sum_a \pi_{k} \spr{a \given x} e^{\eta Q_{k+1} \spr{x, a}}}$.
      \STATE $\pi_{k+1} = \pi_{k} \odot e^{\eta \spr{Q_{k+1} - E V_{k+1}}}$.
      \ENDFOR
      \STATE {\bfseries Output:} $\pi_I$, with $I \sim \cU \spr{\sbr{K}}$.
  \end{algorithmic}
\end{algorithm}

\noindent Next, we include the pseudocode for our imitation learning algorithms built on \algname. At line~\ref{algline:expert_estimation}, the learner computes an estimate of the expert features expectation computing an elementwise empirical average of the features in the dataset $\cD_\expert$. Such an estimate is leveraged in the online gradient descent (OGD) update given by the function at lines~~\ref{algline:OGDstart}-\ref{algline:OGDend}. This function instantiates the general $\textsc{RewardUpdate}$ routine given in \Cref{alg:linear-rmax-ravi-ucb}. That is, after each policy update in \algname, the reward player estimates the feature expectation of the current policy $\pi_k$ as the plug in estimator $\phi_\cost \spr{X_k, A_k}$ with $X_k, A_k$ sampled from the occupancy measure $\mu \spr{\pi_k}$. Notice that for the reinforcement learning applications, the adversarial reward sequence is generated online observing the policies. Therefore, for this application it is important that the guarantees in \Cref{thm:main} holds against adaptive adversaries.
\begin{algorithm}[t]
  \caption{\FRAalg (Feature Rmax Adversarial Imitation Learning) \label{alg:fra}}
  \centering
  \begin{algorithmic}[1]
    \STATE {\bfseries Inputs:} \\
    (1) a features dataset $\cD_\expert = \scbr{\phi_\cost \spr{X^i_E, A^i_E}}^{\tau_E}_{i=1}$ where for any $i \in \sbr{\tau_E}$, $X^i_E, A^i_E \sim \mu \spr{\expert}$, \\
    (2) read access to $\phi_P \spr{x, a}$ for all $x, a \in \cX \times \cA$, \\
    (3) trajectory access to $\cM \setminus \true$, and \\
    (4) the reward weights class $\cW$ such that $w_{\mathrm{true}} \in \cW$ and $\norm{w} \leq \WMAX$ for all $w \in \cW$.
    \STATE Set $K, \eta, \beta, \omega, \alpha$ as in \Cref{thm:main}.
    \STATE Set $\eta_r = \nicefrac{\WMAX}{B \sqrt{K}}$.
    \STATE Estimate $\widehat{\lambda \spr{\expert}} = \frac{1}{\abs{\cD_\expert}} \sum^{\tau_E}_{i=1} \phi_\cost \spr{X_E^i, A_E^i}$. \label{algline:expert_estimation}
    \STATE \textbf{Function} \label{algline:OGDstart}{$\textsc{Ogd}$}{$\spr{\mu \spr{\pi_k}, w_{k-1}}$}
    \STATE Sample $X_k, A_k \sim \mu \spr{\pi_k}$.
    \STATE $\quad \quad \text{\textbf{return}} \quad w_k = \Pi_{\cW} \spr{w_{k-1} + \eta_r \spr{\widehat{\lambda \spr{\expert}} - \phi_r \spr{X_k, A_k}}}$. \label{algline:OGDend}
    \STATE \textbf{Output:}  \algname $\spr{K, \eta, \beta, \omega, \alpha, \textsc{RewardUpdate} = \textsc{Ogd}}$.
  \end{algorithmic}
\end{algorithm}

%% file: sections/analysis_proofs.tex
\section{Omitted proofs from Section~\ref{sec:analysis}}
\label{app:analysis-proofs}

\subsection{Proof of Lemma~\ref{lem:reward-bias-bound} (reward bias)}
\label{app:reward-bias-bound}

\rewardbiasbound*

\begin{proof}
    First note that for any action $a$, the rewards are equal, \ie $r_k \spr{x^\upplus, a} = r_k^\upplus \spr{x^\upplus, a}$. For the other states, plugging the definition of $r_k^\upplus$ gives
    \begin{align*}
        \rewardbias_k &= \inp{\mu \spr{\pi^\star} - \mu \spr{\pi_k}, r_k - r_k^\upplus} \\
        &= \inp{\mu \spr{\pi^\star} - \mu \spr{\pi_k}, p_k^\upplus \odot \spr{r_k - \RMAX \bfone}} \\
        &\leq - \inp{\mu \spr{\pi_k}, p_k^\upplus \odot \spr{r_k - \RMAX \bfone}} \\
        &\leq \RMAX \inp{\mu \spr{\pi_k}, p_k^\upplus}\,,
    \end{align*}
    where the first inequality follows from $r_k - \RMAX \bfone \preceq 0$ and $\mu \spr{\pi^\star} \succeq 0$, and the second inequality is due to $r_k \succeq 0$.
\end{proof}

    \subsection{Proof of Lemma~\ref{lem:model-bias-bounds} (model bias)}
    \label{app:model-bias-bounds}

\modelbiasbounds*

\begin{proof}
    Let us consider a process  $\spr{X_\tau, A_\tau}_{\tau \in \bbN}$ generated by the policy $\pi$ in the real MDP, \ie, such that $X_0 \sim \initial$, and for any $\tau \in \bbN$, $A_\tau \sim \pi \spr{\cdot \given X_\tau}$, and $X_{\tau+1} \sim P \spr{\cdot \given X_\tau, A_\tau}$. We denote $\spr{X^\upplus_{k, \tau}, A^\upplus_{k, \tau}}_{\tau \in \bbN}$ its coupled process in the optimistic MDP at episode $k$ generated as follows. At the first stage we set $X^\upplus_{k, 0} = X_0$, then for any $\tau \geq 1$, the coupled process evolves as follows
    \begin{equation*}
        X^\upplus_{k, \tau+1}, A^\upplus_{k, \tau+1} =
        \begin{cases}
            X_{\tau+1}, A_{\tau+1} \quad &\text{w.p.} \quad 1 - p^\upplus_k \spr{X_\tau, A_\tau} \quad \text{if} \quad  X^\upplus_{k, \tau}, A^\upplus_{k, \tau} = X_\tau, A_\tau \\
            x^\upplus, a \quad & \text{w.p.} \quad p^\upplus_k \spr{X_\tau, A_\tau} \quad \text{if} \quad  X^\upplus_{k, \tau}, A^\upplus_{k, \tau} = X_\tau, A_\tau \\
            x^\upplus, a \quad & \text{if} \quad X^\upplus_{k, \tau}, A^\upplus_{k, \tau} \neq X_\tau, A_\tau
        \end{cases}\,,
    \end{equation*}
    where $a$ can be any action. Then, we can rewrite the bias term as
    \begin{align*}
        &\modelbias_k \spr{\pi} = \spr{1 - \gamma} \bbE \sbr{\sum_{\tau=0}^\infty \gamma^\tau \spr{r_k^\upplus \spr{X^\upplus_{k, \tau}, A^\upplus_{k, \tau}} - r_k^\upplus \spr{X_\tau, A_\tau}}} \\
        &\quad\quad\quad= \spr{1 - \gamma} \bbE \sbr{\sum_{\tau=0}^\infty \gamma^\tau \II{\spr{X^\upplus_{k, \tau}, A^\upplus_{k, \tau}} \neq \spr{X_\tau, A_\tau}} \spr{r_k^\upplus \spr{X^\upplus_{k, \tau}, A^\upplus_{k, \tau}} - r_k^\upplus \spr{X_\tau, A_\tau}}}\,.
    \end{align*}
    By definition, the state-action pairs $\spr{X^\upplus_{k, \tau}, A^\upplus_{k, \tau}}$ and $\spr{X_\tau, A_\tau}$ differ when the coupled process goes to heaven, \ie $X^\upplus_{k, \tau} = x^\upplus$. Noting that $r_k \spr{x^\upplus, a} = \RMAX$ for any action $a \in \cA$, we further get
    \begin{align*}
        \modelbias_k \hspace{-1pt}\spr{\pi} &= \spr{1 - \gamma} \bbE \sbr{\sum_{\tau=0}^\infty \gamma^\tau \II{\spr{X^\upplus_{k, \tau}, A^\upplus_{k, \tau}} \neq \spr{X_\tau, A_\tau}} \hspace{-1pt}\spr{r_k^\upplus \spr{x^\upplus, A^\upplus_{k, \tau}} - r_k^\upplus \spr{X_\tau, A_\tau}}} \\
        &= \spr{1 - \gamma} \bbE \sbr{\sum_{\tau=0}^\infty \gamma^\tau \II{\spr{X^\upplus_{k, \tau}, A^\upplus_{k, \tau}} \neq \spr{X_\tau, A_\tau}} \spr{\RMAX - r_k^\upplus \spr{X_\tau, A_\tau}}}\,,
    \end{align*}
    and $\modelbias_k \spr{\pi} \geq 0$ follows from $r_k^\upplus \preceq \RMAX$. For the upper bound, we can instead use $r_k \succeq 0$ and continue as follows
    \begin{align*}
        \modelbias_k \spr{\pi} &\leq \spr{1 - \gamma} \RMAX \bbE \sbr{\sum_{\tau=0}^\infty \gamma^\tau \II{\spr{X^\upplus_{k, \tau}, A^\upplus_{k, \tau}} \neq \spr{X_\tau, A_\tau}} } \\
        &= \spr{1 - \gamma} \gamma \RMAX \sum_{\tau=0}^\infty \gamma^\tau \bbP \sbr{\spr{X^\upplus_{k, \tau+1}, A^\upplus_{k, \tau+1}} \neq \spr{X_{\tau+1}, A_{\tau+1}}}\,,
    \end{align*}
    \begin{figure}
        \begin{center}
            \begin{tikzpicture}[every node/.style={circle, draw, minimum size=1cm, font=\small},
                >={Stealth}, 
                shorten >=1pt, shorten <=1pt, 
                node distance=1.3cm and 1.3cm 
            ]
            
            \node (X0) at (0,0) {$X_{k, 0}$};
            \node (X1) [right=of X0] {$X_{k, 1}$};
            \node (X2) [right=1.3cm of X1] {$X_{k, 2}$};
            \node (X3) [right=of X2] {$X_{k, 3}$};
            \node (dots) [right=of X3, draw=none] {$\cdots$};
            
            \node (uX0) [below=of X1] {$x^\upplus$};
            \node (uX1) [below=of X2] {$x^\upplus$};
            \node (uX2) [below=of X3] {$x^\upplus$};
            \node (udots) [right=of uX2, draw=none] {$\cdots$};
            
            \draw[->] (X0) -- (X1);
            \draw[->] (X1) -- (X2);
            \draw[->] (X2) -- (X3);
            \draw[->] (X3) -- (dots);
            
            \draw[->, dashed] (X0) to[bend left=20] (X1);
            \draw[->, dashed] (X0) -- (uX0);
            \draw[->, dashed] (X1) to[bend left=20] (X2);
            \draw[->, dashed] (X1) -- (uX1);
            \draw[->, dashed] (X2) to[bend left=20] (X3);
            \draw[->, dashed] (X2) -- (uX2);
            \draw[->, dashed] (X3) to[bend left=20] (dots);
            
            \draw[->, dashed] (uX0) -- (uX1);
            \draw[->, dashed] (uX1) -- (uX2);
            \draw[->, dashed] (uX2) -- (udots);
            
            \end{tikzpicture}
        \end{center}
        \caption{The thick arrows represent the transitions of the process in the original MDP, while the dashed ones correspond to the utopian one.}
        \label{fig:bias-term}
    \end{figure}
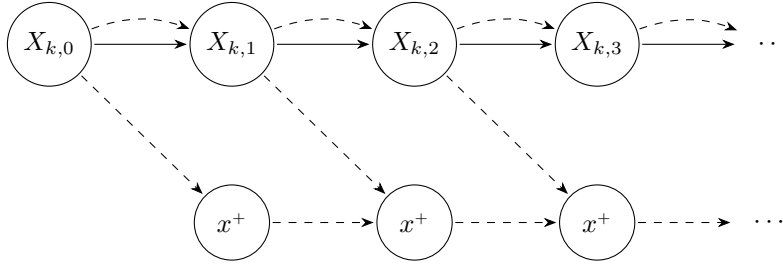
    \hspace{-4.5pt}where we used $\bbP \sbr{\spr{X^\upplus_{k, 0}, A^\upplus_{k, 0}} \neq \spr{X_0, A_0}} = 0$ by definition. Then, as illustrated in Figure~\ref{fig:bias-term}, two cases can happen. Either the coupled process was still in the original MDP and transitioned to heaven, either it was already in heaven. Denoting for any $\tau \geq 0$ the event $\cE_\mathrm{split} \spr{\tau} = \scbr{\spr{X^\upplus_{k, \tau}, A^\upplus_{k, \tau}} \neq \spr{X_\tau, A_\tau}}$ and $\mathcal{E}^c_\mathrm{split} \spr{\tau}$ its complementary, we have
    \begin{align*}
        \bbP \sbr{\cE_\mathrm{split} \spr{\tau+1}} &= \bbP \sbr{\cE_\mathrm{split} \spr{\tau+1} \given \cE_\mathrm{split} \spr{\tau}} \bbP \sbr{\cE_\mathrm{split} \spr{\tau}} \\
        &\phantom{=}+ \bbP \sbr{\cE_\mathrm{split} \spr{\tau+1} \given \cE^c_\mathrm{split} \spr{\tau}} \bbP \sbr{\cE^c_\mathrm{split} \spr{\tau}}\,.
    \end{align*}
    If the coupled process is already in the heaven state $x^\upplus$, then it stays there. Otherwise, it can transition there with probability $\bbE \sbr{p_k^\upplus \spr{X_\tau, A_\tau}}$, thus
    \begin{align*}
        \bbP \sbr{\cE_\mathrm{split} \spr{\tau+1}} &= \bbP \sbr{\cE_\mathrm{split} \spr{\tau}} + \bbE \sbr{p_k^\upplus \spr{X_\tau, A_\tau}} \bbP \sbr{\cE^c_\mathrm{split} \spr{\tau}} \\
        &\leq \bbP \sbr{\cE_\mathrm{split} \spr{\tau}} + \bbE \sbr{p_k^\upplus \spr{X_\tau, A_\tau}} \\
        &\leq \sum_{u=0}^\tau \bbE \sbr{p_k^\upplus \spr{X_u, A_u}}\,,
    \end{align*}
    by induction. Therefore, we get
    \begin{align*}
        \modelbias_k \spr{\pi} &\leq \spr{1 - \gamma} \gamma \RMAX \sum_{\tau=0}^\infty \gamma^\tau \sum_{u=0}^\tau \bbE \sbr{p_k^\upplus \spr{X_u, A_u}} \\
        &= \spr{1 - \gamma} \gamma \RMAX \bbE \sbr{\sum_{u=0}^\infty \sum_{\tau=0}^\infty \gamma^\tau \II{u \leq \tau} p_k^\upplus \spr{X_u, A_u}} \\
        &= \spr{1 - \gamma} \gamma \RMAX \bbE \sbr{\sum_{u=0}^\infty \sum_{\tau=u}^\infty \gamma^\tau p_k^\upplus \spr{X_u, A_u}} \\
        &= \gamma \RMAX \bbE \sbr{\sum_{u=0}^\infty \gamma^u p_k^\upplus \spr{X_u, A_u}}\,,
    \end{align*}
    and the conclusion follows from the definition of $\mu \spr{\pi}$ and $\gamma < 1$.
\end{proof}

        \subsubsection{Alternative Proof of Lemma~\ref{lem:model-bias-bounds}}

\noindent We also provide an alternative proof based on a simulation lemma.
\begin{proof}
    By the flow constraints associated to $\mu_k^\upplus \spr{\pi}$ and after rearranging, we get
    \begin{align*}
        \modelbias_k \spr{\pi} &= \spr{1 - \gamma} \inp{\nu_0, V_{P_k^\upplus, r_k^\upplus}^\pi - V_{P, r_k^\upplus}^\pi} \\
        &= \inp{\opE\transpose \mu_k^\upplus \spr{\pi} - \gamma \spr{\opPplusk}\transpose \mu_k^\upplus \spr{\pi}, V_{P_k^\upplus, r_k^\upplus}^\pi - V_{P, r_k^\upplus}^\pi} \\
        &= \inp{\mu_k^\upplus \spr{\pi}, \opE V_{P_k^\upplus, r_k^\upplus}^\pi - \opE V_{P, r_k^\upplus}^\pi - \gamma \opPplusk V_{P_k^\upplus, r_k^\upplus}^\pi + \gamma \opPplusk V_{P, r_k^\upplus}^\pi}\,.
    \end{align*}
    Applying Lemma~\ref{lem:from-ev-to-q} to both $V_{P_k^\upplus, r_k^\upplus}^\pi$ and $V_{P, r_k^\upplus}^\pi$ and Bellman's equation for $Q_{P_k^\upplus, r_k^\upplus}^\pi$, we further have    
    \begin{align*}
        \modelbias_k \spr{\pi} &= \inp{\mu_k^\upplus \spr{\pi}, Q_{P_k^\upplus, r_k^\upplus}^\pi - Q_{P, r_k^\upplus}^\pi - \gamma \opPplusk V_{P_k^\upplus, r_k^\upplus}^\pi + \gamma \opPplusk V_{P, r_k^\upplus}^\pi} \\
        &= \inp{\mu_k^\upplus \spr{\pi}, r_k^\upplus + \gamma \opPplusk V_{P, r_k^\upplus}^\pi - Q_{P, r_k^\upplus}^\pi} \,,
    \end{align*}
    Plugging the definition of $P_k^\upplus$ and this time using Bellman's equation for $Q_{P, r_k^\upplus}^\pi$, we obtain
    \begin{align}
        \modelbias_k \spr{\pi} &= \inp{\mu_k^\upplus \spr{\pi}, r_k^\upplus + \gamma \spr{1 - p_k^\upplus} \odot \opP V_{P, r_k^\upplus}^\pi + p_k^\upplus \odot \bfe_{x^\upplus} V_{P, r_k^\upplus}^\pi - Q_{P, r_k^\upplus}^\pi} \nonumber \\
        &= \inp{\mu_k^\upplus \spr{\pi}, p_k^\upplus \odot \spr{\bfe_{x^\upplus} V_{P, r_k^\upplus}^\pi - \gamma \opP V_{P, r_k^\upplus}^\pi}} \nonumber \\
        &= \inp{\mu_k^\upplus \spr{\pi}, p_k^\upplus \odot \spr{\frac{\RMAX}{1 - \gamma} \bfone - \gamma \opP V_{P, r_k^\upplus}^\pi}}\,, \label{eq:hope-decomp}
    \end{align}
    where the last equality is due to having $\spr{\bfe_{x^\upplus} V_{P, r_k^\upplus}^\pi} \spr{x, a} = V_{P, r_k^\upplus}^\pi \spr{x^\upplus} = \frac{\RMAX}{1 - \gamma}$ for any state-action pair $\spr{x, a}$. The lower bound follows from noticing that $\opP V_{P, r_k^\upplus}^\pi \preceq \frac{1}{1 - \gamma} \bfone \preceq \frac{\RMAX}{1 - \gamma} \bfone$,
    \begin{align*}
        \modelbias_k \spr{\pi} &\geq \spr{\frac{\RMAX}{1 - \gamma} - \frac{\gamma \RMAX}{1 - \gamma}} \cdot \inp{\mu_k^\upplus \spr{\pi}, p_k^\upplus} \\
        &= \RMAX \cdot \inp{\mu_k^\upplus \spr{\pi}, p_k^\upplus} \\
        &\geq 0\,.
    \end{align*}
    Moving to the upper bound, from Equation~\ref{eq:hope-decomp} and $\opP V_{P, r_k^\upplus}^\pi \succeq 0$, we get
    \begin{align*}
        \modelbias_k \spr{\pi} &= \inp{\mu_k^\upplus \spr{\pi}, p_k^\upplus \odot \spr{\frac{\RMAX}{1 - \gamma} \bfone - \gamma \opP V_{P, r_k^\upplus}^\pi}} \\
        &\leq \frac{\RMAX}{1 - \gamma} \inp{\mu_k^\upplus \spr{\pi}, p_k^\upplus} \\
        &\leq \frac{\RMAX}{1 - \gamma} \inp{\mu \spr{\pi}, p_k^\upplus}\,,
    \end{align*}
    where the last inequality follows from Lemma~\ref{lem:mass-reduced}.
\end{proof}

\subsection{Proof of Lemma~\ref{lem:bound-regret-plus} (optimistic regret)}
\label{app:bound-regret-plus}

In order to prove Lemma~\ref{lem:bound-regret-plus}, we first need the following result that shows that the functions $Q_k$ are optimistic estimates of an ideal sequence of dynamic-programming updates computed in the augmented MDPs. The statement is is an adaptation of Lemma~4.2 of \citet{MN23}, and its complete proof is provided below.
\begin{restatable}{Lem}{boundq} \label{lem:bound-q}
    Suppose that the bonuses $\CB_k$ are valid for the MDP $\cM_k$ in the sense of Equation~\ref{eq:validity-bonuses}. Then, for any $k$ and any state-action pair $\spr{x, a} \in \cX^\upplus \times \cA$, the iterates $Q_k$ satisfy
    \begin{equation*} \label{eq:bound-q}
        r_k^\upplus + \gamma \opPplusk V_k \leq Q_{k+1} \leq r_k^\upplus + 2 \spr{1 - p_k^\upplus} \odot \CB_k + \gamma \opPplusk V_k\,.
    \end{equation*}
\end{restatable}

\begin{proof}
    For $x^\upplus$ and any action $a$, it is straightforward to check that both inequalities are equalities. Let $\spr{x, a} \in \cX \times \cA$. We have
    \begin{align*}
        r_k^\upplus \spr{x, a} + \gamma \opPplusk V_k \spr{x, a} &= r_k^\upplus \spr{x, a} + \gamma \spr{\opPplusk - \ophPplusk} V_k \spr{x, a} + \gamma \ophPplusk V_k \spr{x, a} \\
        &\leq Q_{k+1} \spr{x, a} \\
        &\leq r_k^\upplus \spr{x, a} + 2 \spr{1 - p_k^\upplus \spr{x, a}} \CB_k \spr{x, a} + \gamma \opPplusk V_k \spr{x, a}\,,
    \end{align*}
    where both inequalities use the fact that
    \begin{equation*}
        \abs{\spr{\opPplusk - \ophPplusk} V_k \spr{x, a}} \leq \spr{1 - p_k^\upplus \spr{x, a}} \CB_k \spr{x, a}\,,
    \end{equation*}
    which is implied by the event $\mathcal{E}_{\mathrm{valid}}$ in Equation~\ref{eq:validity-bonuses}.
\end{proof}

\boundregretplus*

\begin{proof}
    We decompose $\regretKplus$ as follows
    \begin{align*}
        \regretKplus &= \sumkK \Bigl(\underbrace{\spr{\inp{\mu_k^\upplus \spr{\pi^\star}, r_k^\upplus} - \spr{1 - \gamma} \inp{\initial, V_k}}}_{= \Delta_k^\star} + \underbrace{\spr{1 - \gamma} \inp{\initial, V_k} - \inp{\mu_k^\upplus \spr{\pi_k}, r_k^\upplus}}_{= \Delta_k}\Bigr)\,,
    \end{align*}
    where we defined $\Delta_k^\star$ and $\Delta_k$. We start with the first term. Using the flow constraint with $\mu_k^\upplus \spr{\pi^\star}$,
    \begin{align*}
        \Delta_k^\star &= \inp{\mu_k^\upplus \spr{\pi^\star}, r_k^\upplus} - \inp{\opE\transpose \mu_k^\upplus \spr{\pi^\star} - \gamma \spr{\opPplusk}\transpose \mu_k^\upplus \spr{\pi^\star}, V_k} \\
        &= \inp{\mu_k^\upplus \spr{\pi^\star}, r_k^\upplus + \gamma \opPplusk V_k - \opE V_{k+1}} + \inp{\mu_k^\upplus \spr{\pi^\star}, \opE V_{k+1} - \opE V_k}\,.
    \end{align*}
    Using the lower bound on $Q_{k+1}$ from Lemma~\ref{lem:bound-q}, we have
    \begin{equation*}
        \Delta_k^\star \leq \inp{\mu_k^\upplus \spr{\pi^\star}, Q_{k+1} - \opE V_{k+1}} + \inp{\mu_k^\upplus \spr{\pi^\star}, \opE V_{k+1} - \opE V_k}\,,
    \end{equation*}
    where the term in $x = x^\upplus$ is equal to zero. Summing over $k \in \sbr{K} = \bigcup_{e \in \sbr{1, E \spr{K}}} \sbr{k_e, k_{e+1} - 1}$,
    \begin{equation*}
        \sumkK \Delta_k^\star \leq \sum_{e=1}^{E \spr{K}} \inp{\mu_{k_e}^\upplus \spr{\pi^\star}, \sum_{k \in \cK_e} \spr{Q_{k+1} - \opE V_{k+1}}} + \sum_{e=1}^{E \spr{K}} \inp{\opE\transpose \mu_{k_e}^\upplus \spr{\pi^\star}, V_{k_{e+1}} - V_{k_e}}_{\cX}\,,
    \end{equation*}
    where the sum within each epoch of the second term telescoped. By \citealp[Lemma~C.1]{MN23}, we have for any state $x \in \cX$
    \begin{align*}
        \sum_{k \in \cK_e} V_{k+1} \spr{x} &= \max_{p \in \Delta \spr{\cA}} \inp{p, \sum_{k \in \cK_e} Q_{k+1} \spr{x, \cdot}} - \frac1\eta \KL \spr{p \| \piunif} \\
        &\geq \inp{\pi^\star \spr{\cdot \given x}, \sum_{k \in \cK_e} Q_{k+1} \spr{x, \cdot}} - \frac1\eta \KL \spr{\pi^\star \spr{\cdot \given x} \| \piunif}\,,
    \end{align*}
    where we used $\pi_{k_e} = \piunif$ in the first equality and denoted $\cK_e$ the set of episodes in epoch $e$. Multiplying the previous inequality by $\nu_{k_e}^\upplus \spr{\pi^\star,x}$, summing over $x \in \cX$, and noting that $\mu_{k_e}^\upplus \spr{\pi^\star} = \nu_{k_e}^\upplus \spr{\pi^\star} \odot \pi^\star$, we obtain
    \begin{align*}
        \sum_{e=1}^{E \spr{K}} \inp{\mu_{k_e}^\upplus \spr{\pi^\star}, \sum_{k \in \cK_e} \spr{Q_{k+1} - \opE V_{k+1}}} &\leq \frac1\eta \sum_{e=1}^{E \spr{K}} \inp{\nu_{k_e}^\upplus \spr{\pi^\star}, \KL \spr{\pi^\star \| \piunif}} \\
        &\leq \frac{E \spr{K} \log \abs{\cA}}{\eta}\,.
    \end{align*}
    The second term can be bounded with Hölder's inequality,
    \begin{equation*}
        \sum_{e=1}^{E \spr{K}} \inp{\opE\transpose \mu_{k_e}^\upplus \spr{\pi^\star}, V_{k_{e+1}} - V_{k_e}} \leq \sum_{e=1}^{E \spr{K}} \norm{\nu_{k_e}^\upplus \spr{\pi^\star}}_1 \norm{V_{k_{e+1}} - V_{k_e}}_\infty \leq 2 E \spr{K} \VMAX\,,
    \end{equation*}
    where we used $\nu_{k_e}^\upplus \spr{\pi^\star} \in \Delta \spr{\cX^\upplus}$ and $\norm{V_k}_\infty \leq \VMAX$ which follows from $\norm{Q_k}_\infty \leq \QMAX$. Therefore, we get
    \begin{equation*}
        \sumkK \Delta_k^\star \leq \frac{E \spr{K} \log \abs{\cA}}{\eta} + 2 E \spr{K} \VMAX\,.
    \end{equation*}

    \noindent Moving to $\Delta_k$, we apply the flow constraints to $\mu_k^\upplus \spr{\pi_k}$ to get
    \begin{align*}
        \Delta_k &= \inp{\opE\transpose \mu_k^\upplus \spr{\pi_k} - \gamma \spr{\opPplusk}\transpose \mu_k^\upplus \spr{\pi_k}, V_k} - \inp{\mu_k^\upplus \spr{\pi_k}, r_k^\upplus} \\
        &= \inp{\mu_k^\upplus \spr{\pi_k}, \opE V_k} - \inp{\mu_k^\upplus \spr{\pi_k}, r_k^\upplus + \gamma \opPplusk V_k} \\
        &\leq \inp{\mu_k^\upplus \spr{\pi_k}, \opE V_k - Q_{k+1}} + 2 \inp{\mu_k^\upplus \spr{\pi_k}, \spr{1 - p_k^\upplus} \odot \CB_k} \\
        &\leq \inp{\mu_k^\upplus \spr{\pi_k}, \opE V_k - Q_{k+1}} + 2 \inp{\mu \spr{\pi_k}, \spr{1 - p_k^\upplus} \odot \CB_k}\,,
    \end{align*}
    where the first inequality follows from the upper bound on $Q_{k+1}$ in Lemma~\ref{lem:bound-q} and the term in $x = x^\upplus$ being equal to zero, and the second inequality is due to Lemma~\ref{lem:mass-reduced}. Next, noticing $\inp{\mu_k^\upplus \spr{\pi_{k+1}}, Q_{k+1}} = \inp{\nu_k^\upplus \spr{\pi_{k+1}}, V_{k+1} + \frac1\eta \KL \spr{\pi_{k+1} \| \pi_k}}$,
    \begin{align*}
        \inp{\mu_k^\upplus \spr{\pi_k}, \opE V_k - Q_{k+1}} &= \inp{\nu_k^\upplus \spr{\pi_k}, V_k} - \inp{\mu_k^\upplus \spr{\pi_{k+1}}, Q_{k+1}} \\
        &\phantom{=}+ \inp{\mu_k^\upplus \spr{\pi_{k+1}}, Q_{k+1}} - \inp{\mu_k^\upplus \spr{\pi_k}, Q_{k+1}} \\
        &= \inp{\nu_k^\upplus \spr{\pi_k}, V_k} - \inp{\nu_k^\upplus \spr{\pi_{k+1}}, V_{k+1}} \\
        &\phantom{=}- \frac1\eta \inp{\nu_k^\upplus \spr{\pi_{k+1}}, \KL \spr{\pi_{k+1} \| \pi_k}}  \\
        &\phantom{=}+ \inp{\mu_k^\upplus \spr{\pi_{k+1}} - \mu_k^\upplus \spr{\pi_k}, Q_{k+1}}\,.
    \end{align*}
    We sum over $k$ and look at the different terms separately. First, we get
    \begin{align*}
        \sumkK \spr{\inp{\nu_k^\upplus \spr{\pi_k}, V_k} - \inp{\nu_k^\upplus \spr{\pi_{k+1}}, V_{k+1}}} &= \sum_{e=1}^{E \spr{K}} \spr{\inp{\nu_{k_e}^\upplus \spr{\pi_{k_e}}, V_{k_e}} - \inp{\nu_{k_e}^\upplus \spr{\pi_{k_{e+1}}}, V_{k_{e+1}}}} \\
        &\leq 2 \VMAX E \spr{K}\,.
    \end{align*}
    For the third term, we have
    \begin{align*}
        \sumkK \inp{\mu_k^\upplus \spr{\pi_{k+1}} - \mu_k^\upplus \spr{\pi_k}, Q_{k+1}} &= \sum_{e=1}^{E \spr{K}} \sum_{k \in \cK_e} \inp{\mu_{k_e}^\upplus \spr{\pi_{k+1}} - \mu_{k_e}^\upplus \spr{\pi_k}, Q_{k+1}}\,.
    \end{align*}
    Successively applying Hölder's inequality, Pinsker's inequality and Lemma~\ref{lem:ineq-kl-entropy},
    \begin{align*}
        \inp{\mu_{k_e}^\upplus \spr{\pi_{k+1}} - \mu_{k_e}^\upplus \spr{\pi_k}, Q_{k+1}} &\leq \QMAX \norm{\mu_{k_e}^\upplus \spr{\pi_{k+1}} - \mu_{k_e}^\upplus \spr{\pi_k}}_1 \\
        &\leq \QMAX \sqrt{2 \KL \spr{\mu_{k_e}^\upplus \spr{\pi_{k+1}} \| \mu_{k_e}^\upplus \spr{\pi_k}}} \\
        &\leq \QMAX \sqrt{\frac{2}{1 - \gamma} \inp{\nu_{k_e}^\upplus \spr{\pi_{k+1}}, \KL \spr{\pi_{k+1} \| \pi_k}}}\,.
    \end{align*}
    For any $x \in \cX$, the KL divergence between $\pi_{k+1}$ and $\pi_k$ in state $x$ can be bounded as
    \begin{align*}
        &\KL \spr{\pi_{k+1} \middle\| \pi_k} \spr{x} \\
        &\quad\quad\quad= \sum_{a \in \cA} \pi_{k+1} \spr{a \given x} \spr{\eta Q_{k+1} \spr{x, a} - \log \spr{\sum_{b \in \cA} \pi_k \spr{b \given x} \exp \sbr{\eta Q_{k+1} \spr{x, b}}}} \\
        &\quad\quad\quad= \eta \sum_{a \in \cA} \pi_{k+1} \spr{a \given x} Q_{k+1} \spr{x, a} - \log \spr{\sum_{b \in \cA} \pi_k \spr{b \given x} \exp \sbr{\eta Q_{k+1} \spr{x, b}}} \\
        &\quad\quad\quad\leq \eta \sum_{a \in \cA} \sbr{\pi_{k+1} \spr{a \given x} - \pi_k \spr{a \given x}} Q_{k+1} \spr{x, a} \\
        &\quad\quad\quad\leq \eta \QMAX \norm{\pi_{k+1} \spr{\cdot \given x} - \pi_k \spr{\cdot \given x}}_1 \\
        &\quad\quad\quad\leq \eta \QMAX \sqrt{2 \KL \spr{\pi_{k+1} \middle\| \pi_k} \spr{x}}\,,
    \end{align*}
    where the first inequality follows from Jensen's and the convexity of $- \log$, the second inequality is by Hölder's and the boundedness of $Q_k$, and the last inequality is due to Pinkser's. Dividing by $\sqrt{\KL \spr{\pi_{k+1} \middle\| \pi_k} \spr{x}}$ and squaring the inequality, we get
    \begin{equation*}
        \KL \spr{\pi_{k+1} \middle\| \pi_k} \spr{x} \leq 2 \eta^2 \QMAX^2\,.
    \end{equation*}
    Plugging this back into the previous inequality and summing over $k \in \sbr{K}$, we get
    \begin{equation*}
        \sumkK \inp{\mu_k^\upplus \spr{\pi_{k+1}} - \mu_k^\upplus \spr{\pi_k}, Q_{k+1}} \leq \frac{2 \QMAX^2 \eta K}{\sqrt{1 - \gamma}}\,.
    \end{equation*}
    The remaining term is nonpositive. The sum of the $\Delta_k$ terms is thus bounded by
    \begin{equation*}
        \sumkK \Delta_k \leq 2 \VMAX E \spr{K} + \frac{2 \QMAX^2 \eta K}{\sqrt{1 - \gamma}} + 2 \sumkK \inp{\mu \spr{\pi_k}, \spr{1 - p_k^\upplus} \odot \CB_k}\,.
    \end{equation*}
    Finally, combining the bounds on $\sumkK \Delta_k^\star$ and $\sumkK \Delta_k$, we get
    \begin{equation*}
        \regretKplus \leq \frac{E \spr{K} \log \abs{\cA}}{\eta} + 4 \VMAX E \spr{K} + \frac{2 \QMAX^2 \eta K}{\sqrt{1 - \gamma}} + 2 \sumkK \inp{\mu \spr{\pi_k}, \spr{1 - p_k^\upplus} \odot \CB_k}\,.
    \end{equation*}
\end{proof}

\subsection{Proof of Lemma~\ref{lem:qmax} (choice of $\QMAX$)}
\label{app:qmax}

\qmaxlemma*

\begin{proof}
    We want to find $\QMAX$ such that for any $k$, $\norm{Q_k}_\infty \leq \QMAX$. Since $V_k$ is a log-sum-exp of $Q_k$, we have $\norm{V_k}_\infty \leq \norm{Q_k}_\infty$. Next, we proceed by induction to find a suitable value of $\QMAX$. Let $k \in \bbN^\star$ and assume $\norm{Q_k}_\infty \leq \QMAX$. For any $\spr{x, a}$,
    \begin{align*}
        \abs{Q_{k+1} \spr{x, a}} &\leq r_k^\upplus \spr{x, a} + \spr{1 - p_k^\upplus \spr{x, a}} \CB_k \spr{x, a} + \gamma \abs{\ophPplusk V_k \spr{x, a}} \\
        &\leq r_k^\upplus \spr{x, a} + 2 \spr{1 - p_k^\upplus \spr{x, a}} \CB_k \spr{x, a} + \gamma \opPplusk V_k \spr{x, a} \\
        &\leq \RMAX + 2 \spr{1 - p_k^\upplus \spr{x, a}} \CB_k \spr{x, a} + \gamma \QMAX\,,
    \end{align*}
    where the second inequality follows from the validity of the bonuses (and corresponds to the upper bound on $Q_{k+1}$ from Lemma~\ref{lem:bound-q}), and the third inequality is due to the inductive assumption and the boundedness of the rewards. Plugging the definition of the probabilities $p_k^\upplus$, we further get
    \begin{align*}
        \abs{Q_{k+1} \spr{x, a}} &\leq \RMAX + 2 \sigma \spr{\omega - \alpha \CB_k \spr{x, a}} \CB_k \spr{x, a} + \gamma \QMAX \\
        &\leq \RMAX + \gamma \QMAX + 2 \sup_{z \geq 0} \scbr{\sigma \spr{\omega - \alpha z} z} \\
        &\leq \RMAX + \gamma \QMAX + \frac{2 \omega}{\alpha}\,,
    \end{align*}
    the last inequality is simply a property of the sigmoid function and is showm in Lemma~\ref{lem:sigmoid-bound2}. For the induction to work at time $k + 1$, we need to set $\QMAX$ such that $\QMAX = \RMAX + \gamma \QMAX + \frac{2 \omega}{\alpha}$, that is
    \begin{equation*}
        \QMAX = \frac{\RMAX + 2 \omega / \alpha}{1 - \gamma}\,.
    \end{equation*}
    The initial case is also true since $\norm{Q_1}_\infty \leq \frac{\RMAX}{1 - \gamma} \leq \QMAX$.
\end{proof}

\subsection{Proof of Lemma~\ref{lem:expected-bonuses-bound} (bound on bonuses)}
\label{app:expected-bonuses-bound}

We now control the sum of bonuses. For any episode $k$, we will denote $\cT_k$ the set of timesteps in episode $k$.

\expectedbonusesbound*

\noindent To prove Lemma~\ref{lem:expected-bonuses-bound}, we need the following result.
\begin{lemma} \label{lem:expectation-to-highprob}
    Suppose $\mathcal{E}_L$ holds. Let $\scbr{f_k}_{k \in \sbr{K}} \subset \bbR^{\cX \times \cA}$ be a sequence of functions with values in $\sbr{0, M}$ almost surely. Then, with probability at least $1 - \delta$ the schedule and policies produced by Algorithm~\ref{alg:linear-rmax-ravi-ucb} satisfy
    \begin{equation*}
        \sumkK \inp{\mu \spr{\pi_k}, f_k} \leq 2 \spr{1 - \gamma} \sumkK \sum_{t \in \cT_k} f_k \spr{X_t, A_t} + 4 M \log \spr{\frac{2 K}{\delta}}^2\,.
    \end{equation*}
\end{lemma}
\begin{proof}
    We denote $\cF_{k-1}$ the $\sigma$-field generated by the history up to the end of episode $k-1$. We have,
    \begin{align*}
        \inp{\mu \spr{\pi_k}, f_k} &= \spr{1 - \gamma} \bbE \sbr{\sum_{\tau = 0}^\infty \gamma^\tau f_k \spr{X_\tau, A_\tau} \given \cF_{k-1}} \\
        &= \spr{1 - \gamma} \bbE \sbr{\sum_{\tau = 0}^\infty \II{\tau < L_k} f_k \spr{X_\tau, A_\tau} \given \cF_{k-1}} \\
        &= \spr{1 - \gamma} \bbE \sbr{\sum_{\tau = 0}^{L_k - 1} f_k \spr{X_\tau, A_\tau} \given \cF_{k-1}} \\
        &= \spr{1 - \gamma} \bbE \sbr{\sum_{\tau \in \cT_k} f_k \spr{X_\tau, A_\tau} \given \cF_{k-1}}\,.
    \end{align*}
    Plugging it back in the previous display,
    \begin{align*}
        \sumkK \inp{\mu \spr{\pi_k}, f_k} &= \spr{1 - \gamma} \sumkK \bbE \sbr{\sum_{t \in \cT_k} f_k \spr{X_t, A_t} \given \cF_{k-1}}\,.
    \end{align*}
    Since we assume $\cE_L$ holds, for any $k$ we have that $\sum_{t \in \cT_k} f_k \spr{X_t, A_t}$ takes values in $\sbr{0, M \LMAX}$ almost surely. Using the concentration inequality from Lemma~\ref{lem:concentration-ineq-cond-exp}, we get
    \begin{align*}
        \sumkK \inp{\mu \spr{\pi_k}, f_k} &\leq 2 \spr{1 - \gamma} \sumkK \sum_{t \in \cT_k} f_k \spr{X_t, A_t} + 4 \spr{1 - \gamma} M \LMAX \log \spr{\frac{2 K}{\delta}} \\
        &\leq 2 \spr{1 - \gamma} \sumkK \sum_{t \in \cT_k} f_k \spr{X_t, A_t} + 4 M \log \spr{\frac{2 K}{\delta}}^2\,,
    \end{align*}
    where we used that $\LMAX = \frac{\log \spr{K / \delta}}{1 - \gamma}$.
\end{proof}

We now prove Lemma~\ref{lem:expected-bonuses-bound}. With a slight abuse of notation, we use the convention that for any epoch $e$ and any $t$ in epoch $e$, the bonuses at time step $t$ are $\CB_t = \CB_{t_e}$. Noting that the bonuses $\CB_t$ take values in $\sbr{0, \beta B}$, we apply Lemma~\ref{lem:expectation-to-highprob} to get
\begin{equation*}
    \sumkK \inp{\mu \spr{\pi_k}, \CB_k} \leq 2 \spr{1 - \gamma} \sumtT \CB_t \spr{X_t, A_t} + 4 \beta B \log \spr{\frac{2 K}{\delta}}^2\,,
\end{equation*}
where we used $\CB_t \succeq 0$, and $T_{K+1} \leq T$ which follows from the event $\cE_L$. Likewise, applying Lemma~\ref{lem:expectation-to-highprob} to $\CB_t^2 \in \sbr{0, \beta^2 B^2}$, we obtain a similar bound for the term $\sumkK \inp{\mu \spr{\pi_k}, \CB_k^2}$,
\begin{equation*}
    \sumkK \inp{\mu \spr{\pi_k}, \CB_k^2} \leq 2 \spr{1 - \gamma} \sumtT \CB_t \spr{X_t, A_t}^2 + 4 \beta^2 B^2 \log \spr{\frac{2 K}{\delta}}^2\,.
\end{equation*}
By Cauchy-Schwartz's inequality, we have $\sumtT \CB_t \spr{X_t, A_t} \leq \sqrt{T} \sqrt{\sumtT \CB_t \spr{X_t, A_t}^2}$, so we can focus on the latter sum. By definition of the bonuses for linear MDPs, we have
\begin{align*}
    \sumtT \CB_t \spr{X_t, A_t}^2 &= \beta^2 \sum_{e=1}^{E \spr{K}} \sum_{k \in \cK_e} \sum_{t \in \cT_k} \norm{\phi \spr{X_t, A_t}}_{\Lambda_{t_e}^{-1}}^2\,.
\end{align*}
Since the covariance matrix only contains the data until the beginning of the epoch, there is a delay with $\phi \spr{X_t, A_t}$ which is further ahead. To compensate for this, note that for any $t \in \sbr{t_e, t_{e + 1}-1}$, we have $\det \Lambda_t \leq 2 \det \Lambda_{t_e}$ due to the update condition in Algorithm~\ref{alg:linear-rmax-ravi-ucb}, so by Lemma~\ref{lem:det-elliptical-bound}
\begin{align*}
    \norm{\phi \spr{X_t, A_t}}_{\Lambda_{t_e}^{-1}}^2 \leq \frac{\det \spr{\Lambda_{t_e}^{-1}}}{\det \spr{\Lambda_t^{-1}}} \norm{\phi \spr{X_t, A_t}}_{\Lambda_t^{-1}}^2 \leq 2 \norm{\phi \spr{X_t, A_t}}_{\Lambda_t^{-1}}^2\,.
\end{align*}
We plug this back into the previous inequality and apply Lemma~\ref{lem:bound-elliptical-potential} to obtain\footnote{Note that this is where the linear dependency in $B$ appears, but this can be removed by setting $\lambda = 1 / B^2$.}
\begin{equation*}
    \sumtT \CB_t \spr{X_t, A_t}^2 \leq 2 \beta^2 \sumtT \norm{\phi \spr{X_t, A_t}}_{\Lambda_t^{-1}}^2 \leq 4 \beta^2 B^2 \log \spr{\frac{\det \Lambda_T}{\det \Lambda_0}}\,.
\end{equation*}
Using the definition of $\Lambda_0, \Lambda_T$, the trace-determinant inequality, and the assumption $\norm{\phi \spr{\cdot, \cdot}}_2 \leq B$, we finally get
\begin{align*}
    \sumtT \CB_t \spr{X_t, A_t}^2 &\leq 4 \beta^2 B^2 d \log \spr{\frac{d + \sumtT \norm{\phi \spr{X_t, A_t}}_2^2}{d}} \\
    &\leq 4 \beta^2 B^2 d \log \spr{1 + \frac{B^2 T}{d}}\,.
\end{align*}
The conclusion follows from plugging this back into the inequalities of interest.

\subsection{Proof of Lemma~\ref{lem:good-event-holds} (good event holds)}
\label{app:good-event-holds}

Before stating the proof of Lemma~\ref{lem:good-event-holds}, we need to define some auxiliary quantities and state two intermediate results. First recall that $\scbr{L_k}_{k=1}^K$ denote the number of steps between consecutive resets and that for any $k \geq 2$, $L_k = T_k - T_{k-1}$, and $L_1 = T_1$. We need to prove the episodes are not too long, \ie $\cE_L = \scbr{\forall k \in \sbr{K}, L_k \leq \LMAX}$ holds with high probability, where $\LMAX = H \log \spr{K / \delta}$. This is done in Lemma~\ref{lem:lmax}. Then, we define the event $\cE_{V, \mathrm{alg}}$ on the iterates generated by Algorithm~\ref{alg:linear-rmax-ravi-ucb}
\begin{equation*}
    \cE_{V, \mathrm{alg}} = \scbr{\forall k \in \sbr{K}, \norm{M V_k - \wh{M V_k}}_{\Lambda_{T_k}} \leq \beta}\,,
\end{equation*}
where $\wh{M V_k} = \Lambda_{T_k}^{-1} \sum_{\spr{x, a, x'} \in \cD_{T_k}} \phi \spr{x, a} V_k \spr{x'}$. To prove $\cE_{V, \mathrm{alg}}$ holds with high probability, we need to resort to a standard uniform covering argument first introduced by \citealp{jin2019provably}. To do so, let us denote with $p^\upplus_{\Lambda, \beta, \alpha} = \sigma (\alpha \beta \norm{\phi\spr{\cdot,\cdot}}_{\Lambda} - w) = 1 - \sigma \spr{- \alpha \beta \norm{\phi \spr{\cdot, \cdot}}_{\Lambda} + \omega}$ an ascension function parametrized by the matrix $\Lambda$, the scalar $\beta$ and the sigmoid slope $\alpha$. Then, we define the following class of functions on $\cX \times \cA$
\begin{align*}
    \cQ &= \bigg\{Q: \X\times\mathcal{A} \rightarrow \mathbb{R} \quad \text{s.t.} \\
    &Q = \spr{1 - p^\upplus_{\Lambda, \beta, \alpha}} \odot \spr{\Phi \theta + \beta \norm{\phi \spr{\cdot, \cdot}}_{\Lambda}} + p^\upplus_{\Lambda, \beta, \alpha} \cdot \frac{\RMAX}{1-\gamma}, \\
    &\beta = \widetilde{\mathcal{O}}\brr{Q_{\max} d}, ~~~~\alpha = 2 \omega, ~~~\lambda_{\max}(\Lambda) \leq 1, ~~~\lambda_{\min}(\Lambda) \geq \frac{1}{2K B L_{\max}}, \\
    &\norm{\theta} \leq \WMAX + \QMAX \LMAX K B, \norm{Q}_\infty \leq \QMAX \bigg\} \cup \scbr{\mathbf{0}}\,,
\end{align*}
where $\QMAX = H \spr{\RMAX + \frac{2 \omega}{\alpha}}$ and we included the function $0$ to make sure $Q_1 \in \cQ$. Furthermore, denote for any $\eta > 0$ the function $f_\eta: \bbR^{\cX \times \cA} \rightarrow \bbR^\cX$ defined for $Q \in \bbR^{\cX \times \cA}$ as $f_\eta \spr{Q} = \frac1\eta \log \sum_{a \in \cA} \exp \spr{\eta Q \spr{\cdot, a}}$. We then define the following function class in $\bbR^\cX$
\begin{equation} \label{eq:function-class-v}
    \cV = \bigg\{ V: \cX \rightarrow \bbR \quad \text{s.t.} \quad\exists \scbr{Q_k}_{k=1}^K, \scbr{\bar{Q}_k}_{k=1}^K \subset \cQ, V = f_\eta \circ \spr{\sum^K_{k=1} Q_k} - f_\eta \circ \spr{\sum^K_{k=1} \bar{Q}_k} \bigg\},
\end{equation}
as well as the event
\begin{equation*}
    \cE_\cV = \scbr{\forall V \in \cV, \forall k \in \sbr{K}, \norm{M V - \wh{M V}}_{\Lambda_{T_k}} \leq \beta}\,,
\end{equation*}
where $\wh{M V} = \Lambda_{T_k}^{-1} \sum_{\spr{x, a, x'} \in \cD_{T_k}} \phi \spr{x, a}  V \spr{x'}$. Finally, we define the event that the iterates of Algorithm~\ref{alg:linear-rmax-ravi-ucb} are in the function class $\cV$
\begin{align*}
    \cE_{\mathrm{in}} = \scbr{\forall k \in \sbr{K}, V_k \in \cV}\,.
\end{align*}
What remains is to show that the iterates of the algorithm belong to $\cV$, and that the event $\cE_\cV$ holds with high probability. This is done in Lemmas~\ref{lem:iterates-in-class} and \ref{lem:uniform-event-holds}, respectively. We can now prove Lemma~\ref{lem:good-event-holds}.

\goodeventholds*

\begin{proof}
    For any episode $k \in \sbr{K}$ and state-action pair $\spr{x, a} \in \cX \times \cA$, we have by Cauchy-Schwartz's inequality
    \begin{equation*}
        \abs{P V_k \spr{x, a} - \wh{P V_k} \spr{x, a}} \leq \norm{M V_k - \wh{M V_k}}_{\Lambda_{T_k}} \norm{\phi \spr{x, a}}_{\Lambda_{T_k}^{-1}}\,.
    \end{equation*}
    This inequality shows that the event $\cE_{V, \mathrm{alg}}$ implies the event $\cE_{\mathrm{valid}}$, \ie
    \begin{align*}
        \bbP \sbr{\cE_{\mathrm{valid}} \cap \cE_L} &= \bbP \sbr{\cE_{\mathrm{valid}} \given \cE_L} \bbP \sbr{\cE_L} \\
        &\geq \bbP \sbr{\cE_{V, \mathrm{alg}} \given \cE_L} \bbP \sbr{\cE_L} \\
        &\geq \bbP \sbr{\cE_{V, \mathrm{alg}} \cap \cE_{\mathrm{in}} \given \cE_L} \spr{1 - \delta}\,,
    \end{align*}
    where in the last inequality we used the monotonicity of $\bbP$ and Lemma~\ref{lem:lmax}. Then, conditioned on the event $\cE_{\mathrm{in}}$ that the iterates are in the function class $\cV$, the event $\cE_\cV$ implies $\cE_{V, \mathrm{alg}}$, that is
    \begin{align*}
        \bbP \sbr{\cE_{V, \mathrm{alg}} \cap \cE_{\mathrm{in}} \given \cE_L} &= \bbP \sbr{\cE_{V, \mathrm{alg}} \given \cE_{\mathrm{in}} \cap \cE_L} \bbP \sbr{\cE_{\mathrm{in}} \given \cE_L} \\
        &\geq \bbP \sbr{\cE_\cV \given \cE_{\mathrm{in}} \cap \cE_L} \bbP \sbr{\cE_{\mathrm{in}} \given \cE_L} \\
        &= \bbP \sbr{\cE_\cV \cap \cE_{\mathrm{in}} \given \cE_L}\,.
    \end{align*}
    Finally, by Lemma~\ref{lem:iterates-in-class} we have $\bbP \sbr{\cE_{\mathrm{in}} \given \cE_\cV, \cE_L} = 1$ thus
    \begin{align*}
        \bbP \sbr{\cE_\cV \cap \cE_{\mathrm{in}} \given \cE_L} &= \bbP \sbr{\cE_{\mathrm{in}} \given \cE_\cV, \cE_L} \bbP \sbr{\cE_\cV \given \cE_L} \\
        &= \bbP \sbr{\cE_\cV \given \cE_L} \\
        &\geq 1 - \delta\,,
    \end{align*}
    where the last inequality follows from Lemma~\ref{lem:uniform-event-holds}. In conclusion, we get
    \begin{equation*}
        \bbP \sbr{\cE_{\mathrm{valid}} \cap \cE_L} \geq \spr{1 - \delta}^2 \geq 1 - 2 \delta\,.
    \end{equation*}
\end{proof}

We now show the episodes are not too long.
\begin{lemma} \label{lem:lmax}
    Let $\delta \in \spr{0, 1}$ and define $\LMAX = H \log \spr{K / \delta}$. Then, the event $\cE_L$ holds with probability at least $1 - \delta$.
\end{lemma}

\begin{proof}
    For any $k$ and by definition of the cumulative density function of the geometric distribution with parameter $1 - \gamma$, we have that $\bbP \sbr{L_k \leq \LMAX} = 1 - \gamma^\LMAX$. Therefore, $\bbP \sbr{L_k \leq \LMAX} \geq 1 - \delta / K$ for $\LMAX \geq \frac{\log \spr{\frac{\delta}{K}}}{\log \spr{1 / \gamma}}$. Lower bounding the denominator as $\log \spr{1 / \gamma} \geq  1 - \gamma$, we have that for $\LMAX = \frac{\log \spr{K / \delta}}{1 - \gamma}$ and a union bound over $k \in \sbr{K}$, we have that  $\bbP \sbr{\cE_L} \geq 1 - \delta$.
\end{proof}

\begin{lemma} \label{lem:iterates-in-class}
    Assume the events $\cE_\cV$ and $\cE_L$ hold. Then, for all $k \in \sbr{K}$, it holds that $V_k \in \cV$, \ie $\cE_{\mathrm{in}}$ holds.
\end{lemma}

\begin{proof}
    The bound is proven by induction over $k \in [K]$. The base case holds by initialization since $Q_0 = \mathbf{0}$ is in $\cQ$. For the induction step, we assume that for all $\ell \in \sbr{k}$, $Q_\ell \in \mathcal{Q}, V_\ell \in \mathcal{V}$ and we show that $Q_{k+1} \in \cQ$ and $V_{k+1} \in \cV$.
    
    By definition of the function classes $\cQ$ and $\cV$ it holds that $\norm{Q_k}_{\infty}, \norm{V_k}_\infty \leq Q_{\max}$. $\cE_\cV$ together with the induction assumption imply that the bonuses are valid at time $k$, meaning that the derivations from Lemma~\ref{lem:qmax} guarantee that $\norm{Q_{k+1}}_\infty \leq \QMAX$. Moreover, denote $\theta_{k+1}$ the vector used to represent $Q_{k+1}$, defined as
    \begin{align*}
        \theta_{k+1} = w_k + \gamma \wh{M V_k} = w_k + \gamma \Lambda_{T_k}^{-1} \sum_{\spr{x, a, x'} \in \cD_{T_k}} \phi \spr{x, a} V_k \spr{x'}\,.
    \end{align*}
    It remains to show that $\theta_{k+1}$ satisfies the norm constraint defined in $\cQ$. By the triangular inequality and plugging the various assumptions, we have
    \begin{align*}
        \norm{\theta_{k+1}} &\leq \norm{w_{k}} + \gamma \norm{(\Lambda_{T_k})^{-1}\sum_{(x,a,x')\in \mathcal{D}_{T_k}} \phi(x,a) V_k (x')} \\
        &\leq \WMAX + \gamma \lambda_{\max}((\Lambda_{T_k})^{-1}) \abs{\cD_{T_k}} \norm{V_k}_\infty \max_{x, a} \norm{\phi \spr{x, a}}_2 \\
        &\leq \WMAX + K L_{\max}Q_{\max} B\,,
    \end{align*}
    where we also used $\gamma < 1$ in the last inequality. This proves that $Q_{k+1} \in \cQ$. Therefore, we have that $Q_\ell \in \cQ$ for $\ell \in \sbr{k+1}$. We now show that $V_{k+1} \in \cV$. Let $x \in \cX$ and $k_e$ be the initial index of the epoch $e$ such that $k \in \mathcal{K}_e$. By \citealp[Lemma~C.1]{MN23}, the sum of $V$ iterates is equal to a log-sum-exp function of the sum of $Q$ iterates. Thus,
    \begin{align*}
        V_{k+1} \spr{x} &= \sum_{i=k_e}^{k+1} V_i \spr{x} - \sum_{j=k_e}^k V_j \spr{x} \\
        &= f_\eta \spr{\sum_{i=k_e}^{k+1} Q_i} \spr{x} - f_\eta \spr{\sum_{j=k_e}^k Q_j} \spr{x}\,.
    \end{align*}
    Since $\mathbf{0} \in \cQ$ and $Q_\ell \in \cQ$ for $\ell \in \sbr{k+1}$, we can pad with zeros the two sums inside the exponentials and conclude that $V_{k+1} \spr{x}$ can be written as the difference between two log-sum-exp functions of the sum of $K$ functions in $\cQ$. Thus $V_{k+1} \in \cV$ and this concludes the induction.
\end{proof}

\begin{lemma} \label{lem:uniform-event-holds}
    Assume the event $\cE_L$ holds, and set $\beta$ as
    \begin{equation*}
        \beta = 8 Q_{\max} d \log(c \alpha \WMAX \RMAX B^{9/2} Q^4_{\max}L^{5/2}_{\max}K^{7/2} d^{5/2} \delta^{-1})\,.
    \end{equation*}
    where $c = 60 \cdot 26$. Then, the event $\cE_\cV$ holds with probability $1 - \delta$.
\end{lemma}

\begin{proof}
    Under the event $\cE_L$, invoking standard concentration results for Linear MDPs (see Lemmas D.3 and D.4 in \cite{jin2019provably}), we have that with probability $1-\delta$ it holds that
    \begin{align*}
        &\norm{MV - (\Lambda_{T_k})^{-1}\sum_{(x,a,x')\in \mathcal{D}_{T_k}} \phi(x,a) V(x')}_{\Lambda_{T_k}} \\ &~~~~~~~\leq Q_{\max} \sqrt{2 d \log\brr{\frac{1 + K L_{\max}B}{\delta}} + 4 \log \mathcal{N}_\epsilon + 8 K^2 L^2_{\max}B^2 \epsilon^2}\,,
    \end{align*}
    where $\mathcal{N}_\epsilon$ is the $\epsilon $-covering number of the class $\mathcal{V}$. In particular, for $\epsilon = (K L_{\max} B)^{-1}$, we can invoke Lemma~\ref{lem:covering-number-v} to obtain
    \begin{align*}
        \log \mathcal{N}_{\epsilon} &\leq 4d^2 \log \brr{4(\WMAX B + Q_{\max}L_{\max}K B^2 + 3\sqrt{d}+ \beta B +\RMAX)\sqrt{K^5L^3_{\max}}\alpha\beta B^{5/2}} \\
        &\leq 4d^2 \log \brr{20(3 \WMAX \beta Q_{\max}L_{\max}KB^2\sqrt{d}\RMAX)\sqrt{K^5L^3_{\max}}\alpha\beta B^{5/2}} \\
        &\leq 4d^2 \log \brr{60 \WMAX \RMAX \beta^2 \alpha B^{9/2} Q^2_{\max}L^{5/2}_{\max}K^{7/2}\sqrt{d}}\,.
    \end{align*}
    Plugging in, we have that
    \begin{align*}
        &\norm{MV - (\Lambda_{T_k})^{-1}\sum_{(x,a,x')\in \mathcal{D}_{T_k}} \phi(x,a) V(x')}_{\Lambda_{T_k}} \\ 
        &~~~~~~~\leq Q_{\max} \sqrt{2 d \log\brr{\frac{1 + K L_{\max} B}{\delta}} + 16 d^2 \log \brr{60 \beta^2 \alpha B^{9/2}Q^2_{\max}L^{5/2}_{\max}K^{7/2}\sqrt{d}} + 8} \\
        &~~~~~~~\leq Q_{\max} \sqrt{26 d^2 \log \brr{\frac{60 \WMAX \RMAX \beta^2 \alpha B^{9/2} Q^2_{\max}L^{5/2}_{\max}K^{7/2}\sqrt{d}}{\delta}}} \\ 
        &~~~~~~~ = \sqrt{26 Q^2_{\max} d^2 \log \brr{\frac{60 \WMAX \RMAX \beta^2 \alpha  B^{9/2} Q^2_{\max}L^{5/2}_{\max}K^{7/2}\sqrt{d}}{\delta}}} \\
    \end{align*}
    At this point, to find a value for $\beta$ such that
    \begin{equation*}
        \beta^2 \geq 26 Q^2_{\max} d^2 \log \brr{\frac{60 \WMAX \RMAX \beta^2 \alpha B^{9/2} Q^2_{\max}L^{5/2}_{\max}K^{7/2}\sqrt{d}}{\delta}}\,,
    \end{equation*}
    we invoke Lemma~\ref{lemma:beta_bound} with $z = 26 Q^2_{\max} d^2$ and $R = \frac{60 \WMAX \RMAX \alpha B^{9/2}Q^2_{\max}L^{5/2}_{\max}K^{7/2}\sqrt{d}}{\delta}$ which gives that the desired inequality holds for all $\beta \in \mathbb{R}$ such that
    \begin{equation*}
        \beta^2 \geq 52 Q^2_{\max} d^2 \log(c \alpha \WMAX \RMAX B^{9/2} Q^4_{\max}L^{5/2}_{\max}K^{7/2} d^{5/2} \delta^{-1})\,,
    \end{equation*}
    where $c = 60 \cdot 26$. Therefore, we select
    \begin{equation*}
        \beta = 8 Q_{\max} d \log(c \alpha \WMAX \RMAX B^{9/2} Q^4_{\max}L^{5/2}_{\max}K^{7/2} d^{5/2} \delta^{-1})\,.
    \end{equation*}
\end{proof}

\begin{remark}
    For the proof of Lemma~\ref{lem:uniform-event-holds}, we need to compute a bound on the covering number of the function class $\cV$. We find this is done in a neat and more direct way than previous analysis \cite{zhong2024theoretical,sherman2023rate,cassel2024warmupfree} that needed to introduce a policy class for the iterates $\bc{\pi_k}^K_{k=1}$ generated by \Cref{alg:linear-rmax-ravi-ucb} as an intermediate step.
\end{remark}

\subsubsection{Proof of Lemma~\ref{lem:covering-number-v} (covering number)}
\label{app:covering-number-v}

\begin{lemma} \label{lem:covering-number-v}
    Let us consider the function class $\cV$ defined in Equation~\eqref{eq:function-class-v} and an $\epsilon$-covering set $\cR \spr{\cV}$ such that for any $V \in \cV$, there exists $V' \in \cR \spr{\cV}$ such that $\norm{V - V'}_\infty \leq \frac{1}{K L_{\max} B}$. The covering number of the class $\cV$ can be bounded as follows
    \begin{equation*}
        \log \cN_{\frac1K} \leq 4 d^2 \log \spr{4 \spr{\WMAX B + \QMAX \LMAX K B^2 + 3 \sqrt{d} + \beta B + Q_{\max}} \sqrt{K^5 \LMAX^3} \alpha \beta B^{5/2}}\,.
    \end{equation*}
\end{lemma}

\begin{proof}
    We will use the following intermediate class of log sum exp state value functions
    \begin{equation*}
        \tilde{\cV} = \scbr{V: \cX \rightarrow \bbR \quad\text{s.t.}\quad \forall x, V \spr{x} = \frac1\eta \log \sum_{a \in \cA}\exp \spr{\eta \sum_{k=1}^K Q_k \spr{x, a}}, Q_k \in \cQ}\,.
    \end{equation*}
    Consider any $V, V' \in \cV$, and notice that for any $x \in \cX$,
    \begin{equation*}
        \abs{V \spr{x} - V' \spr{x}} \leq \abs{\bar{V} \spr{x} - \bar{V}' \spr{x}} + \abs{\tilde{V} \spr{x} - \tilde{V}' \spr{x}}.
    \end{equation*}
    with $\bar{V}, \tilde{V} \in \tilde{\cV}$ such that $V \spr{x} = \bar{V} \spr{x} - \tilde{V} \spr{x}$ for all $x \in \cX$ and with $\bar{V}', \tilde{V}' \in \tilde{\cV}$ such that $V' \spr{x} = \bar{V}' \spr{x} - \tilde{V}' \spr{x}$ for all $x \in \cX$. Therefore, the above bound guarantees that an $\epsilon / 2$ covering set on the function class $\tilde{\cV}$ implies an $\epsilon$ covering for the class $\cV$. Hence, in the following we focus on computing a $\epsilon / 2$ covering number for $\tilde{\cV}$. By definition of $\bar{V}, \bar{V}'$ and Lemma~\ref{lem:lse-lipschitz}, we have
    \begin{align*}
        \abs{\bar{V} \spr{x} - \bar{V}' \spr{x}} &= \abs{\frac1\eta \log \sum_{a \in \cA} \exp \brr{\eta \sum^K_{k=1} \bar{Q}_k \spr{x, a}} - \frac1\eta \log \sum_{a \in \cA} \exp \spr{\eta \sum^K_{k=1} \bar{Q}_k' \spr{x, a}}} \\
        &\leq \max_{a \in \A} \abs{ \sum^K_{k=1} \bar{Q}_k \spr{x, a} -  \sum^K_{k=1} \bar{Q}'_k \spr{x, a}} \\
        &= \norm{\sum^K_{k=1} \bar{Q}_k(x,\cdot) - \sum^K_{k=1}\bar{Q}'_k(x,\cdot) }_\infty \\
        &\leq \norm{(1 - p_{\Lambda,\beta,\alpha}^\upplus(x,\cdot))\odot \phi(x,\cdot) \sum^K_{k=1} \theta_k - (1 - p_{\Lambda',\beta,\alpha}^\upplus(x,\cdot))\odot \phi(x,\cdot) \sum^K_{k=1} \theta'_k }_\infty \\
        & \phantom{=} + \beta \norm{(1 - p_{\Lambda,\beta,\alpha}^\upplus(x,\cdot)) \odot \norm{\phi(x,\cdot)}_{\Lambda} - (1 - p_{\Lambda',\beta,\alpha}^\upplus(x,\cdot))\odot \norm{\phi(x,\cdot)}_{\Lambda'}  }_\infty \\
        &\leq B \norm{ \sum^K_{k=1} \theta_k - \sum^K_{k=1} \theta'_k}_\infty \\
        & \phantom{=} + (K(1 + KL_{\max}\QMAX B)  + \beta  B) \norm{p_{\Lambda',\beta,\alpha}^\upplus(x,\cdot) - p_{\Lambda,\beta,\alpha}^\upplus(x,\cdot)}_{\infty}\\
        &\phantom{=}+ \beta B \norm{\norm{\phi(x,\cdot)}_{\Lambda} - \norm{\phi(x,\cdot)}_{\Lambda'}}_{\infty}.
    \end{align*}
    Above we have denoted the parameters of the state-action value function $\bar{Q}_k$ as $\theta_k$ and  $\Lambda$ for $k \in [K]$, for $\bar{Q}'_k$ we used parameters $\theta'_k$ and  $\Lambda'$ for $k \in [K]$ . 
    We do not need to consider an indexed covariance matrices because, due to the slow-changing bonus design, the covariance matrix remains fixed within epochs.
    Moreover, by $1$-Lipschitzness of the sigmoid function we have that
    \[
    \norm{p_{\Lambda',\beta,\alpha}^\upplus(x,\cdot) - p_{\Lambda,\beta,\alpha}^\upplus(x,\cdot)}_{\infty} \leq \alpha\beta \norm{\norm{\phi(x,\cdot)}_{\Lambda} - \norm{\phi(x,\cdot)}_{\Lambda'}}_{\infty}
    \]
    and that
    \begin{align*}
\norm{\norm{\phi(x,\cdot)}_{\Lambda} - \norm{\phi(x,\cdot)}_{\Lambda'}}_{\infty} &= \norm{\norm{\Lambda^{1/2}\phi(x,\cdot)} - \norm{(\Lambda')^{1/2}\phi(x,\cdot)}} \\
        &= \norm{(\Lambda^{1/2} - (\Lambda')^{1/2} )\phi(x,\cdot)} \\&\leq
        B \norm{\Lambda^{1/2} - (\Lambda')^{1/2}} \\
        &\leq \sqrt{\frac{B K L_{\max}}{2}} \norm{\Lambda - \Lambda'}_F
    \end{align*}
    where we have used \cite[Lemma 17]{cassel2024warmupfree}, 
    which ensures that $\norm{\Lambda^{1/2} - (\Lambda')^{1/2}} \leq \frac{1}{2\sqrt{\lambda_{\min}}}\norm{\Lambda - \Lambda'} = \sqrt{\frac{B K L_{\max}}{2}} \norm{\Lambda - \Lambda'}$ and in the last step we have upper bounded the spectral norm with the Frobenious norm. 
    
    Therefore, all in all, we have that 
    \begin{align*}
        \abs{\bar{V}(x) - \bar{V}'(x)} &\leq B \norm{ \sum^K_{k=1}{\theta}_k -  \sum^K_{k=1} {\theta}'_k}_\infty \\
        & \phantom{=} + \alpha \beta K (1 + KL_{\max}\QMAX B  + 2 \beta B ) \sqrt{\frac{B K L_{\max}}{2}} \norm{{\Lambda} - {\Lambda}'}_F \\
    \end{align*}
    therefore, an $\epsilon$-covering number for the class $\mathcal{V}$ is obtained via appropriately fine-grained coverings of the euclidean ball of radius for the $\theta$ parameters and Frobenius norm ball for the $\Lambda$ parameters. We carry out these final calculations in the following.
    At this point, if we have a $\epsilon_{\Lambda}$-covering set for the set
    \begin{equation*}
        \mathbf{\Lambda} = \bc{\Lambda \in \mathbb{R}^{d\times d} : \lambda_{\max}(\Lambda) \leq 1, ~~~\lambda_{\min}(\Lambda) \geq \frac{1}{2BKL_{\max}}}
    \end{equation*}
    and an $\epsilon_\theta$-covering set for the set
    \begin{equation*}
        \mathbf{\Theta} = \bc{\theta \in \mathbb{R}^d: \norm{\theta} \leq K(\WMAX + Q_{\max}L_{\max}K B)}
    \end{equation*}
    we would have that
    \begin{align*}
        \abs{\bar{V}(x) - \bar{V}'(x)} &\leq 2\sqrt{ K^3 L_{\max}}\alpha\beta B^{3/2} (\norm{\theta} B + \beta B +Q_{\max} + 1) \epsilon_{\Lambda} + 2 B K \epsilon_\theta \\
        &\leq 2 \sqrt{K^3L_{\max}}\alpha \beta B^{3/2} (\norm{\theta} B + \beta B +Q_{\max} + 1)  \brr{\epsilon_{\Lambda} + \epsilon_\theta},
    \end{align*}
    where in the last inequality we assumed that $\beta \geq 1$ and $B\geq 1$.
    Therefore, to have an $\epsilon$-covering set for $\mathcal{V}$, we need to construct an $\epsilon_{\Lambda}$-covering set for $\mathbf{\Lambda}$, where
    \begin{equation*}
        \epsilon_\Lambda = \frac{\epsilon}{4 \sqrt{K^3L_{\max}}\alpha\beta B^{3/2} (\norm{\theta} B + \beta B +Q_{\max} + 1)}
    \end{equation*}
    and an $\epsilon_\theta = \frac{\epsilon}{4 \sqrt{K^3L_{\max}}\alpha\beta B^{3/2} (\norm{\theta} B + \beta B +Q_{\max} + 1)}$-covering set for $\mathbf{\Theta}$. Then, using the fact that the $\epsilon$-covering number for the Euclidean ball of radius $R$ in $d$ dimension is given by $(1 + 2 R/\epsilon)^d$, we obtain
    \begin{equation*}
        \log \mathcal{N}_{\epsilon_{\theta}}(\mathbf{\Theta}) \leq d \log  \brr{1 + 8\frac{(\WMAX + Q_{\max}L_{\max}K B )\sqrt{K^3L_{\max}}\alpha\beta B^{3/2} (\norm{\theta} B + \beta B +Q_{\max} + 1)}{\epsilon}}
    \end{equation*}
    Moreover, noticing that for all matrices $\Lambda \in \mathbf{\Lambda}$ it holds that $\norm{\Lambda}_F \leq \sqrt{d} \lambda_{\max}(\Lambda) \leq \sqrt{d}$, we need to cover the Frobenious norm ball with radius $\sqrt{d}$. Recalling that the Frobenious norm of a matrix is equivalent to the euclidean norm of the vectorization of the matrix, this equivalent to cover the euclidean ball in $\mathbb{R}^{d^2}$ with radius $\sqrt{d}$.
    \begin{equation*}
        \log \mathcal{N}_{\epsilon_{\Lambda}}(\mathbf{\Lambda}) 
        \leq d^2 \log  \brr{1 + 8\sqrt{d}\frac{\sqrt{K^3L_{\max}}\alpha\beta B^{3/2} (\norm{\theta} B + \beta B +Q_{\max} + 1)}{\epsilon}}.
    \end{equation*}
    Therefore, using the fact that
    \begin{align*}
        &\log \mathcal{N}_{\epsilon}(\mathcal{V})=  
         \log \mathcal{N}_{\epsilon_{\Lambda}}(\mathbf{\Lambda}) + 
        \log \mathcal{N}_{\epsilon_{\theta}}(\mathbf{\Theta}) \\
        \\ & \leq 2 d \log  \brr{1 + 8\frac{(\WMAX + Q_{\max}L_{\max}K B )\sqrt{K^3L_{\max}}\alpha\beta B^{3/2} 
        (\norm{\theta} B + \beta B +Q_{\max} + 1)}{\epsilon}} \\&\phantom{=}+ 
         d^2 \log  \brr{1 + 8\sqrt{d}\frac{\sqrt{K^3L_{\max}}\alpha\beta B^{3/2} (\norm{\theta} B + \beta B +Q_{\max} + 1)}{\epsilon}} \\
        \\ & \leq 2 d^2\log   \brr{1 + 8\frac{(\WMAX + Q_{\max}L_{\max}K B + \sqrt{d})\sqrt{K^3L_{\max}}\alpha\beta B^{3/2} (\norm{\theta} B + 
        \beta B +Q_{\max} + 1)}{\epsilon}} \\
        & \leq 2 d^2 \log  \brr{1 + 8\frac{(\WMAX B + Q_{\max}L_{\max}K B^2 + \sqrt{d}+ \beta B +Q_{\max} +1)^2\sqrt{K^3L_{\max}}\alpha\beta B^{3/2}}{\epsilon}} \\
        & \leq 2 d^2 \log  \brr{16\frac{(\WMAX B + Q_{\max}L_{\max}K B^2 + 3\sqrt{d}+ \beta B +Q_{\max})^2\sqrt{K^3L_{\max}}\alpha\beta B^{3/2}}{\epsilon}} \\
        & \leq 4d^2 \log  \brr{4\frac{(\WMAX B + Q_{\max}L_{\max}K B^2  + 3\sqrt{d}+ \beta B +Q_{\max})\sqrt{K^3L_{\max}}\alpha\beta B^{3/2}}{\epsilon}} \\
        & = 4d^2 \log  \brr{4(\WMAX B + Q_{\max}L_{\max}K B^2 + 3\sqrt{d}+ \beta B +Q_{\max})\sqrt{K^5L^3_{\max}}\alpha\beta B^{5/2}}.
    \end{align*}
    where we used $d > 1$ and the last step uses the fact that we are looking for a $\epsilon = \frac{1}{K L_{\max} B}$ covering set. Finally,
    \begin{equation*}
        \log \mathcal{N}_{\epsilon} \leq 4 d^2 \log  \brr{4(\WMAX B +  Q_{\max}L_{\max}K B^2 + 3\sqrt{d}+ \beta B +Q_{\max})\sqrt{K^5L^3_{\max}}\alpha\beta B^{5/2}}\,.
    \end{equation*}
\end{proof}

\subsection{Putting everything together (proof of Theorem~\ref{thm:main})}
\label{app:putting-together-main}

\begin{theorem} \label{thm:main-full}
    Run Algorithm~\ref{alg:linear-rmax-ravi-ucb} with parameters $\omega = \log K$, $\alpha = 2 \log K$,
    \begin{align*}
        \eta = \sqrt{\frac{5 d \log \spr{1 + B^2 T / d} \log \abs{\cA}}{8 \RMAX^2 H^{5 / 2} K}}, \quad\text{and } \beta = C H \RMAX d \log \spr{B H \WMAX \RMAX d K \delta^{-1}}\,,
    \end{align*}
    for some absolute constant $C > 0$ and $\delta \in \spr{0, 1}$. Then, with probability at least $1 - \delta$, we have
    \begin{align*}
        \regretK &= \tilde\cO \spr{\sqrt{d^3 H^3 K} + \sqrt{d H^{9/2} K \log \spr{\abs{\cA}}}} \\
        &= \tilde\cO \spr{\sqrt{d^3 H^2 T} + \sqrt{d H^{7/2} T \log \spr{\abs{\cA}}}}\,.
    \end{align*}
\end{theorem}

\begin{proof}
    We are now ready to prove Theorem~\ref{thm:main}. Combining Lemma~\ref{lem:reward-bias-bound}, and the bounds in Equations~\eqref{eq:model-bias-bounds} and Lemma~\ref{lem:bound-regret-plus} we first get
    \begin{align*}
        \frac1H \regretK &\leq 2 \RMAX H \sumkK \inp{\mu \spr{\pi_k}, p_k^\upplus} + 4 \QMAX E \spr{K} + \frac{E \spr{K} \log \abs{\cA}}{\eta} \\
        &\phantom{=}+ 2 \eta \QMAX^2 \sqrt{H} K + 2 \sumkK \inp{\mu \spr{\pi_k}, \spr{1 - p_k^\upplus} \odot \CB_k}\,.
    \end{align*}
    Using the bound on the ascension functions provided in Inequality~\ref{eq:ascension-function-bound} and $1 - p_k^\upplus \preceq 1$, we further have
    \begin{align*}
        \frac1H \regretK &\leq 4 \RMAX H \alpha^2 \sumkK \inp{\mu \spr{\pi_k}, \CB_k^2} + 4 \RMAX e^{- \omega} H K + 4 \QMAX E \spr{K} \\
        &\phantom{=}+ \frac{E \spr{K} \log \abs{\cA}}{\eta} + 2 \eta \QMAX^2 \sqrt{H} K + 2 \sumkK \inp{\mu \spr{\pi_k}, \CB_k}\,.
    \end{align*}
    Lemma~\ref{lem:expected-bonuses-bound} can be used to bound the bonuses
    \begin{align*}
        \frac1H \regretK &\leq 32 \RMAX \alpha^2 \beta^2 B^2 d \log \spr{1 + \frac{B^2 T}{d}} + 16 \RMAX H \alpha^2 \beta^2 B^2 \log \spr{\frac{2 K}{\delta}}^2 \\
        &\phantom{=}+ 4 \RMAX e^{- \omega} H K + 4 \QMAX E \spr{K} + \frac{E \spr{K} \log \abs{\cA}}{\eta} + 2 \eta \QMAX^2 \sqrt{H} K \\
        &\phantom{=}+ \frac{8 \beta B}{H} \sqrt{d T \log \spr{1 + \frac{B^2 T}{d}}} + 8 \beta B \log \spr{\frac{2 K}{\delta}}^2\,.
    \end{align*}
    Following Lemmas~\ref{lem:qmax} and \ref{lem:lmax}, we plug the values of $\QMAX = H \spr{\RMAX + \frac{2 \omega}{\alpha}}$ and $\LMAX = H \log \spr{K / \delta}$,
    \begin{align*}
        \frac1H \regretK &\leq 32 \RMAX \alpha^2 \beta^2 B^2 d \log \spr{1 + \frac{B^2 T}{d}} + 16 \RMAX H \alpha^2 \beta^2 B^2 \log \spr{\frac{2 K}{\delta}}^2 \\
        &\phantom{=}+ 4 \RMAX e^{- \omega} H K + 4 H \spr{\RMAX + \frac{2 \omega}{\alpha}} E \spr{K} + \frac{E \spr{K} \log \abs{\cA}}{\eta} \\
        &\phantom{=}+ 2 \eta \spr{\RMAX + \frac{2 \omega}{\alpha}}^2 H^{5/2} K + \frac{8 \beta B}{H} \sqrt{d T \log \spr{1 + \frac{B^2 T}{d}}} \\
        &\phantom{=}+ 8 \beta B \log \spr{\frac{2 K}{\delta}}^2\,.
    \end{align*}
    By Lemma~\ref{lemma:number-epochs-bound}, we can bound $E \spr{K} \leq 5 d \log \spr{1 + \frac{B^2 T}{d}}$
    \begin{align*}
        \frac1H \regretK &\leq 32 \RMAX \alpha^2 \beta^2 B^2 d \log \spr{1 + \frac{B^2 T}{d}} + 16 \RMAX \alpha^2 \beta^2 B^2 H \log \spr{\frac{2 K}{\delta}}^2 \\
        &\phantom{=}+ 4 \RMAX e^{- \omega} H K + 20 H d \spr{\RMAX + \frac{2 \omega}{\alpha}} \log \spr{1 + \frac{B^2 T}{d}} \\
        &\phantom{=}+ \frac{5 d}{\eta} \log \spr{1 + \frac{B^2 T}{d}} \log \abs{\cA} + 2 \eta \spr{\RMAX + \frac{2 \omega}{\alpha}}^2 H^{5/2} K \\
        &\phantom{=}+ \frac{8 \beta B}{H} \sqrt{d T \log \spr{1 + \frac{B^2 T}{d}}} + 8 \beta B \log \spr{\frac{2 K}{\delta}}^2\,.
    \end{align*}
    It remains to choose the parameters. We start by setting $\alpha = 2 \omega$ and use $\RMAX \geq 1$ to get
    \begin{align*}
        \frac1H \regretK &\leq 128 \RMAX \omega^2 \beta^2 B^2 d \log \spr{1 + \frac{B^2 T}{d}} + 64 \RMAX \omega^2 \beta^2 B^2 H \log \spr{\frac{2 K}{\delta}}^2 \\
        &\phantom{=}+ 4 \RMAX e^{- \omega} H K + 40 H d \RMAX \log \spr{1 + \frac{B^2 T}{d}} \\
        &\phantom{=}+ \frac{5 d}{\eta} \log \spr{1 + \frac{B^2 T}{d}} \log \abs{\cA} + 8 \eta \RMAX^2 H^{5/2} K \\
        &\phantom{=}+ \frac{8 \beta B}{H} \sqrt{d T \log \spr{1 + \frac{B^2 T}{d}}} + 8 \beta B \log \spr{\frac{2 K}{\delta}}^2\,.
    \end{align*}
    Then, we set $\omega = \log K$
    \begin{align*}
        \frac1H \regretK &\leq 128 \RMAX \beta^2 B^2 d \log \spr{K}^2 \log \spr{1 + \frac{B^2 T}{d}} + 64 \RMAX \beta^2 B^2 H \log \spr{K}^2 \log \spr{\frac{2 K}{\delta}}^2 \\
        &\phantom{=}+ 4 \RMAX H + 40 H d \RMAX \log \spr{1 + \frac{B^2 T}{d}} \\
        &\phantom{=}+ \frac{5 d}{\eta} \log \spr{1 + \frac{B^2 T}{d}} \log \abs{\cA} + 8 \eta \RMAX^2 H^{5/2} K \\
        &\phantom{=}+ \frac{8 \beta B}{H} \sqrt{d T \log \spr{1 + \frac{B^2 T}{d}}} + 8 \beta B \log \spr{\frac{2 K}{\delta}}^2\,.
    \end{align*}
    We choose the learning rate as $\eta = \sqrt{\frac{5 d \log \spr{1 + B^2 T / d} \log \abs{\cA}}{8 \RMAX^2 H^{5 / 2} K}}$ and we obtain
    \begin{align*}
        \frac1H \regretK &\leq 128 \RMAX \beta^2 B^2 d \log \spr{K}^2 \log \spr{1 + \frac{B^2 T}{d}} + 64 \RMAX \beta^2 B^2 H \log \spr{K}^2 \log \spr{\frac{2 K}{\delta}}^2 \\
        &\phantom{=}+ 4 \RMAX H + 40 H d \RMAX \log \spr{1 + \frac{B^2 T}{d}} \\
        &\phantom{=}+ 4 \sqrt{10 \RMAX^2 H^{5/2} d \log \spr{1 + \frac{B^2 T}{d}} \log \spr{\abs{\cA}} K} \\
        &\phantom{=}+ \frac{8 \beta B}{H} \sqrt{d T \log \spr{1 + \frac{B^2 T}{d}}} + 8 \beta B \log \spr{\frac{2 K}{\delta}}^2\,.
    \end{align*}
    Finally, following Lemma~\ref{lem:good-event-holds} we set $\beta = C H \RMAX d \log \spr{B H \WMAX \RMAX d K \delta^{-1}}$ where $C > 0$ is an absolute constant and we get
    \begin{align*}
        \frac1H \regretK &\leq 128 C^2 \RMAX^3 B^2 d^3 H^2 \log \spr{K}^2 \log \spr{1 + \frac{B^2 T}{d}} \log \spr{B H \WMAX \RMAX d K \delta^{-1}}^2 \\
        &\phantom{=}+ 64 C^2 \RMAX^3 d^2 B^2 H^3 \log \spr{K}^2 \log \spr{\frac{2 K}{\delta}}^2 \log \spr{B H \WMAX \RMAX d K \delta^{-1}}^2 \\
        &\phantom{=}+ 4 \RMAX H + 40 H d \RMAX \log \spr{1 + \frac{B^2 T}{d}} \\
        &\phantom{=}+ 4 \sqrt{10 \RMAX^2 H^{5/2} d \log \spr{1 + \frac{B^2 T}{d}} \log \spr{\abs{\cA}} K} \\
        &\phantom{=}+ 8 C \RMAX B d \sqrt{d T \log \spr{1 + \frac{B^2 T}{d}}} \log \spr{B H \WMAX \RMAX d K \delta^{-1}} \\
        &\phantom{=}+ 8 C \RMAX d B H^2 \log \spr{\frac{2 K}{\delta}}^2 \log \spr{B H \WMAX \RMAX d K \delta^{-1}}\,.
    \end{align*}
    After multiplying by $H$, we get
    \begin{align*}
        \regretK &= \tilde\cO \spr{\sqrt{d^3 H^3 K} + \sqrt{d H^{9/2} K \log \spr{\abs{\cA}}}} \\
        &= \tilde\cO \spr{\sqrt{d^3 H^2 T} + \sqrt{d H^{7/2} T \log \spr{\abs{\cA}}}}\,.
    \end{align*}
\end{proof}

%% file: sections/related_works_IL.tex
\section{Motivation for \emph{Learning from Features Alone} and related works in imitation learning}
\label{app:related_works_IL}

\paragraph{Related works in theoretical imitation learning.} A special case of our setting is imitation learning from state-only expert trajectories, which is recovered when $\phi_\cost (x,a) = \mathbf{e}_x$. 
This setting was first studied in \cite{sun2019provably} in the finite-horizon setting with general function approximation. 
There are some notable differences between their work and ours, primarily that they focus on the finite-horizon setting and learn a non-stationary policy. 
In principle, their algorithm could be applied to the infinite-horizon setting by truncating the trajectories after $\tilde{\mathcal{O}}(1-\gamma)^{-1}$ steps. 
However, this would still result in a non-stationary policy, whereas our approach outputs a stationary policy. 
Their realizability assumption on the expert policy and expert state-value function is not required in our work which leverages, instead, the linear MDP assumption. 
These assumptions are not directly comparable, even when the function classes in \cite{sun2019provably} are assumed to be linear. 
Indeed, the realizability assumption imposed in \cite{sun2019provably} would imply having access to the values of the features $\sum_a \expert(a|x) \phi(x,a)$ for each state $x \in \cX$. 
In contrast, our approach does not require this additional knowledge about the expert.

Furthermore, the guarantees on the number of expert trajectories in \cite[Theorem 3.3]{sun2019provably} adapted to the infinite-horizon setting, 
would scale as $\tilde{\mathcal{O}}((1-\gamma)^{-4} \epsilon^{-2})$ whereas we only require $\tilde{\mathcal{O}}((1-\gamma)^{-2} \epsilon^{-2})$ state-only samples from the expert occupancy measure.

Similarly, \cite{ADKLS20} develop a framework for imitation and representation learning from observation alone based on bilevel optimization but assume the realizability of the state-value function, which is not needed in our work.

The work of \citet{kidambi2021mobile} investigates the idea of exploration in state-only imitation learning. 
Unlike our work, they focus on the finite-horizon setting and on different structural assumptions on the MDP. 
Specifically, \citet{kidambi2021mobile} consider tabular MDPs, nonlinear kernel regulators, and MDPs with Gaussian transition kernels and bounded Eluder dimension, whereas our work focuses on infinite-horizon linear MDPs and observing only the feature directions visited by the expert, which is a weaker requirement than observing the states directly.
Moreover, our algorithm \FRAalg is computationally efficient, whereas the model fitting step in \cite{kidambi2021mobile} cannot be implemented efficiently for various situations, including linear MDPs \citep{jin2019provably} and KNRs \citep{Kakade:2020}.

\citet{wu2024diffusing} operate under a different set of assumptions, namely that the learner has access to a function class for the expert's score function and that the expected state norm remains bounded during learning. 
Under this setting, the authors are the first to achieve first- and second-order bounds for imitation learning, which lead to a faster rate in the case of low-variance expert policies and transitions. The authors do not quantify the MDP trajectory complexity, but it would scale suboptimally with $1/\epsilon$ because they require an expensive \emph{RL in the loop} routine that we avoid in our work.

\citet{xu2022understanding} develop an analysis for horizon-free bounds on $\tau_E$ for a special class of MDPs, where expert states can be visited only by visiting all preceding expert states.

The trajectory access to the MDP $\mathcal{M}\setminus \true$ assumed in this work should not be confused with interactive/online imitation learning, where the expert can be queried during learning \citep{Ross:2010,Ross:2011,swamy2021moments,li2022efficient,lavington2022improved,sekhari2024selective,sun2017deeply,sekhari2024contextual}. Furthermore, our trajectory access is a much weaker requirement compared to generative model access used in \citep{swamy2022minimax,Kamoutsi:2021}.

Moreover, it is important to note that we do not require any ergodicity or self-exploration properties of the dynamics, whereas such assumptions are needed in \citep{viano2022proximal,zeng2022maximum}. Additionally, uniformly good evaluation error, which is essentially possible only under generative model or ergodic dynamics assumptions, is required in \citep{wu2023inverse,zeng2022structural,zeng2023understanding}. 
Also, the use of exploration bonuses in imitation learning has also been useful for the related problem of finding the reward feasible set without using a generative model \citep{lazzati2024scale,lindner2022active}.

\begin{table}[t]
    \caption{\label{tab:literature} Comparison with related imitation learning algorithms.}
    \setlength{\tabcolsep}{4pt}  
    \renewcommand{\arraystretch}{1.5}  
    \resizebox{\textwidth}{!}{%
    \begin{tabular}{|c|M{4cm}|M{1.5cm}|c|c|}
        \hline
        \textbf{Algorithm}                                         & \textbf{Setting}                                                & \textbf{F.O.}            & \textbf{Expert Traj. $(\tau_E)$}                                       & \textbf{MDP Traj. $(K)$}                                                              \\ \hline
        \multirow{3}{*}{Behavioural Cloning}                       & Function Approximation, Episodic \cite{foster2024behavior}      & \redcross                & $\cO \brr{\frac{H^2 \log \abs{\Pi}}{\epsilon^2}}$                      & \multicolumn{1}{c|}{-}                                                                \\ \cline{2-4}
                                                                   & Tabular, Episodic \cite{rajaraman2020toward}                    & \redcross                & $\widetilde{\cO} \brr{\frac{H^2 \abs{\cX}}{\epsilon}}$                 & \multicolumn{1}{c|}{-}                                                                \\ \cline{2-4}
                                                                   & Deterministic Linear Expert, Episodic \cite{rajaraman2021value} & \redcross                & $\widetilde{\cO} \brr{\frac{H^2 d}{\epsilon}}$                         & \multicolumn{1}{c|}{-}                                                                \\ \hline
        Mimic-MD \cite{rajaraman2020toward}                        & Tabular, Known $P$, Deterministic Expert, Episodic              & \redcross                & $\cO \brr{\frac{H^{3/2} \abs{\cX}}{\epsilon}}$                         & \multicolumn{1}{c|}{-}                                                                \\ \hline
        OAL \cite{Shani:2021}                                      & Episodic Tabular                                                & \redcross                & $\tilde{\cO} \brr{\frac{H^2 \abs{\cX}}{\epsilon^{2}}}$                 & $\tilde{\mathcal{O}}\brr{\frac{H^4 \abs{\X}^2 \abs{\aspace} }{\epsilon^{2}}}$         \\ \hline
        MB-TAIL \cite{xu2023provably}                              & Episodic, Tabular, Deterministic Expert                         & \redcross                & $\cO \brr{\frac{H^{3/2} \abs{\cX}}{\epsilon}}$                         & $\mathcal{O}\brr{\frac{H^3 \abs{\X}^2 \abs{\aspace} }{\epsilon^{2}}}$                 \\ \hline
        FAIL \cite{sun2019provably}                                & Episodic, $\expert \in \Pi$ and $V^\expert\in\mathcal{F}$       & \greentick$^\star$       & $\tilde{\cO} \brr{\frac{H^4 \log(\abs{\Pi} \abs{\cF} H)}{\epsilon^2}}$ & $\tilde{\mathcal{O}}\brr{\frac{H^4 \log(\abs{\Pi}\abs{\mathcal{F}} H) }{\epsilon^2}}$ \\ \hline
        Mobile \cite{kidambi2021mobile}                            & Episodic, $\true \in \mathcal{R}$ and $P \in \mathcal{P}$       & \greentick$^\star$       & $\tilde{\cO} \brr{\frac{H^2 \log (\abs{\mathcal{R}} H)}{\epsilon^2}}$  & $\tilde{\mathcal{O}}\brr{\frac{H^5 \log \abs{\mathcal{P}}}{\epsilon^2}}$              \\ \hline
        OGAIL \cite{Liu:2022}                                      & Episodic Linear Mixture MDP, $\WMAX=\sqrt{d}$                   & \greentick               & $\tilde{\cO} \brr{\frac{H^{3} d^2}{\epsilon^{2}}}$                     & $\tilde{\mathcal{O}}\brr{\frac{H^4 d^3}{\epsilon^{2}}}$                               \\ \hline
        ILARL \cite{viano2024imitation}                            & Linear MDP, $\WMAX=1$                                           & \greentick               & $\tilde{\cO} \brr{\frac{d}{(1 - \gamma)^{2} \epsilon^{2}}}$            & $\tilde{\mathcal{O}}\brr{ \frac{d^3 }{(1 - \gamma)^{8} \epsilon^{4}}}$                \\ \hline
        \FRAalg (This Work)                                        & Linear MDP                                                      & \greentick               & $\tilde{\cO} \brr{\frac{\WMAX^2}{(1 - \gamma)^{2} \epsilon^{2}}}$      & $\tilde{\mathcal{O}}\brr{ \frac{d^3 }{(1 - \gamma)^{4.5} \epsilon^{2}}}$              \\ \hline
        \textbf{Lower Bound} (This Work)                           & Linear MDP                                                      & \greentick               & $\Omega\brr{ \frac{\WMAX^2}{(1 - \gamma)^{2} \epsilon^{2}}}$           & $\Omega\brr{ \frac{d }{(1 - \gamma)^{2} \epsilon^{2}}}$                               \\ \hline
    \end{tabular}}
\end{table}

Finally, we present Table~\ref{tab:literature}, which compares our bounds with existing ones.
 We show the number of expert trajectories and MDP interactions required for $\epsilon$-suboptimal expected performance. 
 The acronym F.O.\ refers to "Features Only" and indicates whether the algorithm applies to the setting we consider here. 
 The star \greentick$^\star$ specifies that the algorithm only applies to state-only imitation learning. "Linear expert" refers to the case where an expert policy is of the form
\begin{align*}
    \pi \spr{x} = \max_{a \in \aspace} \phi \spr{x, a}\transpose \theta\,,
\end{align*}
for some vector $\theta$. Finally, in the work by \cite{kidambi2021mobile}, the bound on $K$ can be tighter than what we report in the table. We report this slightly looser version for sake of simplicity and avoiding to introduce the information gain (see \cite{kidambi2021mobile} for details).

%% file: sections/applications_proofs.tex
\section{Omitted proofs for Section~\ref{sec:application}}
\label{app:proof_IL}

To improve readability, we define the feature expectation vector as $\lambda(\pi) = \phim_r\transpose \mu(\pi)$ for any policy $\pi$, where $\mu(\pi_k)$ denotes the occupancy measure of policy $\pi_k$. This notation will be used in the following proofs.

\subsection{Proof of Theorem~\ref{thm:FraUpper} (guarantee for the output of Algorithm~\ref{alg:fra})}

\FraUpper*

\begin{proof}
    Using the decomposition presented in \Cref{sec:application}, we can express the regret as
    \begin{equation*}
        (1 - \gamma) \regretIL = \underbrace{\sum^K_{k=1} \innerprod{r_k}{\mu\brr{\expert} - \mu\brr{\pi_k}}}_{(1-\gamma)\regretK^\pi(\mu\brr{\expert})} + \underbrace{\sum^K_{k=1} \innerprod{\Phi_r\transpose\mu(\pi_k) - \Phi_r\transpose\mu(\expert)}{w_k - w_{\mathrm{true}}}}_{(1-\gamma)\regretK^w(w_{\mathrm{true}})} \label{eq:dec}\,.
    \end{equation*}
    By bounding $\regret^w(w_{\mathrm{true}})$ and $\regret^\pi(\mu(\expert))$ using \Cref{thm:reward_regret_bound} and \Cref{thm:main}, respectively, we obtain that with probability $1 - 3 \delta$
    \begin{align*}
        \frac{1}{K} \regretIL &\leq  10 H (B \WMAX+\RMAX) \sqrt{ \frac{\log \delta^{-1}}{K}} + 24 H\WMAX B \sqrt{\frac{ \log\brr{\frac{1}{\delta}}}{\tau_E}} \\&\phantom{=}+ \tilde{\mathcal{O}}(d^{3/2}(1 - \gamma)^{-9/4} \log^{1/2} \abs{\aspace} K^{-1/2})\,.
    \end{align*}
    Therefore, by considering $B$ and $\RMAX$ as constants and choosing $K = \widetilde{\mathcal{O}}\brr{\frac{ d^{3} \log (\abs{\aspace}\delta^{-1})}{(1 - \gamma)^{4.5}\varepsilon^2}}$ and $\tau_E = \widetilde{\mathcal{O}}\brr{\frac{ \WMAX^2 \log(1/\delta)}{(1 - \gamma)^2\varepsilon^2}}$ we have that with probability $1 - 3 \delta$ it holds that
    \begin{equation*}
        \frac{1}{K} \regretIL \leq 4 \epsilon\,.
    \end{equation*}
    Since $\frac{1}{K} \regretIL$ is a random variable bounded by $(1-\gamma)^{-1}$ almost surely, in expectation we have the following bound
    \begin{equation*}
        \bbE_{\mathrm{Alg}} \bs{\frac{1}{K} \regretIL} \leq \frac{3 \delta}{1-\gamma} + 4 \varepsilon\,.
    \end{equation*}
    Thus, by choosing $\delta \leq \nicefrac{\varepsilon}{3(1-\gamma)}$ we can conclude that
    \begin{equation*}
        \bbE_{\mathrm{Alg}} \bs{\frac{1}{K} \regretIL} \leq 5 \varepsilon\,.
    \end{equation*}
    Finally, by selecting $\pi^{\mathrm{out}}$ uniformly at random from the policies generated by \Cref{alg:fra} 
    we have that
    \begin{equation*}
        \bbE_{\mathrm{Alg}} \innerprod{\initial}{V_{\true}^{\expert} - V_{\true}^{\pi^{\mathrm{out}}}} \leq 5 \varepsilon\,.
    \end{equation*}
\end{proof}

\subsection{Proof of Theorem~\ref{thm:reward_regret_bound} (regret bound for the reward player)}

\begin{restatable}{theorem}{regretreward} \label{thm:reward_regret_bound}
    Assume that $w_{\mathrm{true}} \in \cW$ for some non-empty closed convex set $\cW$ and that for any $w \in \cW$, $\norm{w}\leq \WMAX$. Then, OGD with $\eta_r = \nicefrac{\WMAX}{B \sqrt{K}}$ ran for $K$ iterations satisfies with probability at least $1 - 2 \delta$ that
    \begin{equation*}
        \regretK^w(w_{\mathrm{true}}) \leq 10 H (B \WMAX + \RMAX) \sqrt{K \log 1 / \delta} + 24 H \WMAX K B \sqrt{\frac{\log\brr{\frac{1}{\delta}}}{\tau_E}}\,.
    \end{equation*}
\end{restatable}

\begin{proof}
    Given the definition of the feature expectation vector $\lambda (\pi)$, we can rewrite the regret for the reward player as follows
    \begin{equation*}
        \spr{1 - \gamma} \regret^w(w_{\mathrm{true}}) = \sum^K_{k=1} \innerprod{\lambda(\pi_k) - \lambda(\expert)}{w_k - w_{\mathrm{true}}}\,.
    \end{equation*}
    Then, adding and subtracting the estimators for the occupancy measures, we get
    \begin{align*}
        (1-\gamma) \regret^w(w_{\mathrm{true}}) &= \sum^K_{k=1} \innerprod{\phi_\cost(X_k, A_k) - \widehat{\lambda(\expert)}}{w_k - w_{\mathrm{true}} } \\
        &\phantom{=}+ \sum^K_{k=1} \innerprod{\lambda(\pi_k) - \phi_\cost(X_k, A_k)}{w_k - w_{\mathrm{true}}} \\
        &\phantom{=}+ \sum^K_{k=1} \innerprod{\widehat{\lambda(\expert)} - \lambda(\expert)}{ w_k-w_{\mathrm{true}}}\,.
    \end{align*}
    Now, using the regret bound for OGD \citep{Zin03}, we can bound the first term in the decomposition above as
    \begin{align*}
        \sum^K_{k=1} &\innerprod{\phi_\cost(X_k, A_k) - \widehat{\lambda(\expert)}}{w_k-w_{\mathrm{true}} } \\&\leq  \frac{\max_{w\in\mathcal{W}}\norm{w_{\mathrm{true}} - w_1}_2^2}{2 \eta_r} + \frac{\eta_r}{2} \sum^K_{k=1} \norm{\widehat{\lambda(\expert)} - \phi_\cost(X_k, A_k)}_2^2 \\
        &\leq \frac{2 \WMAX^2}{\eta_r} + 2 \eta_r B^2 K\,,
    \end{align*}
    Looking at the term $\sum^K_{k=1} \innerprod{\lambda (\pi_k) - \phi_\cost (X_k, A_k)}{w_k - w_{\mathrm{true}}}$, we notice that
    \begin{equation*}
        \psi_k = \innerprod{\lambda(\pi_k) - \phi_\cost(X_k, A_k)}{w_k - w_{\mathrm{true}}}
    \end{equation*}
    is a martingale difference sequence such that
    \begin{equation*}
        \abs{\innerprod{\lambda(\pi_k) - \phi_\cost(X_k, A_k)}{w_k - w_{\mathrm{true}}}} \leq  4 \RMAX\,.
    \end{equation*}
    Applying Azuma-Hoeffding's inequality, we have that with probability $1 - \delta$
    \begin{equation*}
        \sum^K_{k=1} \innerprod{\lambda(\pi_k) - \phi_\cost(X_k, A_k)}{w_k - w_{\mathrm{true}}} \leq \RMAX \sqrt{8  K \log \brr{\frac{1}{\delta}}}.
    \end{equation*}
    Then, plugging in this bound in the regret decomposition we obtain
    \begin{align*}
        (1-\gamma) \regret^w(w_{\mathrm{true}}) &\leq \frac{2 \WMAX^2}{\eta_r} + 2 \eta_r B^2 K + \RMAX \sqrt{8K \log 1 / \delta} \\
        &\phantom{=}+ \sum^K_{k=1} \innerprod{\widehat{\lambda(\expert)} - \lambda(\expert)}{ w_k-w_{\mathrm{true}}}\,.
    \end{align*}
    Then, we treat the last term using Cauchy-Schwartz's inequality
    \begin{align*}
        \sum^K_{k=1} \innerprod{\widehat{\lambda(\expert)} - \lambda(\expert)}{ w_k-w_{\mathrm{true}}} &\leq \sum^K_{k=1}\norm{w_{\mathrm{true}}-w^k}_{2}\norm{\widehat{\lambda(\expert)} - \lambda(\expert)}_2 \\
        &\leq 2 \WMAX K \norm{\widehat{\lambda(\expert)} - \lambda(\expert)}_2\,.
    \end{align*}
    It remains to find a high probability (dimension-free) upper bound on $\norm{\widehat{\lambda(\expert)} - \lambda(\expert)}_{2}$. First, notice that $\norm{\widehat{\lambda(\expert)} - \lambda(\expert)}_{2} = \norm{(\tau_E)^{-1} \brr{\sum^{\tau_E}_{i=1} \phi_\cost(X^i_E,A^i_E) - \lambda(\expert) }}_2$. Then, we use the notation $u_{x,a} = \phi_\cost(x,a) - \lambda(\expert) $ for all state action pairs $x,a$ and using that for all $x, a \in \cX \times \cA$, $\norm{\phi_\cost(x, a)}_2 \leq B$, we have
    \begin{equation*}
        \sum^{\tau_E}_{i=1} \mathbb{E}\bs{\norm{u_{X^i_E,A^i_E}}_2^2} \leq \sum^{\tau_E}_{i=1} \mathbb{E}\bs{\norm{\phi_\cost(X^i_E, A^i_E) - \lambda(\expert)}_2^2} \leq 4\tau_E B^2\,.
    \end{equation*}
    Moreover, for any $x, a \in \cX \times \cA$, $\norm{u_{x,a}} \leq 2 B$ and $\bbE \bs{u_{X^i_E, A^i_E}} = 0$ because of the distribution of the dataset $\mathcal{D}_\expert $. Thus, by applying \cite[Proposition 2]{hsu2012tail}, it holds that for all $t > 0$
    \begin{equation*}
        \mathbb{P}\bs{\norm{\sum^{\tau_E}_{i=1} u_{X^i_E, A^i_E}} > \sqrt{4 \tau_E} B + \sqrt{32 \tau_E t} B + (8/3)2Bt} \leq e^{-t}
    \end{equation*}
    Therefore, choosing $t = \log \frac{1}{\delta}$, we obtain that with probability $1 - \delta$
    \begin{align*}
        \norm{\sum^{\tau_E}_{i=1} \phi_\cost(X^i_E, A^i_E) - \lambda(\expert)} &\leq \sqrt{4 \tau_E} B + \sqrt{32 \tau_E \log\brr{\frac{1}{\delta}}} B + \frac{16 B}{3}  \log\brr{\frac{1}{\delta}} \\
        &\leq 6 B \sqrt{\tau_E \log\brr{\frac{1}{\delta}}} + \frac{16 B}{3} \log \brr{\frac{1}{\delta}}\,.
    \end{align*}
    Then, dividing by $\tau_E$ we obtain that
    \begin{align*}
        \norm{\widehat{\lambda(\expert)} - \lambda(\expert)}_{2} \leq 6 B \sqrt{\frac{\log\brr{\frac{1}{\delta}}}{\tau_E}} + \frac{16 B}{3 \tau_E} \log\brr{\frac{1}{\delta}}\,.
    \end{align*}
    Then, for $\tau_E \geq \frac{64}{18^2} \log \frac{1}{\delta}$, we have that
    \begin{equation*}
        6 \sqrt{\frac{\log\brr{\frac{1}{\delta}}}{\tau_E}} \geq \frac{8}{3 \tau_E} \log\brr{\frac{1}{\delta}}\,,
    \end{equation*}
    and hence that with probability $1 - \delta$,
    \begin{align*}
        \norm{\widehat{\lambda(\expert)} - \lambda(\expert)}_{2} \leq 12 B \sqrt{\frac{\log\brr{\frac{1}{\delta}}}{\tau_E}}\,.
    \end{align*}
    Thus, by a union bound and choosing $\eta_r = \nicefrac{\WMAX}{B \sqrt{K}}$, we have that with probability $1 - 2 \delta$,
    \begin{align*}
        (1-\gamma) \regret^w(w_{\mathrm{true}}) &= 4 B \WMAX \sqrt{K} +  \RMAX\sqrt{8 K \log \delta^{-1}} + 24 \WMAX B K \sqrt{ \frac{\log\brr{\frac{1}{\delta}}}{\tau_E}} \\
        &\leq 10 (B \WMAX + \RMAX) \sqrt{K \log \delta^{-1}} + 24 \WMAX K B \sqrt{\frac{\log \brr{\frac{1}{\delta}}}{\tau_E}}\,.
    \end{align*}
\end{proof}

%% file: sections/lower_bounds_appendix.tex
\section{Lower bounds for imitation learning}
\label{app:lower}
In this section, we prove lower bounds for both $K$ and $\tau_E$ for all algorithms following Protocol~\ref{prot:interaction} given hereafter.
\begin{protocol}[!h]
    \caption{Imitation learning from features alone in Linear MDPs. \label{prot:interaction}}
    \centering
    \begin{algorithmic}[1]
        \STATE The learner adopts a learning algorithm $\mathrm{Alg}$ that receives as input \\
        (1) a features dataset $\cD_\expert = \scbr{\phi_\cost \spr{X^i_E, A^i_E}}_{i=1}^{\tau_E}$ where for any $i \in \sbr{\tau_E}$, $X^i_E, A^i_E \sim \mu \spr{\expert}$, \\
        (2) read access to $\phi_P \spr{x, a}$ for all $x, a \in \cX \times \cA$, \\
        (3) trajectory access to $\cM \setminus \true$, and \\
        (4) the reward class $\cR$ such that $\true \in \cR$.
        
        \STATE $\mathrm{Alg}$ samples $K$ trajectories from $\cM \setminus \true$ and outputs $\pi^{\mathrm{out}}$ s.t. $\bbE \sbr{\inp{\initial, V_{\true}^\expert - V_{\true}^{\pi^{\mathrm{out}}}}} \leq \varepsilon$.
    \end{algorithmic}
\end{protocol}

We prove an $\Omega \spr{\varepsilon^{-2}}$ lower bound for both cases, demonstrating that Algorithm~\ref{alg:fra} is rate optimal. First, we state the lower bound $K$ that holds even with perfect knowledge of the expert feature expectation vector $\lambda \spr{\expert}$, a strictly easier setting compared the one under which \Cref{thm:FraUpper} is proven.
\begin{restatable}{theorem}{LowerK} \textbf{(Lower Bound on $K$)} \label{thm:LowerK}
    For any algorithm $\mathrm{Alg}$, there exists an MDP $\cM$ and an expert policy $\expert$ such that $\mathrm{Alg}$, taking as input $ \phim_r\transpose \mu_\cM \spr{\expert}$, requires $K = \Omega \spr{\frac{d}{\spr{1 - \gamma}^2 \varepsilon^2}}$ to guarantee $\bbE_{\mathrm{Alg}} \sbr{\inp{\initial, V_\cM^{\expert} - V_\cM^{\pi^{\mathrm{out}}}}} = \cO \spr{\varepsilon}$.
\end{restatable}

\noindent Next, we establish a lower bound on the required number of expert demonstration $\tau_E$. The result holds even with perfect knowledge of the transition dynamics (\ie for $K = \infty$).
\begin{restatable}{theorem}{LowerTauE} \textbf{(Lower Bound on $\tau_E$)} \label{thm:LowerTauE} \label{thm:lower-bound-2states_expert}
    Let $\gamma \geq \frac12$. For any algorithm $\mathrm{Alg}$, there exists an MDP $\cM$ and an expert policy $\expert$ such that $\mathrm{Alg}$ taking as input the transitions dynamics and an expert dataset of size $\tau_E$ requires $\tau_E = \Omega \spr{\frac{\WMAX^2}{\spr{1 - \gamma}^2 \varepsilon^2}}$ to guarantee $\bbE_{\mathrm{Alg}} \sbr{\inp{\initial, V_\cM^{\expert} - V_\cM^{\pi^{\mathrm{out}}}}} = \cO \spr{\varepsilon}$.
\end{restatable}
\noindent The proofs are provided in the following sections.

\subsection{Proof of Theorem~\ref{thm:LowerK} (lower bound on the number of interactions)}

We start with the proof of the lower bound on $K$. We consider a class of possibly randomized algorithms that output a policy $\pi^{\mathrm{out}}$ given a dataset of expert features $\cD_\expert$ and $K$ trajectories collected by the learner in the MDP $\cM$.

\noindent \emph{Proof Idea}. To construct a lower bound, we consider the case of imitation learning from states alone (\ie $\phim_r \spr{x, a} = \bfe_x$), and $\lambda \spr{\expert}$ represents the state expert occupancy measure. We consider the case of a two-state MDP, where $\cX = \scbr{x_0, x_1}$, and the learner knows the \emph{good} state $x_0$ that maximizes the expert's occupancy measure due to having access to $\lambda \spr{\expert}$. The leaner's objective is to maximize the time spent in this good state. All actions in the \emph{bad} state $x_1$ share the same transition kernel. Therefore, the agent's decisions in the state $x_0$ is the only factor that influences the outcome. An action labeled as $a^\star$ is available in the state $x_0$. The transition kernel $P\spr{x_0 \given x_0, a}$ is identical for all actions $a \neq a^\star$, while for $a^\star$, it is defined as $P \spr{x_0 \given x_0, a} + \epsilon$. We then consider a family of $\abs{\cA}$ MDPs, where each MDP assigns the role of $a^\star$ to a different action. We will formally demonstrate that for any algorithm in $\mathrm{Alg}$, there exists at least one MDP within this family where achieving $\bbE_{\mathrm{Alg}} \sbr{\inp{\initial, V_\cM^{\expert} - V_\cM^{\pi^{\mathrm{out}}}}} = \cO \spr{\varepsilon}$ requires $K = \Omega \spr{\frac{\abs{\cA}}{\spr{1 - \gamma}^2 \varepsilon^2}}$. Finally, the bound for an arbitrary dimension $d$ is obtained noticing that this MDP can be written as a linear MDP with features dimension $d = 2 + 2 \abs{\cA}$.

\bigskip
\begin{proof}
    For any policy $\pi$, we denote $\lambda_\cM \spr{\pi} = \phim_r^\trans \mu_\cM \spr{\pi}$ the expected feature vector of the policy $\pi$ in the MDP $\cM$. We consider a deterministic algorithm $\mathrm{Alg}$ that maps $\lambda_\cM \spr{\expert}$ and $K$ environment trajectories to a policy. The extension to randomized algorithms can be done by an application of Fubini's theorem (see \cite{bubeck2012regret}). The hard instance we consider for the lower bound is an MDP $\cM$ with two states, $x_0$ and $x_1$, and $\abs{\cA}$ actions per state. For any action $a$, the reward function is given by $\true \spr{x_0, a} = 1$, and $\true \spr{x_1, a} = 0$. We will refer to state $x_0$ as the ``good'' state and to state $x_1$ as the ``bad'' state.  In state $x_1$, the transition kernel induced by any action $a$ is the same, \ie $P \spr{x_1 \given x_1, a} = 1 - \delta_1$, and $P \spr{x_0 \given x_1, a} = \delta_1$ for some $\delta_1 \in \spr{0, 1}$. Let $\delta_0 \in \spr{0, 1}$ and $\epsilon \in \spr{0, \delta_0}$. In state $x_0$, there is an action $a^\star$ with a slightly different transition kernel
    \begin{align*}
        P \spr{x_1 \given x_0, a^\star} = \delta_0 - \epsilon, \quad P \spr{x_0 \given x_0, a^\star} = 1 - \delta_0 + \epsilon\,,
    \end{align*}
    whereas for any action $a \neq a^\star$, we set
    \begin{align*}
        P \spr{x_1 \given x_0, a} = \delta_0, \quad P \spr{x_0 \given x_0, a} = 1 - \delta_0\,.
    \end{align*}
    We set the unknown expert policy $\expert$ such that it always select action $a^\star$ in both states, \ie $\expert \spr{a^\star \given x_0} = \expert \spr{a^\star \given x_1} = 1$. Setting $\nu_0 \spr{x_0} = 1$, we can write the flow constraints and get
    \begin{align*}
        &\nu \spr{\expert, x_0} = 1 - \gamma + \gamma \spr{1 - \delta_0 + \epsilon} \nu \spr{\expert, x_0} + \gamma \delta_1 \nu \spr{\expert, x_1}\,, \\
        &\nu \spr{\expert, x_1} = \gamma \spr{1 - \delta_1} \nu \spr{\expert, x_1} + \gamma \spr{\delta_0 - \epsilon} \nu \spr{\expert, x_0}\,.
    \end{align*}
    The second equation gives $\nu \spr{\expert, x_1} = \frac{\gamma \spr{\delta_0 - \epsilon}}{1 - \gamma \spr{1 - \delta_1}} \nu \spr{\expert, x_0}$, which we can plug back into the first equation to obtain
    \begin{align*}
        \nu \spr{\expert, x_0} = 1 - \gamma + \spr{\gamma \spr{1 - \delta_0 + \epsilon} + \frac{\gamma^2 \delta_1 \spr{\delta_0 - \epsilon}}{1 - \gamma \spr{1 - \delta_1}}} \nu \spr{\expert, x_0} \,,
    \end{align*}
    which we can rearrange to get
    \begin{align*}
        \nu \spr{\expert, x_0} = \frac{1 - \gamma + \gamma \delta_1}{1 - \gamma + \gamma \delta_1 + \gamma \delta_0 - \gamma \epsilon}\,.
    \end{align*}
    Using the normalization constraint $\nu \spr{\expert, x_0} + \nu \spr{\expert, x_1} = 1$, we also get
    \begin{equation*}
        \nu \spr{\expert, x_1} = \frac{\gamma \delta_0 - \gamma \epsilon}{1 - \gamma + \gamma \delta_1 + \gamma \delta_0 - \gamma \epsilon}\,.
    \end{equation*}
    %
    Furthermore, let $\pi_{\mathrm{bad}}$ be a ``bad'' policy that always plays an action $a \neq a^\star$. The same calculation with $\epsilon = 0$ shows that the state occupancy measure for the policy $\pi_{\mathrm{bad}}$ is given by
    \begin{align*}
        \nu \spr{\pi_{\mathrm{bad}}, x_0} &= \frac{1 - \gamma + \gamma \delta_1}{1 - \gamma + \gamma \delta_1 + \gamma \delta_0}\,, \\
        \nu \spr{\pi_{\mathrm{bad}}, x_1} &= \frac{\gamma \delta_0}{1 - \gamma + \gamma \delta_1 + \gamma \delta_0}\,.
    \end{align*}
    Let $\tilde{\pi}$ be any policy. Noting that for any $x$, $V^\expert \spr{x} = Q^\expert \spr{x, a^\star}$, we can use the performance difference lemma and get
    \begin{align*}
        \inp{\mu \spr{\expert} - \mu \spr{\tilde{\pi}}, \true} &= \bbE_{\spr{x, a} \sim \mu \spr{\tilde{\pi}}} \sbr{V^\expert \spr{x} - Q^\expert \spr{x, a}} \\
        &= \bbE_{\spr{x, a} \sim \mu \spr{\tilde{\pi}}} \sbr{Q^\expert \spr{x, a^\star} - Q^\expert \spr{x, a}}\,.
    \end{align*}
    All actions share the same transition kernel in $x_1$ thus for any action $a$, $Q^\expert \spr{x_1, a^\star} = Q^\expert \spr{x_1, a}$ and we have
    \begin{align*}
        \inp{\mu \spr{\expert} - \mu \spr{\tilde{\pi}}, \true} &= \nu \spr{\tilde{\pi}, x_0} \sum_{a \in \cA \setminus \scbr{a^\star}} \tilde{\pi} \spr{a \given x_0} \spr{Q^\expert \spr{x_0, a^\star}  - Q^\expert \spr{x_0, a}}\,.
    \end{align*}
    Next, we need to compute the difference of Q-values. Using the Bellman equations for $\expert$ in state $x_0$, we have
    \begin{align}
        \forall a \neq a^\star, &Q^\expert \spr{x_0, a} = 1 + \gamma \delta_0 Q^\expert \spr{x_1, a^\star} + \gamma \spr{1 - \delta_0} Q^\expert \spr{x_0, a^\star} \label{eq:bellman-eq-x0-a} \\
        &Q^\expert \spr{x_0, a^\star} = 1 + \gamma \spr{\delta_0 - \epsilon} Q^\expert \spr{x_1, a^\star} + \gamma \spr{1 - \delta_0 + \epsilon} Q^\expert \spr{x_0, a^\star}\,. \label{eq:bellman-eq-x0-astar}
    \end{align}
    Solving the second equation for $Q^\expert \spr{x_0, a^\star}$ gives
    \begin{align}
        Q^\expert \spr{x_0, a^\star} = \frac{1}{1 - \gamma \spr{1 - \delta_0 + \epsilon}} \spr{1 + \gamma \spr{\delta_0 - \epsilon} Q^\expert \spr{x_1, a^\star}} \label{eq:q-x0-astar-inter}\,.
    \end{align}
    By the Bellman equation in state $x_1$ and action $a^\star$, we further have
    \begin{align*}
        Q^\expert \spr{x_1, a^\star} = 0 + \gamma \delta_1 Q^\expert \spr{x_0, a^\star} + \gamma \spr{1 - \delta_1} Q^\expert \spr{x_1, a^\star}\,,
    \end{align*}
    which implies that
    \begin{align}
        Q^\expert \spr{x_1, a^\star} = \frac{\gamma \delta_1}{1 - \gamma(1 - \delta_1)} Q^\expert \spr{x_0, a^\star}\,. \label{eq:q-x1-astar-inter}
    \end{align}
    Replacing \eqref{eq:q-x1-astar-inter} into \eqref{eq:q-x0-astar-inter}, we get
    \begin{align*}
        Q^\expert \spr{x_0, a^\star} &= \frac{1}{1 - \gamma \spr{1 - \delta_0 + \epsilon}} + \frac{\gamma^2 \delta_1 \spr{\delta_0 - \epsilon}}{\spr{1 - \gamma \spr{1 - \delta_0 + \epsilon}} \spr{1 - \gamma \spr{1 - \delta_1}}} Q^\expert \spr{x_0, a^\star}\,.
    \end{align*}
    Rearranging the terms gives
    \begin{align}
        Q^\expert \spr{x_0, a^\star} &= \spr{1 - \frac{\gamma^2 \delta_1 \spr{\delta_0 - \epsilon}}{\spr{1 - \gamma \spr{1 - \delta_0 + \epsilon}} \spr{1 - \gamma \spr{1 - \delta_1}}}}^{-1} \frac{1}{1 - \gamma \spr{1 - \delta_0 + \epsilon}} \nonumber \\
        &= \frac{1 - \gamma \spr{1 - \delta_1}}{\spr{1 - \gamma \spr{1 - \delta_0 + \epsilon}} \spr{1 - \gamma \spr{1 - \delta_1}} - \gamma^2 \delta_1 \spr{\delta_0 - \epsilon}}\,. \label{eq:q-x0-astar}
    \end{align}
    Plugging Equation~\eqref{eq:q-x0-astar} into Equation~\eqref{eq:q-x1-astar-inter}, we can deduce the value of the expert at $\spr{x_1, a^\star}$
    \begin{equation*}
        Q^\expert \spr{x_1, a^\star} = \frac{\gamma \delta_1}{\spr{1 - \gamma \spr{1 - \delta_0 + \epsilon}} \spr{1 - \gamma \spr{1 - \delta_1}} - \gamma^2 \delta_1 \spr{\delta_0 - \epsilon}}\,.
    \end{equation*}
    Looking at the difference $Q^\expert \spr{x_0, a^\star} - Q^\expert \spr{x_0, a}$, we can take the difference of Equations~\eqref{eq:bellman-eq-x0-astar} and \eqref{eq:bellman-eq-x0-a} to get
    \begin{align*}
        Q^\expert \spr{x_0, a^\star} - Q^\expert \spr{x_0, a} &= \gamma \epsilon \spr{Q^\expert \spr{x_0, a^\star} - Q^\expert \spr{x_1, a^\star}} \\
        &= \frac{\gamma \epsilon \spr{1 - \gamma}}{\underbrace{\spr{1 - \gamma \spr{1 - \delta_0 + \epsilon}} \spr{1 - \gamma \spr{1 - \delta_1}} - \gamma^2 \delta_1 \spr{\delta_0 - \epsilon}}_{\spr{\diamondsuit}}}\,.
    \end{align*}
    Next, we upper bound the denominator as follows
    \begin{align*}
        \spr{\diamondsuit} &= 1 - \gamma \spr{1 - \delta_0 + \epsilon} - \gamma \spr{1 - \delta_1} \\
        &\phantom{=}+ \gamma^2 \spr{1 - \delta_0 + \epsilon - \delta_1 + \delta_0 \delta_1 - \epsilon \delta_1 - \delta_0 \delta_1 + \epsilon \delta_1} \\
        &= 1 - \gamma \spr{1 - \delta_0 + \epsilon} - \gamma \spr{1 - \delta_1} + \gamma^2 \spr{1 - \delta_0 - \delta_1 + \epsilon} \\
        &= 1 - \gamma + \gamma \delta_0 \spr{1 - \gamma} + \gamma \delta_1 \spr{1 - \gamma} - \gamma \spr{1 - \gamma} - \gamma \epsilon \spr{1 - \gamma} \\
        &= \spr{1 - \gamma}^2 + \gamma \delta_0 \spr{1 - \gamma} + \gamma \delta_1 \spr{1 - \gamma} - \gamma \epsilon \spr{1 - \gamma} \\
        &\leq \spr{1- \gamma}^2 + \gamma \delta_0 \spr{1 - \gamma} + \gamma \delta_1 \spr{1 - \gamma}\,,
    \end{align*}
    where the inequality follows from $\gamma \epsilon \spr{1 - \gamma} > 0$. Setting $\delta_1 = \delta_0 = \frac{1 - \gamma}{\gamma}$, we obtain
    \begin{equation*}
        \spr{\diamondsuit} \leq 3 \spr{1 - \gamma}^2\,,
    \end{equation*}
    and it holds that
    \begin{align*}
        Q^\expert \spr{x_0, a^\star} - Q^\expert \spr{x_0, a} \geq \frac{\gamma \epsilon}{3 \spr{1 - \gamma}}\,.
    \end{align*}
    Moreover, the choice of $\delta_0$ and $\delta_1$ implies that $\nu \spr{\pi_{\mathrm{bad}}, x_0} = \frac23$. By definition of the transitions, note that always playing $a \neq a^\star$ like $\pi_{\mathrm{bad}}$ does minimizes the probability of being in state $x_0$. Thus, for any policy $\tilde{\pi}$, $\nu \spr{\tilde{\pi}, x_0} \geq \nu \spr{\pi_{\mathrm{bad}}, x_0}$, and we have
    \begin{align*}
        \inp{\mu \spr{\expert} - \mu \spr{\tilde{\pi}}, \true} &\geq \nu \spr{\tilde{\pi}, x_0} \sum_{a \in \cA \setminus \scbr{a^\star}} \tilde{\pi} \spr{a \given x_0} \frac{\gamma \epsilon}{3 \spr{1 - \gamma}} \\
        &\geq \nu \spr{\pi_{\mathrm{bad}}, x_0} \spr{1 - \tilde{\pi} \spr{a^\star \given x_0}} \frac{\gamma \epsilon}{3 \spr{1 - \gamma}} \\
        &= 2 \spr{1 - \tilde{\pi} \spr{a^\star \given x_0}} \frac{\gamma \epsilon}{9(1 - \gamma)} \\
        &\geq \frac{\spr{1 - \tilde{\pi} \spr{a^\star \given x_0}} \epsilon}{9 \spr{1 - \gamma}}\,.
    \end{align*}
    where the last inequality follows from $\gamma \geq 1/2$. We now consider the policy $\tilde{\pi} = \bar\pi$ produced by a learning algorithm $\mathrm{Alg}$ interacting with the MDP described above (with $\epsilon > 0$). We also consider  $\underline{\pi}$ the output of the same learning algorithm $\mathrm{Alg}$ when interacting with the MDP $\underline{\cM}$, a copy of $\cM$ with $\epsilon = 0$ (note that in $\underline{\cM}$, all actions are identical in \emph{both} states $x_0$ and $x_1$, so there is nothing to learn). In $\cM$, all actions are identical in state $x_1$, thus we can assume both policies are the same in state $x_1$, \ie $\bar\pi \spr{\cdot \given x_1} = \underline{\pi} \spr{\cdot \given x_1} = \bfe_{a^\star}$, and focus exclusively on learning in state $x_0$. By Pinkser's inequality, we have that
    \begin{align*}
        \bar\pi \spr{a^\star \given x_0} - \underline{\pi} \spr{a^\star \given x_0} \leq \sqrt{2 \KL \spr{\underline{\pi} \spr{\cdot \given x_0} \| \bar\pi \spr{\cdot \given x_0}}}\,,
    \end{align*}
    and the previous inequality becomes
    \begin{align*}
        \inp{\mu \spr{\expert} -  \mu \spr{\bar\pi}, \true} \geq \frac{\epsilon}{9 \spr{1 - \gamma}} \spr{1 - \underline{\pi} \spr{a^\star \given x_0} - \sqrt{2 \KL \spr{\underline{\pi} \spr{\cdot \given x_0} \| \bar\pi \spr{\cdot \given x_0}}}}\,.
    \end{align*}
    Denote $A = \abs{\cA}$ and let $\cH = \scbr{\cM_i}_{i=1}^A$ be a collection of MDPs instances where for any $i = 1, \dots, A$, the MDP $\cM_i$ is a copy of $\cM$ where the $i$th action is equal to $a^\star$, \ie $a_i = a^\star$. We denote $P_i$ the corresponding transitions. For any $i \in [\![1, A]\!]$, we denote $\bar\pi^i$ the policy output by the learning algorithm $\mathrm{Alg}$ after interacting with the instance $\cM_i$, and $\experti$ be the expert policy for the instance $\cM_i$, \ie the policy that always plays $a_i$. We denote $\mu_i \spr{\pi}$ the occupancy measure of any policy $\pi$ in the MDP $\cM_i$. Then, notice that the previous derivations apply for any MDP in $\cH$. Thus, summing over $i \in [\![1, A]\!]$ and noting that $\underline{\pi} \spr{\cdot \given x_0}$ is a probability distribution over $\cA$, we get
    \begin{align}
        \sum_{i = 1}^A \inp{\mu_i \spr{\experti} - \mu_i \spr{\bar\pi^i}, \true} \geq \frac{\epsilon}{9 \spr{1 - \gamma}} \spr{A - 1 - \sum_{i = 1}^A \sqrt{2 \KL \spr{\underline{\pi} \spr{\cdot \given x_0} \| \bar\pi^i \spr{\cdot \given x_0}}}}\,. \label{eq:sum_lower_bound}
    \end{align}
    For any $i \in \sbr{A}$ and $T \in \bbN^\star$, denote $\bbP_i^T$ the probability distribution over sets $\cD_i^T = \scbr{x_0, A_t^i, X_t^i}_{t \in \sbr{T}}$ of $T$ transitions starting from $x_0$ induced by the interaction between the algorithm $\mathrm{Alg}$ and the MDP $\cM_i$. Likewise, we denote $\underline{\bbP}^T$ the probability distribution corresponding to $\underline{\cM}$. Then, by the data processing inequality for the KL divergence, for any $i \in \sbr{A}$, it holds that
    \begin{equation*}
        \KL \spr{\underline{\pi} \spr{\cdot \given x_0} \| \bar\pi^i \spr{\cdot \given x_0}} \leq \KL \spr{\underline{\bbP}^T \| \bbP_i^T}\,.
    \end{equation*}
    Denoting $\underline{\bbE}$ the expectation with respect to $\underline{\bbP}^T$, we can use the Markov property of the environment and continue as follows
    \begin{align*}
        \KL \spr{\underline{\bbP}^T \| \bbP_i^T} &= \underline{\bbE} \sbr{\log \spr{\frac{\prod_{t=1}^T \underline{P} \spr{\underline{X}_t \given x_0, \underline{A}_t} \underline{\bbP}^T \spr{\underline{A}_t \given \underline{X}_1, \underline{A}_1, \dots, \underline{X}_{t-1}}}{\prod_{t=1}^T P_i \spr{\underline{X}_t \given x_0, \underline{A}_t} \bbP_i^T \spr{\underline{A}_t \given \underline{X}_1, \underline{A}_1, \dots, \underline{X}_{t-1}}}}} \\
        &= \underline{\bbE} \sbr{\log \spr{\frac{\prod_{t=1}^T \underline{P} \spr{\underline{X}_t \given x_0, \underline{A}_t}}{\prod_{t=1}^T P_i \spr{\underline{X}_t \given x_0, \underline{A}_t}}}} \\
        &= \underline{\bbE} \sbr{\sumtT \log \spr{\frac{\underline{P} \spr{\underline{X}_t \given x_0, \underline{A}_t}}{P_i \spr{\underline{X}_t \given x_0, \underline{A}_t}}}}\,,
    \end{align*}
    where the probabilities on the actions are equal due to running the same algorithm $\mathrm{Alg}$ with the same history up to time $t-1$. Next, we have
    \begin{align*}
        \KL \spr{\underline{\bbP}^T \| \bbP_i^T} &= \sumtT \sum_{\spr{x, a} \in \cX \times \cA} \underline{\bbP}^T \sbr{\spr{\underline{X}_t, \underline{A}_t} = \spr{x, a}} \log \spr{\frac{\underline{P} \spr{x \given x_0, a}}{P_i \spr{x \given x_0, a}}} \\
        &= \sumtT \sum_{x \in \cX} \underline{\bbP}^T \sbr{\spr{\underline{X}_t, \underline{A}_t} = \spr{x, a_i}} \log \spr{\frac{\underline{P} \spr{x \given x_0, a_i}}{P_i \spr{x \given x_0, a_i}}}\,,
    \end{align*}
    where we used that the transitions $\underline{P}$ and $P_i$ are the same for any action $a \neq a_i$. By definition of the transitions, we further have
    \begin{align*}
        \KL \spr{\underline{\bbP}^T \| \bbP_i^T} &= \sumtT \underline{\bbP}^T \sbr{\spr{\underline{X}_t, \underline{A}_t} = \spr{x_0, a_i}}  \log \spr{\frac{1 - \delta_0}{1 - \delta_0 + \epsilon}} \\
        &\phantom{=}+ \sumtT \underline{\bbP}^T \sbr{\spr{\underline{X}_t, \underline{A}_t} = \spr{x_1, a_i}} \log \spr{\frac{\delta_0}{\delta_0 - \epsilon}}\,.
    \end{align*}
    Next, by definition of $\underline{\bbP}^T$, we have
    \begin{align*}
        \KL \spr{\underline{\bbP}^T \| \bbP_i^T} &= \sumtT \underline{\bbP}^T \sbr{\underline{A}_t = a_i} \underline{P} \spr{x_0 \given x_0, a_i} \log \spr{\frac{1 - \delta_0}{1 - \delta_0 + \epsilon}} \\
        &\phantom{=}+ \sumtT \underline{\bbP}^T \sbr{\underline{A}_t = a_i} \underline{P} \spr{x_1 \given x_0, a_i} \log \spr{\frac{\delta_0}{\delta_0 - \epsilon}} \\
        &= \underline{\bbE} \sbr{\sumtT \mathds{1} \scbr{\underline{A}_t = a_i}} \spr{\spr{1 - \delta_0} \log \spr{\frac{1 - \delta_0}{1 - \delta_0 + \epsilon}} + \delta_0 \log \spr{\frac{\delta_0}{\delta_0 - \epsilon}}}\,.
    \end{align*}
    By \citealp[Lemma 20]{AJO08}, we can bound the KL divergence as follows
    \begin{align*}
        \KL \spr{\underline{\bbP}^T \| \bbP_i^T} &\leq \frac{\epsilon^2}{\delta_0 \log(2)} \underline{\bbE} \sbr{\sumtT \mathds{1} \scbr{\underline{A}_t = a_i}} \\
        &\leq \frac{\epsilon^2}{\spr{1 - \gamma} \log \spr{2}} \underline{\bbE} \sbr{\sumtT \mathds{1} \scbr{\underline{A}_t = a_i}}\,,
    \end{align*}
    where the last inequality is due to the choice of $\delta_0 = \frac{1 - \gamma}{\gamma}$ and $\gamma < 1$. Plugging this into Equation~\eqref{eq:sum_lower_bound} and dividing by $A$, we have
    \begin{align*}
        \frac1A \sum_{i = 1}^A \inp{\mu_i \spr{\experti} - \mu_i \spr{\bar\pi^i}, \true} &\geq \frac{\epsilon}{9 \spr{1 - \gamma}} \spr{1 - \frac1A - \frac{\epsilon}{A} \sum_{i=1}^A \sqrt{\frac{\underline{\bbE} \sbr{\sumtT \mathds{1} \scbr{\underline{A}_t = a_i}}}{\spr{1 - \gamma} \log \spr{2}}}}\,.
    \end{align*}
    By Jensen's inequality, we further get
    \begin{align*}
        \frac1A \sum_{i = 1}^A \inp{\mu_i \spr{\experti} - \mu_i \spr{\bar\pi^i}, \true} &\geq \frac{\epsilon}{9 \spr{1 - \gamma}} \spr{1 - \frac1A - \epsilon \sqrt{\frac{\underline{\bbE} \sbr{\sum_{i=1}^A \sumtT \mathds{1} \scbr{\underline{A}_t = a_i}}}{A \spr{1 - \gamma} \log \spr{2}}}} \\
        &\geq \frac{1}{9 \spr{1 - \gamma}} \spr{\frac{\epsilon}{2} - \epsilon^2 \sqrt{\frac{T}{A \spr{1 - \gamma} \log \spr{2}}}}\,,
    \end{align*}
    where the second inequality follows from $\sum_{i=1}^A \mathds{1} \scbr{\underline{A}_t = a_i} = 1$ almost surely for any $t$ and $1 - \frac1A \geq \frac12$. Note that the value of $\epsilon$ maximizing the lower bound is given by $\epsilon^\star = \frac14 \sqrt{\frac{A \spr{1 - \gamma} \log \spr{2}}{T}}$. To satisfy the constraint $\epsilon^\star \in \spr{0, \delta_0}$ with $\delta_0 = \frac{1 - \gamma}{\gamma}$, assume we have $T \geq \frac{\gamma^2 A \log \spr{2}}{16 \spr{1 - \gamma}}$. We plug the value of $\epsilon^\star$ in the previous inequality to get
    \begin{align*}
        \frac1A \sum_{i = 1}^A \inp{\mu_i \spr{\experti} - \mu_i \spr{\bar\pi^i}, \true} &\geq \frac{1}{16 \cdot 9 \spr{1 - \gamma}} \sqrt{\frac{A \spr{1 - \gamma} \log \spr{2}}{T}} \\
        &= \frac{1}{144} \sqrt{\frac{A \log \spr{2}}{\spr{1 - \gamma} T}}\,,
    \end{align*}
    The average can be upper bounded by the maximum, thus
    \begin{align*}
        \max_{i = 1, \dots, A} \inp{\nu_0, V_{\cM_i}^{\experti} - V_{\cM_i}^{\bar\pi^i}} &= \frac{1}{1 - \gamma} \max_{i = 1, \dots, A} \inp{\mu_i \spr{\experti} - \mu_i \spr{\bar\pi^i}, \true} \\
        &\geq \frac{1}{144} \sqrt{\frac{A \log \spr{2}}{\spr{1 - \gamma}^3 T}}\,.
    \end{align*}
    What remains is to set the number of samples $T$ to make the lower bound small enough to make $\max_{i = 1, \dots, A} \inp{\nu_0, V_{\cM_i}^{\experti} - V_{\cM_i}^{\bar\pi^i}} = \cO \spr{\varepsilon}$ possible, \ie we need to have $T = \Omega \spr{\frac{A}{\spr{1 - \gamma}^3 \varepsilon^2}}$ samples. Therefore, we need $T = \Omega \spr{\frac{A}{\spr{1 - \gamma}^3 \varepsilon^2}}$ samples to learn a $\cO \spr{\varepsilon}$-suboptimal policy in the MDP that achieves the maximum. In order to derive a lower bound on the episodes number $K$ we can divide the sample complexity lower bound for $T$ by the the expected number of transitions per episode which is $\spr{1-\gamma}^{-1}$. This gives $K = \Omega \spr{\frac{A}{\spr{1 - \gamma}^2 \varepsilon^2}}$. We can conclude by noting that our construction used in the lower bound is a linear MDP with dimensionality $d = 2 + 2 \abs{\cA}$, thus we have $K = \Omega \spr{\frac{d}{\spr{1 - \gamma}^2 \varepsilon^2}}$.
\end{proof}

\subsection{Proof of \Cref{thm:LowerTauE} (lower bound on the number of expert transitions)}
\label{app:lower_tau_E}

\emph{Proof Idea:} The construction of the lower bound consists in relating the problem to that of distinguishing two Bernoullis distributions with close means. For that, we consider two MDPs $\cM_0$ and $\cM_1$ that only differ in their reward function. They have two states $\cX = \scbr{x_0, x_1}$ and $\abs{\cA}$ actions available at each state. The initial distribution $\initial$ is chosen to be the uniform distribution over $\cX$. In state $x_1$, any action $a$ induces the same transition kernel: $P \spr{x_0 \given x_1, a} = \delta$. In state $x_0$, any action $a$ except some action $a^\star$ is such that $P \spr{x_1 | x_0, a} = \delta$. However, the special action $a^\star$ allows to stay in the state $x_0$ with a slightly higher probability, \ie $P \spr{x_1 \given x_0, a^\star} = \delta - \epsilon$. Then, the reward function in $\cM_0$ is defined as $\true^0 \spr{x_0, \cdot} = \WMAX$, and $\true^0 \spr{x_0, \cdot} = 0$, while in $\cM_1$, it is defined as $\true^1 \spr{x_0, \cdot} = 0$, $\true^1 \spr{x_1, \cdot} = \WMAX$. Finally, we define an expert $\pi_E^0$ for $\cM_0$ as the policy that always play the action $a^\star$, and an expert $\pi_E^1$ for $\cM_1$ that always play some action $a \neq a^\star$. We then show that the expert occupancy measures satisfy $\nu \spr{\pi_E^0, x_0} = 1/2 + \Delta$, for some small $\Delta > 0$, while $\nu \spr{\pi_E^1, x_0} = \frac12$. The remaining step is to reduce this problem to a lower bound on the regret of a two-arm Bernoulli bandits instance with means $\spr{1/2, 1/2}$ and $\spr{1/2 + \Delta, 1/2 - \Delta}$. The proof is formally presented hereafter.

\LowerTauE*

\begin{proof}
    As mentioned earlier, it is sufficient to consider deterministic algorithms that map histories to policies. The lower bound for randomized algorithms follows by an application of Fubini's theorem (see \citealp{bubeck2012regret}). We consider two MDPs $\cH = \scbr{\cM_0, \cM_1}$ with the same state space $\cX = \scbr{x_0, x_1}$ and $\abs{\cA}$ actions available in each state. The initial distribution $\initial$ is chosen to be the uniform distribution over $\cX$, \ie $\initial \spr{x_0} = \initial \spr{x_1} = \frac12$. The transitions are the same in both MDPs: in state $x_1$, each action $a \in \cA$ induces the following transition kernel
    \begin{equation*}
        P \spr{x_0 \given x_1, a} = \delta, \quad P \spr{x_1 \given x_1, a} = 1 - \delta\,
    \end{equation*}
    while in state $x_0$, there is an action $a^\star$ giving a slightly higher probability on staying in state $x_0$, \ie
    \begin{align*}
        &P \spr{x_0 \given x_0, a^\star} = 1 - \delta + \epsilon, \quad P \spr{x_1 \given x_0, a^\star} = \delta - \epsilon \\
        \forall a \neq a^\star, &P \spr{x_0 \given x_0, a} = 1 - \delta, \quad P \spr{x_1 \given x_0, a} = \delta\,.
    \end{align*}
    The reward functions, $\true^0$ and $\true^1$, are different. In $\cM_0$, the ``good'' state is $x_0$, \ie for any action $a \in \cA$, we set $\true^0 \spr{x_0, a} = \WMAX$, $\true^0 \spr{x_1, a} = 0$, and in $\cM_1$, the ``good'' state is $x_1$, \ie $\true^1 \spr{x_0, a} = 0$, and $\true^1 \spr{x_1, a} = \WMAX$. Note that $\WMAX = \RMAX$ due to using the features $\phi_r \spr{x, a} = \bfe_x$ for any state-action pair $x, a$.
     
    Then, we define one expert policy for each MDP. In $\cM_0$, the expert $\pi_E^0$ is the policy that always plays $a^\star$ and in $\cM_1$, the expert $\pi_E^1$ is the policy that always plays an action $a \neq a^\star$. Therefore, the state occupancy measure of expert $\pi_E^0$ in MDP $\cM_0$ has the highest mass in state $x_0$, while $\pi_E^1$ put equal mass on both states. Indeed, writing the flow constraints for both experts, we have
    \begin{align*}
        \begin{pmatrix}
            1 - \gamma + \gamma \delta - \gamma \epsilon & - \gamma \delta \\
            - \gamma \spr{\delta - \epsilon} & 1 - \gamma + \gamma \delta
        \end{pmatrix}
        \nu \spr{\pi_E^0} &= \nu_0\,, \\
        \begin{pmatrix}
            1 - \gamma + \gamma \delta & - \gamma \delta \\
            - \gamma \delta & 1 - \gamma + \gamma \delta
        \end{pmatrix}
        \nu \spr{\pi_E^1} &= \nu_0\,.
    \end{align*}
    Solving these linear systems using, \eg, Cramer's rule, we obtain
    \begin{align*}
        \nu \spr{\pi_E^0, x_0} &= \frac{1 - \gamma + 2 \gamma \delta}{2 \spr{1 - \gamma - \gamma \epsilon + 2 \gamma \delta}}, &\nu \spr{\pi_E^0, x_1} &= \frac{1 - \gamma - 2 \gamma \epsilon + 2 \gamma \delta}{2 \spr{1 - \gamma - \gamma \epsilon + 2 \gamma \delta}}\,, \\
        \nu \spr{\pi_E^1, x_0} &= \frac12, &\nu \spr{\pi_E^1, x_1} &= \frac12\,.
    \end{align*}
    For $i \in \scbr{0, 1}$, let $\bar\pi^i$ be the policy output by $\mathrm{Alg}$ when given a dataset $\cD_{\pi_E^i}$ as input and let $V_i^{\bar\pi^i}$ be the value function of policy $\bar\pi^i$ corresponding to the reward function $\true^i$ from the MDP $\cM_i$. By definition of $\true^i$, we can write
    \begin{align}
        \frac12 \sum_{i \in \scbr{0, 1}} \inp{\initial, V_i^{\pi_E^i} - V_i^{\bar\pi^i}} &= \frac{1}{2 \spr{1 - \gamma}} \sum_{i \in \scbr{1, 2}} \inp{\mu \spr{\pi_E^i} - \mu \spr{\bar\pi^i}, \true^i} \nonumber \\
        &= \frac{\WMAX}{2 \spr{1 - \gamma}} \spr{\nu \spr{\pi_E^0, x_0} - \nu \spr{\bar\pi^0, x_0} + \nu \spr{\pi_E^1, x_1} - \nu \spr{\bar\pi^1, x_1}}\,. \label{eq:appE2-average-return}
    \end{align}
    Thus, we need to compute the difference between state occupancy measures. Let $\tilde\pi$ be an arbitrary policy and denote $\alpha \in \sbr{0, 1}$ the probability of playing action $a^\star$ in state $x_0$, \ie $\tilde\pi \spr{a^\star \given x_0} = \alpha$. Writing down the flow constraints again, we can show that
    \begin{equation*}
        \nu \spr{\tilde\pi, x_0} = \frac{1 - \gamma + 2 \gamma \delta}{2 \spr{1 - \gamma - \gamma \alpha \epsilon + 2 \gamma \delta}}, \quad \nu \spr{\tilde\pi, x_1} = \frac{1 - \gamma - 2 \gamma \alpha \epsilon + 2 \gamma \delta}{2 \spr{1 - \gamma - \gamma \alpha \epsilon + 2 \gamma \delta}}\,.
    \end{equation*}
    Looking at the difference with $\pi_E^0$ in state $x_0$, we have
    \begin{align*}
        \nu \spr{\pi_E^0, x_0} - \nu \spr{\tilde\pi, x_0} &= \frac{1 - \gamma + 2 \gamma \delta}{2 \spr{1 - \gamma - \gamma \epsilon + 2 \gamma \delta}} - \frac{1 - \gamma + 2 \gamma \delta}{2 \spr{1 - \gamma - \gamma \alpha \epsilon + 2 \gamma \delta}} \\
        &= \frac{\spr{1 - \gamma + 2 \gamma \delta} \spr{\spr{- \gamma \alpha \epsilon} - \spr{- \gamma \epsilon}}}{2 \spr{1 - \gamma - \gamma \epsilon + 2 \gamma \delta} \spr{1 - \gamma - \gamma \alpha \epsilon + 2 \gamma \delta}} \\
        &= \frac{\spr{1 - \gamma + 2 \gamma \delta} \gamma \epsilon \spr{1 - \alpha}}{2 \spr{1 - \gamma - \gamma \epsilon + 2 \gamma \delta} \spr{1 - \gamma - \gamma \alpha \epsilon + 2 \gamma \delta}}\,.
    \end{align*}
    Setting $\delta = \frac{1 - \gamma}{\gamma}$ and noting $\epsilon \geq 0$, $\gamma \geq \frac12$, we can lower bound the difference as follows
    \begin{align}
        \nu \spr{\pi_E^0, x_0} - \nu \spr{\tilde\pi, x_0} &= \frac{3 \spr{1 - \gamma} \gamma \epsilon \spr{1 - \alpha}}{2 \spr{3 \spr{1 - \gamma} - \gamma \epsilon} \spr{3 \spr{1 - \gamma} - \gamma \alpha \epsilon}} \nonumber \\
        &\geq \frac{\epsilon \spr{1 - \alpha}}{12 \spr{1 - \gamma}}\,. \label{eq:appE2-diffx0}
    \end{align}
    Likewise, the difference between $\nu \spr{\pi_E^1}$ and $\nu \spr{\tilde\pi}$ in state $x_1$ is given by
    \begin{align*}
        \nu \spr{\pi_E^1, x_1} - \nu \spr{\tilde\pi, x_1} &= \frac12 - \frac{1 - \gamma - 2 \gamma \alpha \epsilon + 2 \gamma \delta}{2 \spr{1 - \gamma - \gamma \alpha \epsilon + 2 \gamma \delta}} \\
        &= \frac{\gamma \alpha \epsilon}{2 \spr{1 - \gamma - \gamma \alpha \epsilon + 2 \gamma \delta}}\,.
    \end{align*}
    Using the definition of $\delta$, and again $\epsilon \geq 0$, $\gamma \geq \frac12$, we get
    \begin{align}
        \nu \spr{\pi_E^1, x_1} - \nu \spr{\tilde\pi, x_1} &= \frac{\gamma \alpha \epsilon}{2 \spr{3 \spr{1 - \gamma} - \gamma \alpha \epsilon}} \nonumber \\
        &\geq \frac{\epsilon \alpha}{12 \spr{1 - \gamma}}\,. \label{eq:appE2-diffx1}
    \end{align}
    Plugging Inequalities~\eqref{eq:appE2-diffx0} and~\eqref{eq:appE2-diffx1} into Equation~\eqref{eq:appE2-average-return} with $\alpha = \bar\pi^0 \spr{a^\star \given x_0}$ and $\alpha = \bar\pi^1 \spr{a^\star \given x_0}$ respectively, we get
    \begin{align*}
        \frac12 \sum_{i \in \scbr{0, 1}} \inp{\initial, V_i^{\pi_E^i} - V_i^{\bar\pi^i}} &\geq \frac{\epsilon \WMAX}{24 \spr{1 - \gamma}^2}\spr{1 - \bar\pi^0 \spr{a^\star \given x_0} + \bar\pi^1 \spr{a^\star \given x_0}} \\
        &= \frac{\epsilon \WMAX}{24 \spr{1 - \gamma}^2} \spr{\sum_{a \neq a^\star} \bar\pi^0 \spr{a \given x_0} + \bar\pi^1 \spr{a^\star \given x_0}}\,.
    \end{align*}
    Next, we can lower bound the right hand side using the Bretagnolle-Huber inequality (see \citealp{bretagnolle1979estimation}, and \citealp[Theorem 14.2]{lattimore2020bandit}), which gives
    \begin{equation} \label{eq:partial_lower_bound}
        \frac12 \sum_{i \in \scbr{0, 1}} \inp{\initial, V_i^{\pi_E^i} - V_i^{\bar\pi^i}} \geq \frac{\epsilon \WMAX}{24 \spr{1 - \gamma}^2} \exp \spr{- \KL \spr{\bar\pi^0 \spr{\cdot \given x_0} \| \bar\pi^1 \spr{\cdot \given x_0}}}\,.
    \end{equation}
    Then, using the data processing inequality and using the fact that the learning algorithm produces $\bar\pi^i$ as a deterministic function of the dataset $\cD_{\pi_E^i}$ for $i = 0, 1$, we have that
    \begin{equation*}
        \KL \spr{\bar\pi^0 \spr{\cdot \given x_0} \| \bar\pi^1 \spr{\cdot \given x_0}} \leq \KL \spr{\bbP_0^{\tau_E} \| \bbP_1^{\tau_E}}\,,
    \end{equation*}
    where, for $i \in \scbr{0, 1}$, we denoted $\bbP_i^{\tau_E}$ the probability distribution over datasets of size $\tau_E$ induced by the interaction between the expert $\pi_E^i$ and the environment (analog to what is done in the proof of Theorem~\ref{thm:LowerK}). Next, we denote $\kl \spr{p, q}$ and $\chi^2 \spr{p, q}$ the KL and chi-squared divergences between bernoulli distributions of means $p$ and $p'$, \ie
    \begin{align*}
        \kl \spr{p, q} &= p \log \spr{\frac{p}{q}} + \spr{1 - p} \log \spr{\frac{1 - p}{1 - q}} \\
        \chi^2 \spr{p, q} &= \frac{\spr{p - q}^2}{q \spr{1 - q}}\,.
    \end{align*}
    By definition of the KL, we have
    \begin{align*}
        \KL \spr{\bbP_0^{\tau_E} \| \bbP_1^{\tau_E}} &= \tau_E \cdot \kl \spr{\frac{3 \spr{1 - \gamma}}{2 \spr{3 \spr{1 - \gamma} - \gamma \epsilon}}, \frac12} \\
        &\leq \tau_E \cdot \chi^2 \spr{\frac{3 \spr{1 - \gamma}}{2 \spr{3 \spr{1 - \gamma} - \gamma \epsilon}}, \frac12} \\
        &= \tau_E \cdot \chi^2 \spr{\frac12 + \frac{\gamma \epsilon}{3 \spr{1 - \gamma} - \gamma \epsilon}, \frac12} \\
        &= \frac{4 \tau_E \gamma^2 \epsilon^2}{\spr{3 \spr{1 - \gamma} - \gamma \epsilon}^2} \\
        &\leq \frac{\tau_E \gamma^2 \epsilon^2}{\spr{1 - \gamma}^2}\,,
    \end{align*}
    where the first inequality follows from the concavity of the logarithm function, and the second inequality uses the fact that $\epsilon \leq \delta = \frac{1-\gamma}{\gamma}$. Thus, plugging in this last inequality into Equation~\eqref{eq:partial_lower_bound}, we obtain
    \begin{align*}
        \frac12 \sum_{i \in \scbr{0, 1}} \inp{\initial, V_i^{\pi_E^i} - V_i^{\bar\pi^i}} &\geq \frac{\epsilon \WMAX}{24 \spr{1 - \gamma}^2} \exp \spr{- \frac{\tau_E \gamma^2 \epsilon^2}{\spr{1 - \gamma}^2}} \\
        &\geq \frac{\epsilon \WMAX}{24 \spr{1 - \gamma}^2} \exp \spr{- \frac{\tau_E \epsilon^2}{\spr{1 - \gamma}^2}}\,,
    \end{align*}
    where we used $\gamma < 1$ in the second inequality. Introducing $\epsilon' = \epsilon \spr{1 - \gamma}^{-1}$, we can rewrite the previous inequality as
    \begin{equation*}
        \frac12 \sum_{i \in \scbr{0, 1}} \inp{\initial, V_i^{\pi_E^i} - V_i^{\bar\pi^i}} \geq \frac{\WMAX \epsilon'}{24 \spr{1 - \gamma}} \exp \spr{- \tau_E \spr{\epsilon'}^2}\,.
    \end{equation*}
    It remains to make the lower bound small enough. To bound the average suboptimality gap by $\frac{\WMAX \epsilon'}{24 e \spr{1 - \gamma}}$ and have $\frac12 \sum_{i \in \scbr{0, 1}} \inp{\initial, V_i^{\pi_E^i} - V_i^{\bar\pi^i}} \leq \frac{\WMAX \epsilon'}{24 e \spr{1 - \gamma}}$, we need at least $\tau_E \geq \frac{1}{\spr{\epsilon'}^2}$ expert transitions. Therefore, to achieve
    \begin{equation*}
        \frac12 \sum_{i \in \scbr{0, 1}} \inp{\initial, V_i^{\pi_E^i} - V_i^{\bar\pi^i}} \leq \varepsilon\,,
    \end{equation*}
    for some $\varepsilon > 0$, we need to choose $\epsilon' = \nicefrac{24 e \spr{1 - \gamma} \varepsilon}{\WMAX}$, which means that every algorithm needs at least $\tau_E \geq \frac{\WMAX^2}{24^2 e^2 \spr{1 - \gamma}^2 \varepsilon^2}$ to guarantee a suboptimality gap of order $\varepsilon$.
\end{proof}

%% file: sections/technical_tools.tex
\section{Technical tools}

\subsection{Reinforcement learning}

\begin{proposition}\label{prop:flow}
  The occupancy measure $\mu(\pi)$ of any policy $\pi$ satisfies the following system of equations:
  \begin{equation}\label{eq:flow}
      E\transpose \mu \spr{\pi} = \gamma P\transpose \mu \spr{\pi} + \spr{1 - \gamma} \nu_0\,.
  \end{equation}
\end{proposition}

\begin{proof}
  Define the transition kernel induced by policy $\pi$ as $P_\pi$, with $P_\pi(\cdot|x) = \mathbb{E}_{A\sim\pi(\cdot|x)}\sbr{P(\cdot|x,A)}$. The proof follows from the following standard calculation:
  \begin{align*}
    E\transpose \mu \spr{\pi} &= (1-\gamma) \sum_{\tau=0}^\infty \pa{\gamma P_\pi\transpose}^\tau \nu_0 \\
    &= (1-\gamma) \sum_{\tau=1}^\infty \pa{\gamma P_\pi\transpose}^\tau \nu_0 + (1-\gamma) \nu_0 \\
    &= \gamma P_\pi (1-\gamma) \sum_{\tau=0}^\infty \pa{\gamma P_\pi\transpose}^\tau \nu_0 + (1-\gamma) \nu_0 \\
    &= \gamma P_\pi E\transpose \mu \spr{\pi} + (1-\gamma) \nu_0 \\
    &= \gamma P \mu \spr{\pi} + (1-\gamma) \nu_0,
  \end{align*}
  where the last step follows from the easily-checked fact that $P \mu \spr{\pi} = P_\pi E\transpose \mu \spr{\pi}$.
\end{proof}

\begin{lemma}\label{lem:from-ev-to-q}
  Let $\pi$ be any policy, $Q \in \bbR^{\cX \times \cA}$ be any function defined on $\cX \times \cA$, and $V \in \bbR^\cX$ be such that for any $x$, $V \spr{x} = \bbE_{A \sim \pi \spr{\cdot \given x}} \sbr{Q \spr{x, A}}$. Then
  \begin{equation*}
    \inp{\mu \spr{\pi}, \opE V} = \inp{\mu \spr{\pi}, Q}\,.
  \end{equation*}
\end{lemma}

\begin{proof}
  We have
  \begin{align*}
    \inp{\mu \spr{\pi}, \opE V} &= \sum_{x \in \cX} \nu \spr{\pi, x} V \spr{x} \\
    &= \sum_{x \in \cX} \sum_{a \in \cA} \nu \spr{\pi, x} \pi \spr{a \given x} Q \spr{x, a} \\
    &= \sum_{x \in \cX} \sum_{a \in \cA} \mu \spr{\pi, x, a} Q \spr{x, a} \\
    &= \inp{\mu \spr{\pi}, Q}\,,
  \end{align*}
  where the second equality follows from the definition of the function $V$ and the first equality from the definition of the state-action occupancy measure.
\end{proof}

\begin{lemma}\label{lem:ineq-kl-entropy}
    Let $\pi$ and $\pi'$ be two policies. Then,
    \begin{equation*}
      \KL \spr{\mu \spr{\pi} \middle\| \mu \spr{\pi'}} \leq \frac{1}{1 - \gamma} \inp{\nu \spr{\pi}, \KL \spr{\pi \| \pi'}}\,.
    \end{equation*}
  \end{lemma}
  
\begin{proof}
  Using the chain rule of the relative entropy, we write
  \begin{equation*}
      \KL \spr{\mu \spr{\pi} \middle\| \mu \spr{\pi'}} = \KL \spr{\nu \spr{\pi} \middle\| \nu \spr{\pi'}} + \inp{\nu \spr{\pi}, \KL \spr{\pi \middle\| \pi'}}\,.
  \end{equation*}
  By the flow constraints and the joint convexity of the relative entropy, we bound the first term as
  \begin{align*}
    \KL \spr{\nu \spr{\pi} \middle\| \nu \spr{\pi'}} &= \KL \spr{\gamma P\transpose \mu \spr{\pi} + \spr{1 - \gamma} \nu_0 \middle\| \gamma P\transpose \mu \spr{\pi'} + \spr{1 - \gamma} \nu_0} \\
    &\leq \spr{1 - \gamma} \KL \spr{\nu_0 \middle\| \nu_0} + \gamma \KL \spr{P\transpose \mu \spr{\pi} \middle\| P\transpose \mu \spr{\pi'}} \\
    &= \gamma \KL \spr{P\transpose \mu \spr{\pi} \middle\| P\transpose \mu \spr{\pi'}} \\
    &\leq \gamma \KL \spr{\mu \spr{\pi} \middle\| \mu \spr{\pi'}}\,,
  \end{align*}
  where we also used the data-processing inequality in the last step. The proof is concluded by reordering the terms.
\end{proof}

\begin{lemma} \label{lem:mass-reduced}
  For any MDP $\cM$, any ascension function $p^\upplus$, and any policy $\pi$, we have for any state-action pair $\spr{x, a} \in \cX \times \cA$,
  \begin{equation*}
    \mu^\upplus \spr{\pi, x, a} \leq \mu \spr{\pi, x, a}\,,
  \end{equation*}
  where $\mu^\upplus \spr{\pi}$ denotes the state-action occupancy of $\pi$ in $\cM^\upplus$, the optimistically augmented MDP induced by $p^\upplus$.
\end{lemma}

\begin{proof}
  Let us consider a process  $\spr{X_\tau, A_\tau}_{\tau \in \bbN}$ generated by the policy $\pi$ in the MDP $\cM$, that is, such that $X_0 \sim \initial$, and for any $\tau \in \bbN$, $A_\tau \sim \pi \spr{\cdot \given X_\tau}$, and $X_{\tau+1} \sim P \spr{\cdot \given X_\tau, A_\tau}$. Additionally, we define a process $\spr{X_\tau^\upplus, A_\tau^\upplus}_{\tau \in \bbN}$ coupled to the process defined above as follows. At the first stage, we set $X_\tau^\upplus = X_0$. Then for any $\tau \geq 1$, the coupled process evolves as
  \begin{equation*}
    X_{\tau+1}^\upplus, A_{\tau+1}^\upplus =
    \begin{cases}
        X_{\tau+1}, A_{\tau+1} \quad &\text{w.p.} \quad 1 - p^\upplus \spr{X_\tau, A_\tau} \quad \text{if} \quad  X_\tau^\upplus, A_\tau^\upplus = X_\tau, A_\tau \\
        x^\upplus, a \quad & \text{w.p.} \quad p^\upplus \spr{X_\tau, A_\tau} \quad \text{if} \quad  X_\tau^\upplus, A_\tau^\upplus = X_\tau, A_\tau \\
        x^\upplus, a \quad & \text{if} \quad X_\tau^\upplus, A_\tau^\upplus \neq X_\tau, A_\tau
    \end{cases}\,.
  \end{equation*}
  It is straightforward to check that this process follows the dynamics of the optimistically augmented MDP $\cM^\upplus(r,p^\upplus)$ (since its transitions obey the kernel $P^\upplus$). By definition, for any state-action pair $\spr{x, a} \in \cX \times \cA$, we have
  \begin{align*}
    \mu^\upplus \spr{\pi, x, a} &= \spr{1 - \gamma} \sum_{\tau=0}^\infty \gamma^\tau \bbP \sbr{X_\tau^\upplus = x, A_\tau^\upplus = a} \\
    &= \spr{1 - \gamma} \sum_{\tau=0}^\infty \gamma^\tau \spr{\bbP \sbr{X_\tau^\upplus = x, A_\tau^\upplus = a, X_\tau^\upplus \neq x^\upplus} + \bbP \sbr{X_\tau^\upplus = x, A_\tau^\upplus = a, X_\tau^\upplus = x^\upplus}} \\
    &= \spr{1 - \gamma} \sum_{\tau=0}^\infty \gamma^\tau \spr{\bbP \sbr{X_\tau = x, A_\tau = a, X_\tau^\upplus \neq x^\upplus} + 0} \\
    &\leq \spr{1 - \gamma} \sum_{\tau=0}^\infty \gamma^\tau \bbP \sbr{X_\tau = x, A_\tau = a} \\
    &= \mu \spr{\pi, x, a}\,.
  \end{align*}
  In the third equality, the second term within the sum is equal to zero because $x \neq x^\upplus$, and in the other term we replaced $\spr{X_\tau^\upplus, A_\tau^\upplus}$ by $\spr{X_\tau, A_\tau}$ because the two coincide long as $X^\upplus_\tau \neq x^+$. This concludes the proof.
\end{proof}

\subsection{Linear algebra and analysis}

\begin{lemma} \label{lemma:number-epochs-bound}
    Under the event $\cE_L$, the number of epochs $E \spr{K}$ in Algorithm~\ref{alg:linear-rmax-ravi-ucb} is bounded as
    \begin{equation*}
      E \spr{K} \leq 5 d \log \spr{1 + \frac{B^2 T}{d}}\,.
    \end{equation*}
    where $T = \LMAX K = \frac{\log \spr{\frac{K}{\delta}} K}{1 - \gamma}$.
\end{lemma}

\begin{proof}
    In the following, we denote $\phi_t = \phi \spr{x_t, a_t}$ for any $t$. The bound on the number of epochs is derived observing that since the determinant of the matrix $\Lambda_k$ can grow at most linearly then the condition is triggered at most a logarithmic number of times. In particular notice that
    \begin{equation*}
        \det \spr{\Lambda_{t_{E \spr{K}}}} \geq 2 \det \spr{\Lambda_{t_{E \spr{K} - 1}}} \geq 2^2 \det \spr{\Lambda_{t_{E \spr{K} - 2}}} \geq 2^{E \spr{K} - 1} \det \spr{I} = 2^{E \spr{K} - 1}\,.
    \end{equation*}
    Hence, it holds that $E \spr{K} - 1 \leq \frac{1}{\log 2} \log \spr{\det \Lambda_{t_{E \spr{K}}}}$. Then, denoting $T_{K+1} = T_K + L_K$ where $L_K$ is the length of episode $K$, we have that
    \begin{align*}
        E \spr{K} &\leq 1 + \frac{1}{\log 2} \log \spr{\det \spr{\Lambda_{T_{K+1}}}} & \spr{\Lambda_{t_{E \spr{K}}} \preceq \Lambda_{T_{K+1}}} \\
        &\leq 1 + \frac{d}{\log 2} \log \brr{\frac{\optrace \spr{\Lambda_{T_{K+1}}}}{d}} & \text{(trace-determinant inequality)}\,.
    \end{align*}
    By definition of the covariance matrix,
    \begin{align*}
        E \spr{K} & \leq 1 + \frac{d}{\log 2} \log \spr{\frac{\optrace \spr{\sum_{t \in \sbr{T_{K+1}}} \phi_t \phi_t\transpose} + d}{d}} \\
        &= 1 + \frac{d}{\log 2} \log \spr{1 + \frac{\sum_{t \in \sbr{T_{K+1}}} \norm{\phi_t}_2^2}{d}}
        \\
        &\leq 1 + \frac{d}{\log 2} \log \spr{1 + \frac{B^2 T}{d}} \\
        &\leq 5 d \log \spr{1 + \frac{B^2 T}{d}}\,,
    \end{align*}
    where the first equality follows from properties of the trace and the second inequality follows from $\norm{\phi_t}_2 \leq B$ and $T_{K+1} \leq T$ which holds under $\cE_L$.
\end{proof}

The following lemma is a generalization of Lemma~19 of \citet{cassel2024warmupfree} for an arbitrary threshold $\omega \geq 0$.
\begin{lemma} \label{lem:sigmoid-bound}
  For all $z \geq 0$, $\omega \geq 0$ it holds that $\sigma \spr{z - \omega} \leq 2 \spr{z^2 + \exp \spr{- \omega}}$.
\end{lemma}

\begin{proof}
  Let us consider the function $g: z \mapsto \sigma \spr{z - \omega} - \spr{z + \frac{1}{e^{\omega/2}}}^2$. Note that for any $z$, we have $\sigma' \spr{z} = \sigma \spr{z} \sigma \spr{- z}$. Thus, the first two derivatives of $g$ are given by
  \begin{align*}
    g' \spr{z} &= \sigma \spr{z - \omega} \sigma \spr{\omega - z} - 2 \spr{z + \frac{1}{e^{\omega / 2}}}\,, \\
    g'' \spr{z} &= \sigma \spr{z - \omega} \sigma \spr{\omega - z}^2 - \sigma \spr{z - \omega}^2 \sigma \spr{\omega - z} - 2 \,.
  \end{align*}
  Since $\sigma \spr{z} \in \spr{0, 1}$ for any $z$, the second derivative of $g$ is nonpositive, $g'' \spr{z} \leq 0$, and $g$ is concave. By the first order condition, for any $z \geq 0$,
  \begin{equation*}
    g \spr{z} \leq g \spr{0} + g' \spr{0} z\,.
  \end{equation*}
  Furthermore, note that
  \begin{equation*}
    g \spr{0} = \sigma \spr{- \omega} - \frac{1}{e^\omega} = \frac{1}{1 + e^\omega} - \frac{1}{e^\omega} \leq 0\,,
  \end{equation*}
  and
  \begin{align*}
    g' \spr{0} &= \frac{1}{1 + e^\omega} \frac{1}{1 + e^{- \omega}} - \frac{2}{e^{\omega / 2}} \\
    &\leq \frac{1}{1 + e^\omega} - \frac{2}{e^{\omega / 2}} \\
    &\leq - \frac{1}{e^{\omega / 2}} \\
    &\leq 0\,,
  \end{align*}
  where we first used that $e^{- \omega} \geq 0$ for any $\omega \geq 0$ and then that $x + 1 \geq \sqrt{x}$ for any $x \geq 0$. Thus, it holds that $g \spr{z} \leq 0$ for all $z \geq 0$, \ie $\sigma \spr{z - \omega} \leq \spr{z + \frac{1}{e^{\omega / 2}}}^2$. Using $\spr{a+b}^2 \leq 2 \spr{a^2 + b^2}$, it holds that
  \begin{equation*}
    \sigma \spr{z - \omega} \leq 2 \spr{z^2 + e^{- \omega}}\,.
  \end{equation*}
\end{proof}

We present a variant of Lemma~18 of \citet{cassel2024warmupfree} which is valid for $\omega \geq 2$ instead of $\omega \geq 0$, but is sharper by a factor of $2$.
\begin{lemma} \label{lem:sigmoid-bound2}
  For all $\omega \geq 2$, it holds that
  \begin{equation*}
    \max_{z \geq 0} z \cdot \sigma \spr{\omega - \alpha z} \leq \frac{\omega}{\alpha}\,.
  \end{equation*}
\end{lemma}

\begin{proof}
  Let $\alpha > 0$, $\omega \geq 2$, and $g: z \geq 0 \mapsto z \cdot \sigma \spr{\omega - \alpha z}$. We recall that the derivative of the sigmoid function is given for any $z$ by $\sigma' \spr{z} = \sigma \spr{z} \sigma \spr{- z}$, and that $\sigma \spr{- z} = 1 - \sigma \spr{z}$. $g$ is twice differentiable. Its first derivative is given by
  \begin{align*}
    g' \spr{z} &= \sigma \spr{\omega - \alpha z} - \alpha z \sigma \spr{\omega - \alpha z} \sbr{1 - \sigma \spr{\omega - \alpha z}} \\
    &= \sigma \spr{\omega - \alpha z} \sbr{1 - \alpha z \spr{1 - \sigma \spr{\omega - \alpha z}}}\,.
  \end{align*}
  We set the derivative to zero and solve the equation to find the critical points. We have
  \begin{align}
    g' \spr{z} = 0 &\quad\text{iff}\quad \alpha z = \frac{1}{1 - \sigma \spr{\omega - \alpha z}} \label{eq:critical-point-prop} \\
    &\quad\text{iff}\quad \alpha z = 1 + e^{\omega - \alpha z} \nonumber \\
    &\quad\text{iff}\quad \spr{\alpha z - 1} e^{\alpha z - 1} = e^{\omega - 1}\,. \nonumber
  \end{align}
  For $x > 0$, the equation $w e^w = x$ has exactly one positive solution $w = W \spr{x}$ which increases with $x$ and where $W$ denotes the Lambert function. Thus, $g' \spr{z} = 0$ if and only if $\alpha z - 1 = W \spr{e^{\omega - 1}}$, \ie $z^\star = \frac{W \spr{e^{\omega - 1}} + 1}{\alpha}$. We check that $z^\star$ is a local maximum. The second derivative of $g$ is given by
  \begin{align*}
    g'' \spr{z} &= - 2 \alpha \sigma \spr{\omega - \alpha z} \sbr{1 - \sigma \spr{\omega - \alpha z}} \\
    &\phantom{=}+ \alpha^2 z \sigma \spr{\omega - \alpha z} \sbr{1 - \sigma \spr{\omega - \alpha z}}^2 \\
    &\phantom{=}- \alpha^2 z \sigma \spr{\omega - \alpha z}^2 \sbr{1 - \sigma \spr{\omega - \alpha z}} \\
    &= - 2 \alpha \sigma \spr{\omega - \alpha z} \sbr{1 - \sigma \spr{\omega - \alpha z}} \\
    &\phantom{=}+ \alpha^2 z \sigma \spr{\omega - \alpha z} \sbr{1 - \sigma \spr{\omega - \alpha z}} \sbr{1 - 2 \sigma \spr{\omega - \alpha z}}\,.
  \end{align*}
  We evaluate it at the critical point $z^\star$ and simplify the expression using Equation~\ref{eq:critical-point-prop}
  \begin{align*}
    g'' \spr{z^\star} &= - 2 \alpha \sigma \spr{\omega - \alpha z^\star} \sbr{1 - \sigma \spr{\omega - \alpha z^\star}} \\
    &\phantom{=}+ \alpha \sigma \spr{\omega - \alpha z^\star} \sbr{1 - 2 \sigma \spr{\omega - \alpha z^\star}} \\
    &= - \alpha \sigma \spr{\omega - \alpha z^\star} \\
    &< 0\,,
  \end{align*}
  thus $z^\star > 0$ is a local maximum. Since $g \spr{0} = 0$, $\lim_{z \rightarrow + \infty} g \spr{z} = 0$, $g \spr{z^\star}$ and $z^\star$ is the only positive critical point, this means $z^\star$ is a global maximum. We evaluate $g$ to get the maximum
  \begin{align*}
    g \spr{z^\star} &= \frac{W \spr{e^{\omega - 1}} + 1}{\alpha} \frac{1}{1 + \exp \spr{W \spr{e^{\omega - 1}}} e^{1 - \omega}} \\
    &= \frac{W \spr{e^{\omega - 1}} + 1}{\alpha} \frac{W \spr{e^{\omega - 1}}}{W \spr{e^{\omega - 1}} + W \spr{e^{\omega - 1}} \exp \spr{W \spr{e^{\omega - 1}}} e^{1 - \omega}} \\
    &= \frac{W \spr{e^{\omega - 1}}}{\alpha}\,,
  \end{align*}
  where we used $W \spr{e^{\omega - 1}} \exp \spr{W \spr{e^{\omega - 1}}} = e^{\omega - 1}$ in the third equality. We now upper bound the Lambert function. Taking the log of the equation that defines it, we have $W \spr{x} = \log x - \log W \spr{x}$ for any $x > 0$. Note that $W \spr{e} = 1$ and that $W$ is increasing, so for any $x > e$, we have $W \spr{x} > 1$ and thus $W \spr{x} < \log x$. Using it on $g \spr{z^\star}$, we further have
  \begin{equation*}
    g \spr{z^\star} \leq \frac{\omega - 1}{\alpha} \leq \frac{\omega}{\alpha}\,,
  \end{equation*}
  where we used $\omega \geq 2$. This concludes the proof.
\end{proof}

\begin{lemma} \label{lem:lse-lipschitz}
  Let $n \in \bbR^n$, and define $\LSE: \bbR^n \rightarrow \bbR$ the function defined for any $x \in \bbR^n$ as
  \begin{equation*}
    \LSE \spr{x} = \log \sum_{i=1}^n e^{x_i}\,.
  \end{equation*}
  Then $\LSE$ is 1-Lipschitz with respect to the norm $\norm{\cdot}_\infty$, \ie for any $x, y \in \bbR^n$,
  \begin{equation*}
    \abs{\LSE \spr{x} - \LSE \spr{y}} \leq \norm{x - y}_\infty\,.
  \end{equation*}
\end{lemma}

\begin{proof}
  For any $i \in \sbr{n}$ and any $x \in \bbR^n$, the gradient of $\LSE$ is given by
  \begin{equation*}
    \nabla \LSE \spr{x} = \frac{e^{x}}{\inp{e^{x}, \bfone}}\,.
  \end{equation*}
  Let $y \in \bbR^n$. By the intermediate mean value theorem, there exists a $z$ on the segment $\sbr{x, y}$ such that
  \begin{align*}
    \abs{\LSE \spr{x} - \LSE \spr{y}} &= \abs{\inp{\nabla \LSE \spr{z}, x - y}} \\
    &\leq \norm{\nabla \LSE \spr{z}}_1 \norm{x - y}_\infty \\
    &= \norm{x - y}_\infty\,,
  \end{align*}
  where the inequality follows from Hölder's inequality.
\end{proof}

\begin{lemma}[\citealp{cohen2019learning}, Lemma~27] \label{lem:det-elliptical-bound}
  If $0 \prec M \preceq N$ then for any vector $v$,
  \begin{equation*}
    \norm{v}_N^2 \leq \frac{\det N}{\det M} \norm{v}_M^2\,.
  \end{equation*}
\end{lemma}

\begin{lemma}[\citealp{sherman2023}, Lemma~15] \label{lemma:beta_bound}
    Let $R, z \geq 1$, then $\beta \geq 2 z \log \spr{R z}$ ensures $\beta \geq z \log \spr{R \beta}$.
\end{lemma}

\begin{lemma}[\citealp{rosenberg2020near}, Lemma~D.4] \label{lem:concentration-ineq-cond-exp}
  Let $\bc{X_k}_{k \in [K]}$ be a sequence of random variables adapted to the filtration $\bc{\mathcal{F}_k}_{k\in[K]}$ and suppose that $0 \leq X_k \leq X_{\max}$ almost surely. Then, with probability at least $1-\delta$, the following holds for all $k \geq 1$ simultaneously
  \begin{equation*}
    \sum^K_{k=1} \mathbb{E}\bs{X_k | \mathcal{F}_{k-1}} \leq 2 \sum^K_{k=1} X_k + 4 X_{\max} \log \frac{2 K}{\delta}\,.
  \end{equation*}
\end{lemma}

\begin{lemma}[\citealp{jin2019provably}, Lemma~D.2] \label{lem:bound-elliptical-potential}
  Let $\scbr{\phi_t}_{t \geq 0}$ be a bounded sequence in $\bbR^d$ satisfying $\sup_{t \geq 0} \norm{\phi_t} \leq 1$. Let $\Lambda_0 \in \bbR^{d \times d}$ be a positive definite matrix. For any $t \geq 0$, we define $\Lambda_t = \Lambda_0 + \sum_{j=1}^t \phi_j \phi_j\transpose$. Then, if the smallest eigenvalue of $\Lambda_0$ satisfies $\lambda_{\mathrm{min}} \spr{\Lambda_0} \geq 1$, we have
  \begin{equation*}
    \sum_{j=1}^t \phi_j \Lambda_{j-1}^{-1} \phi_j \leq 2 \log \spr{\frac{\det \Lambda_t}{\det \Lambda_0}}\,.
  \end{equation*}
\end{lemma}